%% file: arxiv.tex
\documentclass[10pt]{article}

\usepackage[utf8]{inputenc}
\usepackage[T1]{fontenc}
\usepackage{fullpage}
\usepackage{newtxtext}
\usepackage{microtype}
\usepackage[authoryear,round]{natbib}
\usepackage{float}
\usepackage[table,dvipsnames]{xcolor}
\usepackage{hyperref}
\usepackage{url}
\usepackage{booktabs}
\usepackage{enumitem}
\usepackage{amsmath,amssymb,mathtools,amsthm}
\usepackage{graphicx}
\usepackage{subcaption}
\usepackage{multirow}
\usepackage{placeins}
\usepackage{etoc}
\usepackage{pifont}
\usepackage[normalem]{ulem}
\usepackage{array}

\makeatletter
\def\input@path{{./}}
\makeatother
\graphicspath{{./}}

\input{math_commands.tex}

\usepackage{newtxmath}

\usepackage[capitalize,noabbrev]{cleveref}
\input{macros.tex}

\theoremstyle{plain}
\newtheorem{theorem}{Theorem}[section]
\newtheorem{proposition}[theorem]{Proposition}
\newtheorem{lemma}[theorem]{Lemma}
\newtheorem{corollary}[theorem]{Corollary}
\theoremstyle{definition}

\newtheorem{assumption}[theorem]{Assumption}
\theoremstyle{remark}
\newtheorem{remark}[theorem]{Remark}

\definecolor{citationblue}{RGB}{0,0,109}
\hypersetup{
  colorlinks=true,
  citecolor=citationblue,
  linkcolor=blue,
  urlcolor=blue
}

\title{\textbf{Out of Support, Not Out of Reach: In-Context Learning with Unsupported Instructions and Composed Tasks}}

\author{
Xuhan Tong\thanks{Department of Computer Sciences, University of Wisconsin--Madison.}
\and
Yuchen Zeng\thanks{Microsoft Research.} $ $ \thanks{Co-last authors.}
\and
Jiawei Zhang\footnotemark[1] $ $ \footnotemark[3]
}

\date{}

\begin{document}

\maketitle

\input{abstract.tex}
\input{main_body.tex}

\section*{AI Use Statement}
Generative AI tools were used to assist with refining the theoretical
framework, formulating and checking mathematical claims and proofs,
designing experimental methodology, developing and editing experimental
code, constructing synthetic experimental pipelines, analyzing and
interpreting results, searching and summarizing relevant literature, and
drafting, editing, and formatting the manuscript. The authors reviewed all
AI-assisted material, checked the mathematical arguments against the stated
assumptions, inspected and tested the experimental implementations, and
validated reported results through repeated runs and explicit sanity
checks. The authors made all final research decisions and take
responsibility for the complete contents of the paper, including all text,
claims, proofs, code, experiments, and artifacts produced with AI
assistance.

\section*{Reproducibility Statement}
We state the probabilistic setting, assumptions, and model-level
distinctions used by the theory in
Sections~\ref{subsec:instruction_setup}--\ref{sec:cot} and
Appendix~\ref{app:probability_models}. Complete proofs of the instruction-
and CoT-level results, including the model--reference transport terms, are
provided in Appendices~\ref{app:instruction} and
\ref{app:cot_main_proof}. The main experimental protocols are summarized in
Section~\ref{sec:experiments}; Appendix~\ref{app:experiment} reports the
task constructions, model variants, training and evaluation procedures,
prompt formats, demonstration counts, random-seed aggregation, metrics, and
additional results. It also documents the controlled model--reference, CoT
error-propagation, continuation-reach, and Qwen3 soft-prompt interpolation
experiments, together with their numerical sanity checks.

\bibliographystyle{plainnat}
\bibliography{references}

\clearpage
\appendix
\tableofcontents

\clearpage
\section{List of Common Notations}
\input{NeurIPS2026/appendix/list_notation}

\section{Related Work}
\input{NeurIPS2026/appendix/related_work}

\section{Probability Models}
\input{appendix_entropy.tex}

\section{Proof of \cref{sec:instruction}}
\input{NeurIPS2026/appendix/minimal_modified_app_instruction}

\section{Proof of \cref{sec:cot}}
\input{NeurIPS2026/appendix/appendix_cot_modified}

\section{Formal Justification of the Padding Reduction}
\input{NeurIPS2026/appendix/appendix_padding}

\section{Experiments}
\input{NeurIPS2026/appendix/appendix_experiment}

\section{Limitations}
\input{NeurIPS2026/appendix/limitations}

\end{document}

%% file: math_commands.tex
\usepackage{amsmath, amsfonts, bm}

\def\eqref#1{(\ref{#1})}

  \def\1{\bm{1}}

   \def\vi{{\bm{i}}}

\def\vz{{\bm{z}}}

\DeclareMathAlphabet{\mathsfit}{\encodingdefault}{\sfdefault}{m}{sl}
\SetMathAlphabet{\mathsfit}{bold}{\encodingdefault}{\sfdefault}{bx}{n}

   \def\gT{{\mathcal{T}}}
 \def\gV{{\mathcal{V}}}

  \def\sP{{\mathbb{P}}}



%% file: macros.tex
\usepackage{enumitem}
\usepackage{amsmath}

\crefname{assumption}{Assumption}{Assumptions} \Crefname{assumption}{Assumption}{Assumptions}

 \def\dpre{D_0} \def\dicl{D}  
   
\def\xquery{x_{\text{query}}}  \def\yquery{y_{\text{query}}}
 \def\instruction{i} \def\iquery{i_{\text{query}}}
  
 \def\kl{\operatorname{KL}}

%% file: abstract.tex
\begin{abstract}

In-Context Learning (ICL) enables pretrained LLMs to adapt to downstream
tasks from a small set of input-output demonstrations, without parameter
updates. Yet existing theory primarily explains ICL as identifying a latent
task already represented in pretraining, leaving unclear whether ICL can
succeed when either the instruction or the target task falls outside the
pretraining support. 
We develop a theoretical framework under mild assumptions
that separates these two forms of distribution shift. For target tasks represented in pretraining,
we show that demonstrations can identify the task and reduce sensitivity to
pretrain-supported prompt templates as their number increases. For an
unsupported instruction, we show that under a fixed model-induced reference
extension, demonstration evidence can dominate an $N$-independent instruction
shift while preserving the exponential identification rate.
We then use the same ICL framework as an interpretive lens for long-horizon
target tasks that are not represented in pretraining but admit a plan that
decomposes them into pretraining-supported subtasks. Applying the supported-task
analysis step by step yields a standard accumulated-error bound and clarifies
how invoking familiar subtasks along an intermediate-state trajectory can
avoid directly selecting the unsupported complete task.
Together, our results characterize
when demonstrations, instructions, and task decomposition can support ICL
beyond the corresponding pretraining supports.
\end{abstract}

%% file: main_body.tex
\section{Introduction}
\label{sec:intro}

In-Context Learning (ICL)~\citep{Brown2020} is widely used to adapt Large
Language Models (LLMs) through an instruction and a small set of input--output
demonstrations, without updating model parameters. A prominent line of theory
explains ICL as implicit Bayesian inference over a latent
task~\citep{Xie2022,Wang2023(b),Panwar2024}:
the demonstrations provide evidence about which task generated the prompt, and
prediction improves as the posterior concentrates on that task. The standard
account relies on two support assumptions. First, the posterior ranges over
latent tasks or concepts already represented by the pretraining mixture, so
the target has positive prior probability under that
mixture~\citep{Xie2022,Wies2023(a),Wang2023(b)}. Second, the instruction is
typically omitted or absorbed into a fixed canonical prompt format, rather
than modeled as a task-conditional variable that may differ across
demonstrations~\citep{Akyurek2023,vonOswald2023,Ahn2023}.
In deployment, however, both assumptions can fail.

Consider first the instruction. Real prompts can depart from the pretraining
distribution in systematic ways. A user of Indian English may
write \emph{please revert with the numbers}, using \emph{revert} to mean
``reply,'' while the same word is associated predominantly with ``undo'' in
the model's coding data. The term \emph{bias} may denote a regression intercept
to one user and a fairness property to another. A company's internal service
may retain the name of an open-source project even after its behavior has
diverged. Other deviations are inconsistent rather than repeated — typos that 
differ from slot to slot, or demonstrations assembled by retrieval or by 
multiple annotators, each carrying its own rendering. What matters is the 
interaction of two axes: whether the instruction was seen paired with the 
intended task during pretraining, and whether one instruction governs the 
prompt or each demonstration carries its own. Table~\ref{tab:instruction_formats} formalizes these seven
structures. They lead to our first research question:
\emph{When can additional demonstrations overcome an instruction outside the
task-conditional pretraining support, and when do additional instructions
introduce enough new shift to offset the evidence supplied by those
demonstrations?}

The other support assumption fails at the level of the target task itself.
Language-model agents are increasingly asked to execute long-horizon objectives
that may never have appeared as complete tasks during pretraining. At the same
time, many such objectives admit plans composed of pretraining-supported skills: extracting
a field, transforming a record, invoking a known operation, checking an
intermediate state, and aggregating the results. A latent-task posterior cannot
directly select a target task with zero pretraining probability, but an agent
need not execute that task as one indivisible mapping. As an application of
the ICL perspective, we ask: \emph{When can an out-of-support target task be realized by
a plan that decomposes it into pretraining-supported subtasks, and how do
step-wise errors accumulate over a long execution horizon?} We study execution
under a given task decomposition, rather than the separate problem of how an
agent discovers that decomposition.

We address both questions in a common latent-task framework. For unsupported
instructions, task evidence from demonstrations competes with the shift induced
by the instruction structure. For long-horizon tasks, the same perspective is
applied along an intermediate-state trajectory of supported subtasks. A
pathwise Lipschitz analysis quantifies sensitivity along both types of paths.
\vspace{-0.3em}
\paragraph{Our Contributions.}
We distinguish instruction-level support from task-level support and
characterize how generalization at these two levels is shaped by the number and
provenance of demonstrations, instruction placement, and task decomposition.
Our main contributions are
summarized as follows.

\begin{itemize}[leftmargin=*, topsep=0pt, itemsep=1pt, parsep=2pt]
        \item \textbf{Characterizing ICL with Supported and
        Unsupported Instructions~(Section~\ref{sec:instruction}).}
        Holding the target task within the pretraining task support, we
        distinguish instructions that lie inside or outside the corresponding
        task-conditional instruction support and analyze the seven prompt
        structures in Table~\ref{tab:instruction_formats}. We show that
        informative demonstrations identify the target task and, under a fixed
        model-induced reference extension, preserve their exponential rate
        despite a fixed unsupported instruction whose continuation cost does
        not grow with the demonstration count. This is a conditional statement
        about the specified extension, rather than an identification of its
        off-support values from the pretraining distribution. For mixed or varying instructions,
        the instruction-shift cost can instead accumulate; convergence depends
        on whether task evidence grows faster than that cost.

    \item \textbf{Applying the ICL View to Long-Horizon Tasks through Supported Subtasks~(Section~\ref{sec:cot}).}
    We use the preceding ICL framework to interpret a target task with zero
    pretraining probability. Although
    such a task cannot be selected directly by the pretrained latent-task
    posterior, it may be realized by a plan that composes supported subtasks
    along an intermediate-state trajectory. We bound the final error by the accumulated errors of the execution
steps, revealing why a single poorly supported or unstable subtask can become the bottleneck.

    \item \textbf{Validating the Theoretical Statements through
    Experiments~(Section~\ref{sec:experiments}).} We test the predicted
    interaction between instruction structure and demonstration count, the
    pathwise characterization of prompt sensitivity, and the effect of
    decomposition quality on multi-step task execution.
\end{itemize}
Taken together, our instruction results offer a unified view of ICL beyond pretraining support. For a
specified off-support reference extension, demonstrations can
supply evidence for the intended latent task that dominates a fixed
instruction-shift cost; realizing this behavior in the actual model further
requires controlled model--reference error. Task decomposition, meanwhile,
brings a novel long-horizon objective within the model's repertoire of
pretrained skills. In both cases, success depends on whether the useful
evidence dominates the sensitivity accumulated along the path from the test
prompt to what the model has already learned.
\vspace{-0.5em}
\section{Problem Setup}
\label{subsec:instruction_setup}

This section formalizes the latent-task view of ICL and
defines the seven prompt formats used throughout our analysis by distinguishing
task-conditional instruction support and instruction placement.

\paragraph{Latent Task Model.}
Let $\gV$ be a finite vocabulary and let $\mathcal S$ be a finite space of
individual inputs, outputs, and intermediate states. Following
\citet{Xie2022}, we assume that each sample
$(x,y)\in\mathcal S\times\mathcal S$ is generated from a latent task $t$ in a
finite task set $\gT$. Conditioned on $t$,
$X \sim P_{X \mid t}$, $Y = f_{t}(X)$,
where $f_t:\mathcal S\to\mathcal S$ is the task-specific mapping. We write
$P(\cdot)$ for distributions and $p(\cdot)$ for their mass or density
functions, and assume $\pi(t)=\mathbb{P}(T=t)>0$ for every $t\in\gT$. Section~\ref{sec:cot}
will additionally consider downstream tasks $t^\dagger\notin\gT$, for which
we extend the pretraining prior by $\pi(t^\dagger)=0$.
An ICL prompt with $N$ demonstrations is written as
$\vz_N=\bigl(\vi,(x_1,y_1),\ldots,(x_N,y_N),\xquery\bigr)$; we omit the
subscript and write $\vz$ when $N$ is clear from context. Here $\vi$ collects
the instruction variables under the chosen prompt format
and $\xquery$ is the query input. Let $\dpre$
denote the support of the pretraining prompt distribution across all values of
$N$.
All almost-sure statements as $N\to\infty$ refer to the length-$N$ prefixes of
one infinite sequence of demonstrations generated by the stated
mechanism.
\vspace{-0.4em}
\paragraph{Instructions.}
A single task may admit multiple instructions and prompt formats. For a
pretraining-supported task $t\in\gT$, we model its instructions as draws from
the task-conditional distribution $P_{\instruction\mid t}$. Support is therefore
task-conditional: for a target task $t^*\in\gT$, an instruction $\instruction$
is \emph{supported} if $p(\instruction\mid t^*)>0$ and \emph{unsupported} if
$p(\instruction\mid t^*)=0$. This definition asks whether pretraining ever
paired the instruction with the intended task, not whether the words appeared
in pretraining or whether a user regards them as semantically correct. For
example, \emph{revert} may be common in coding data but unsupported for the
task of drafting a reply. The distinction matters because
$p(\instruction\mid t^*)$ enters the latent-task posterior directly: ordinary
Bayesian conditioning applies to supported instructions, whereas an
unsupported instruction requires the off-support analysis developed below.

Fix a maximum
instruction length $L_i$, let
$e:\gV\cup\{\langle\mathrm{pad}\rangle\}\to\mathbb R^d$ be the token
embedding map with $e(\langle\mathrm{pad}\rangle)=0$, and right-pad every
instruction to length $L_i$. Write
$\gV^{\le L_i}:=\bigcup_{m=0}^{L_i}\gV^m$. For
$\instruction=(v_1,\ldots,v_m)\in\gV^{\le L_i}$, define its padded embedding
matrix as $E(\instruction):=[e(v_1),\ldots,e(v_m),0,\ldots,0]^{\top}
\in\mathbb R^{L_i\times d}$.
Thus instructions with different token lengths lie in one common ambient
space; Appendix~\ref{app:padding} provides the corresponding padding-stability
justification. For each competing task $t\ne t^*$ and supported instruction
$\instruction$, define the task-relative instruction score
$
r_t(\instruction)
:=\log(p(\instruction\mid t) / p(\instruction\mid t^*))$.
The score compares how likely the instruction is under the competing task $t$
and the target task $t^*$.
For the continuous-path analysis, we use a fixed model-induced extension of
this score along the instruction-embedding path defined below.
We do not require the off-support extension to arise from separate normalized
densities on the continuous instruction-embedding space. Its values away from the
symbolic embeddings are additional structure made explicit in
the continuation condition below, rather than something identified by the
finite symbolic support. Appendix~\ref{app:probability_models} shows that an
analytic Transformer scoring map supplies one concrete such extension and
implies the required finite continuation chain.
\vspace{-0.4em}
\paragraph{Prompt Formats.}
The formats below model common ways in which prompts are assembled. A user or
system message supplies one global instruction; a structured dataset may copy
one field or operator into every record; and retrieval or multiple annotators
may attach a different rendering to each demonstration. These placement
patterns matter because a shared instruction contributes one piece of task
evidence, whereas independently varying instructions can contribute a new
likelihood factor with every demonstration. Crossing the three placement
patterns---single, repeated, and varying---with supported versus unsupported
instructions gives formats~1--6. Format~7 captures partial corruption, where a
supported template remains in most slots but typos or schema errors affect a
subset (Table~\ref{tab:instruction_formats}).
This distinction follows from the assumed generative factorization rather than from the Transformer architecture alone.
Appendix~\ref{app:effective_instruction_sample_size} gives a controlled
experiment showing that a Transformer can learn the corresponding effective
instruction sample sizes from cross-position correlation structure.

The support labels in the table follow directly from the sampling mechanism.
In formats~1, 3, and~5, every instruction is drawn from
$P_{\instruction\mid t^*}$ and is therefore supported. Formats~2, 4, and~6
instead select instructions outside its support, so
$p(\tilde{\instruction}\mid t^*)=0$. Format~7 contains both the shared
supported template and distinct unsupported corruptions. For notational convenience, we use $\vi$ to denote the collection of all
instructions appearing in a prompt, regardless of the specific prompt format.
Unless otherwise specified, input-output pairs $(X,Y)$ are drawn i.i.d.
conditional on the latent task.

\begin{table}[t]
\centering
\caption{Seven prompt formats classified by support and
instruction placement. Supported means that the instruction was paired with
the intended task with positive probability during pretraining; unsupported
means that this probability is zero. Appendix~\ref{app:prompt_formats} gives complete
examples.}
\label{tab:instruction_formats}
\setlength{\tabcolsep}{4pt}
\renewcommand{\arraystretch}{1.35}
\small
\definecolor{rowmaingreen}{HTML}{F3F8E8}
\resizebox{\textwidth}{!}{%
\begin{tabular}{@{}c c l l p{6.2cm}@{}}
\toprule
\textbf{\#} & \textbf{Support} & \textbf{Appearance} & \textbf{Prompt format} & \textbf{Example} \\
\midrule
\rowcolor{rowmaingreen}
1 & Supported & Single   & $\vz=(\instruction,(x_n,y_n)_{n=1}^{N},\xquery)$
  & A standard instruction stated once: ``Classify ticket priority.'' \\
\rowcolor{rowmaingreen}
2 & Unsupported & Single   & $\vz=(\tilde{\instruction},(x_n,y_n)_{n=1}^{N},\xquery)$
  & Indian-English \emph{revert} used once to mean ``reply.'' \\
\rowcolor{rowmaingreen}
3 & Supported & Repeated & $\vz=\bigl((\instruction,x_n,y_n)_{n=1}^{N},(\instruction,\xquery)\bigr)$
  & A standard schema key, e.g. \texttt{task: sentiment}, repeated in every record. \\
\rowcolor{rowmaingreen}
4 & Unsupported & Repeated & $\vz=\bigl((\tilde{\instruction},x_n,y_n)_{n=1}^{N},(\tilde{\instruction},\xquery)\bigr)$
  & One private API label, e.g. \texttt{operation: Atlas}, reused in every record. \\
\rowcolor{rowmaingreen}
5 & Supported & Varying  & $\vz=\bigl((\instruction_n,x_n,y_n)_{n=1}^{N},(\iquery,\xquery)\bigr)$
  & Retrieved examples using familiar paraphrases such as ``sentiment'' and ``tone.'' \\
\rowcolor{black!8}
6 & Unsupported & Varying  & $\vz=\bigl((\tilde{\instruction}_n,x_n,y_n)_{n=1}^{N},(\tilde{\iquery},\xquery)\bigr)$
  & Multi-tenant examples carrying distinct private labels such as \texttt{triage\_acme} and \texttt{route\_7G}. \\
\rowcolor{black!8}
7 & Mixed & Corrupted & $\vz=\bigl((\instruction_n,x_n,y_n)_{n=1}^{N},(\iquery,\xquery)\bigr)$,
$\instruction_n\in\{\instruction,\tilde{\instruction}_n\}$
  & A standard template with sparse typos such as \texttt{sentimnet} and \texttt{sentimemt}. \\
\bottomrule
\end{tabular}%
}
\end{table}

\vspace{-0.4em}
\section{ICL with Supported and Unsupported Instructions}
\label{sec:instruction}

Throughout this section, the target task is pretraining-supported:
$t^*\in\gT$ and $\pi(t^*)>0$. The distinction studied below concerns only
whether the instruction is supported conditional on that task.

Let $M(y\mid\vz)$ denote the output probability of the actual pretrained
model. The Bayesian predictor $p(y\mid\vz)$ on supported prompts, and its fixed
continuation $\bar p(y\mid\vz)$ on unsupported prompts, serve as theoretical
reference predictors, which need not coincide with $M$ exactly (see Appendix~\ref{app:model_reference_transport}).

\subsection{ICL with Pretraining-Supported Instructions}

We first consider prompt formats whose instructions lie in the support of the
pretraining distribution.

\begin{assumption}[Identifiability of task evidence]
\label{assp:identifiability}
For every pair of distinct tasks $t_1,t_2\in\gT$, the demonstration
distributions satisfy $\kl\!\left(P_{X,Y\mid t_1}\,\Vert\,
P_{X,Y\mid t_2}\right)>0$,
and the corresponding demonstration log-likelihood ratio is integrable under
$P_{X,Y\mid t_1}$. For format~5, the log-likelihood ratio of the joint
instruction--demonstration distribution is also integrable. Under the
format-specific generation mechanism described in Section~\ref{subsec:instruction_setup},
\begin{equation}
\label{eq:main_joint_identifiability}
\kl\!\left(
P_{\vi\mid t_1}\otimes P_{X,Y\mid t_1}
\,\Vert\,
P_{\vi\mid t_2}\otimes P_{X,Y\mid t_2}
\right)
=
\kl\!\left(P_{\vi\mid t_1}\,\Vert\,P_{\vi\mid t_2}\right)
+
\kl\!\left(P_{X,Y\mid t_1}\,\Vert\,
P_{X,Y\mid t_2}\right)>0.
\end{equation}
Thus the task evidence is identifiable in formats~1, 3, and~5, while the
demonstration KL condition above also supplies the gap used in the
unsupported-instruction results below.
\end{assumption}

Under this assumption, the demonstrations provide exponentially increasing
evidence for the target task. In format~5, the independently sampled
paraphrases form part of this task evidence rather than a single shared
instruction factor.

\begin{theorem}[Informal: model bound with supported instructions]
\label{thm:decay_1_5}
Consider formats~1, 3, and~5 under Assumption~\ref{assp:identifiability}.
Let $t^*$ be the target task. Then there exist constants $\alpha,\beta>0$ and
a finite model-approximation error $\varepsilon_M\ge0$
such that, almost surely, for
all sufficiently large $N$,
\begin{equation}
\label{eq:main_supported_convergence}
\left|M(\yquery\mid\vz)-p(\yquery\mid\xquery,t^*)\right|
\le \beta e^{-\alpha N}+\varepsilon_M.
\end{equation}
\end{theorem}

The bound shows several key insights. The
exponent $\alpha$ is controlled by the least distinguishable competing task:
when every demonstration provides positive evidence for $t^*$, posterior mass
on the competitors decays exponentially with $N$. The prefactor $\beta$
collects effects that do not grow with the demonstration count, including the
task prior, the query distribution, and a shared instruction factor. Thus, in
formats~1 and~3, copying the same supported instruction into more positions
does not create additional independent evidence; the improvement with $N$
comes from the new input--output pairs. In format~5, by contrast, each
supported paraphrase is part of an
instruction--demonstration triple, so it can contribute additional evidence
for identifying the task. See Appendix~\ref{app:supported_instruction_proof}
for complete proof.

\subsection{ICL with Pretraining-Unsupported Instructions}

\subsubsection{Single or repeated unsupported instructions.}
\label{subsec:unsupported_shared}

We first consider formats~2 and~4 when the same unsupported instruction is used
throughout the prompt. Format~2 contains one displayed occurrence; format~4
contains $N+1$ copies of one shared realization. The repeated copies in
format~4 are not treated as $N+1$ independent task observations: the
instruction template is sampled once and copied deterministically into all
displayed slots. Thus, for every task $t$,
\begin{equation}
\label{eq:main_shared_factorization}
p(\vz\mid t)
=p(\vi\mid t)p(\xquery\mid t)
\prod_{n=1}^{N}p(x_n,y_n\mid t).
\end{equation}
Format~2 is the corresponding single-occurrence special case.

Let $\vi_0$ be a supported anchor and $\widetilde\vi$ the unsupported
endpoint. In the padded instruction-embedding space of
Section~\ref{subsec:instruction_setup}, define
\begin{equation}
\label{eq:main_instruction_embedding_path}
E_s:=(1-s)E(\vi_0)+sE(\widetilde\vi),\qquad s\in[0,1].
\end{equation}
For each competing task, we analytically continue its supported instruction
log-likelihood ratio along this path. We denote the corresponding continued
predictive probability by $\bar p(\yquery\mid\vz)$. This continuation is a
model-induced reference extension; because the endpoint has zero probability,
its value is not identified by the pretraining distribution alone.

Under this condition, we can give an exponential bound on the continued reference predictor and
a corresponding bound for the actual model.

\begin{theorem}[Informal: unsupported shared instruction]
\label{thm:repeated_unsupported}
Suppose Assumption~\ref{assp:identifiability} holds and, for every
$t\neq t^*$, the supported log-likelihood ratio admits a fixed finite analytic
continuation from the supported anchor to the unsupported endpoint. Then there
exist constants $\alpha,\beta>0$, independent of $N$, such that,
almost surely for all sufficiently large $N$, the continued reference
predictor satisfies
\begin{equation}
\label{eq:main_unsupported_shared_reference_bound}
\left|\bar p(\yquery\mid\vz)
-p(\yquery\mid\xquery,t^*)\right|
\le \beta e^{-\alpha N}.
\end{equation}
For the actual pretrained model,
\begin{equation}
\label{eq:main_unsupported_shared_bound}
\left|M(\yquery\mid\vz_N(\widetilde\vi))
-p(\yquery\mid\xquery,t^*)\right|
\le \beta e^{-\alpha N}+\varepsilon_{M,N}(\widetilde\vi),
\end{equation}
where $\varepsilon_{M,N}(\widetilde\vi)$ is the model--reference transport
error along the tied shared-template path, defined formally in
Appendix~\ref{app:model_reference_transport}. Assumption~\ref{assp:uniform_shared_tracking}
shows when the transport error is uniformly bounded, and when the actual model converges to the target-task predictor.
\end{theorem}
Theorem~\ref{thm:repeated_unsupported} handles a
prompt outside the pretraining support in the repeated-instruction setting. In
particular, the same unsupported instruction $\widetilde\vi$ is repeated across
all demonstrations and satisfies $p(\widetilde\vi\mid t^*)=0$, so the ordinary
Bayes conditional at $\vz_N(\widetilde\vi)$ is not determined directly by the
pretraining distribution. Because all displayed occurrences are copies of one
shared instruction realization, its fixed finite continuation is performed
only once and introduces only an $N$-independent factor. At the same time,
evidence from the demonstrations that identifies the target task accumulates
with $N$. Therefore,
even though $p(\widetilde\vi\mid t^*)=0$, the theoretical continuation error
still decreases exponentially for this specified reference extension.
Appendix~\ref{app:uniform_local_transport} gives sufficient quantitative
conditions under which the transport error is small, controlled
study
in Appendix~\ref{app:empirical_transport_validation} complements this result.

\begin{remark}
This setting reflects a common mismatch between the notation represented during LLM
pretraining and the notation chosen by a user. A user may consistently express
a supported task using a symbol or style of language that never appeared in
pretraining. The theorem suggests that, provided the underlying latent task was
learned during pretraining and the new notation is used consistently across the
demonstrations, sufficiently many examples can still identify the target task
within the reference construction. When the model-transport error is also small,
the actual model can exhibit the corresponding behavior from the
input pattern rather than requiring that notation itself have been
observed during pretraining.
The continuation perspective is also relevant when
$p(\vi\mid t^*)>0$ but extremely small, since it may occur too rarely in a
finite pretraining corpus for the model to learn its association with $t^*$
reliably, making it effectively unsupported for the learned model.
\end{remark}

\subsubsection{Mixed and varying supported and unsupported instructions (format~6-7).}
\label{subsec:mixed_unsupported}
We next consider format~7, in which most demonstrations share a supported instruction while a
subset contain distinct unsupported corruptions. The supported occurrences act
as a common anchor, and each corrupted occurrence is connected to that anchor by
the same construction used in
Section~\ref{subsec:unsupported_shared}, which quantifies how strongly the
corruption can favor a competing task (see
Appendix~\ref{app:mixed_supported_unsupported}).

Intuitively, every informative demonstration supplies evidence that separates
the target task from its competitors, while each distinct unsupported
instruction can shift some of that evidence back toward a competing task.
Generalization therefore depends on whether the evidence accumulated from the
demonstrations grows faster than the total adverse effect of the corruptions.

Format~6 is the limiting case in which every demonstration has a
distinct unsupported instruction. The demonstrations provide accumulating
evidence for the target task, but the continuation costs of the unsupported
instructions also accumulate with $N$. Convergence remains possible when the
demonstration evidence grows faster than the total continuation cost. Thus the analysis gives no unconditional convergence
guarantee for format~6. This explains why increasing the number of
examples need not remove instruction sensitivity in this setting, consistent
with Figure~\ref{fig:instruction_inconsistent_main}.
Complete details are given in
Appendix~\ref{app:format6_varying_unsupported}.

\begin{remark}
These setting models typos in user-provided demonstrations. Some
demonstrations may use a supported instruction from the pretraining support,
whereas a typo may produce an instruction that the model never observed during
pretraining. Moreover, different typos need
not produce the same unsupported instruction, which distinguishes it from the repeated setting of
Section~\ref{subsec:unsupported_shared}. Their demonstration-dependent cost can still vanish when it grows
more slowly than the evidence supplied by the demonstrations, while the actual
model retains the displayed transport error.
\end{remark}
\vspace{-0.5em}
\section{ICL for task beyond the pretraining task support}
\label{sec:cot}
%

Language-model agents are often asked to carry out long-horizon tasks that may
not be represented as single latent tasks during pretraining. Such a target has
$\pi(t^\dagger)=0$ and therefore receives zero posterior mass for every prompt
and every $N$.
Consequently, the Bayesian latent-task reference for direct ICL can only form a
mixture of supported-task predictions, and increasing $N$ cannot make
this reference predictor select the unsupported target task. In this section,
we apply the preceding ICL analysis to an agent that follows a plan decomposing
such a task into supported subtasks, and our aim here is to
use the ICL framework to explain how support and model--reference errors behave
along agent execution, rather than to introduce a separate theory of agent
learnability.

\subsection{Executing Long-Horizon Tasks through Supported Subtasks}
\label{subsec:cot_main}
\label{subsec:cot_bound}

Agent execution circumvents the reference-level obstruction by never asking
the posterior to represent $t^\dagger$ as a single latent task. Consider a
$K$-step plan with supported subtasks $t_0,\ldots,t_{K-1}\in\gT$. Let
$x_n^{(k)}\in\mathcal S$ be the state of demonstration $n$ after step $k$,
where the state contains any execution history needed by the next step. We
assume the recursion $x_n^{(k+1)}=f_{t_k}(x_n^{(k)})$,
and define $f_{t^\dagger}(x)$ as the terminal state $x^{(K)}$ obtained by this
recursion from $x^{(0)}=x$. Thus
$f_{t^\dagger}=f_{t_{K-1}}\circ\cdots\circ f_{t_0}$. We reuse $\vz_N$ for the
full agent-execution prompt, whose $N$ demonstrations are the trajectories
$(x_n^{(0)},\ldots,x_n^{(K)})$. For step $k$ and a current state $x$, let
$\vz_N^{(k)}(x)$ denote the prompt formed by the supported instruction for
$t_k$, the $N$ transition pairs
$(x_n^{(k)},x_n^{(k+1)})$, and $x$. Starting from
$x^{(0)}=\xquery$, the agent produces $x^{(1)},\ldots,x^{(K)}$ sequentially,
with $\yquery=x^{(K)}$. The latent task of the step-$k$ prompt is therefore the
supported subtask $t_k\in\gT$. This is the decisive structural point: although
$\pi(t^\dagger)=0$, every step prompt represents a supported task with
$\pi(t_k)>0$.

\begin{assumption}[Step-wise supported identifiability]
\label{assp:cot_stepwise}
There are constants $\alpha_k,\beta_k>0$ independent of $N$ such that, almost
surely for all sufficiently large $N$ and every state $x$ reachable at step
$k$,
\begin{equation}
\label{eq:cot_step_bound}
1-p\bigl(f_{t_k}(x)\mid\vz_N^{(k)}(x)\bigr)
\le\beta_k e^{-\alpha_k N}.
\end{equation}
\end{assumption}
Note that Assumption~\ref{assp:cot_stepwise} is not an extra hypothesis about
$t^\dagger$; it is the component of
Theorem~\ref{thm:decay_1_5} applied $K$ times, together with
the requirement that the bound be uniform over the correct intermediate
states rather than at a single query.

\begin{theorem}[Informal: step-wise ICL bound for a decomposed long-horizon task]
\label{thm:cot_main}
\label{thm:cot_bound}
Let $f_{t^\dagger}=f_{t_{K-1}}\circ\cdots\circ f_{t_0}$ with
$t_0,\ldots,t_{K-1}\in\gT$, and let Assumption~\ref{assp:cot_stepwise} hold.
For each step $k$, let $\varepsilon_M^{(k)}(N)\ge0$ be its finite
model-transport error, defined formally in
Appendix~\ref{app:cot_main_proof}. Then,
almost surely for all sufficiently large $N$,
\begin{equation}
\label{eq:cot_main_bound}
\Bigl|M\bigl(\yquery\mid\vz_N\bigr)
-p(\yquery\mid\xquery,t^\dagger)\Bigr|
\le\sum_{k=0}^{K-1}\left[
\beta_k e^{-\alpha_k N}+\varepsilon_M^{(k)}(N)\right].
\end{equation}
\end{theorem}
This bound follows by applying the supported-task ICL guarantee at each step
and accumulating the resulting errors across the fixed execution horizon under
the explicit step-call protocol above. It connects a standard composition argument to
the support distinctions developed above. For each
step, $\alpha_k$ measures how quickly the demonstrations distinguish the
intended subtask from competing supported tasks, while $\beta_k$ collects the
$N$-independent difficulty of that step, including its prior, instruction and
query factors and the uniformity over reachable intermediate states. The bound
is strongest when $K$ is moderate, every $\alpha_k$ is bounded away from zero,
and the $\beta_k$ remain controlled. Its decreasing component is determined by
the hardest step through $\min_k\alpha_k$, whereas every step also contributes
its model-transport error. Thus one poorly identified or poorly
approximated subtask can dominate the entire chain.

The main-text bound uses an explicit step-call protocol: at step $k$, the
executor calls the model on exactly the extracted prompt
$\vz_N^{(k)}(x^{(k)})$ and retains
no additional model-visible history or cache.
Full-context autoregressive rollout requires additional rollout and,
when applicable, padding corrections, treated in
Appendix~\ref{app:cot_main_proof} and \ref{app:padding}.
\begin{remark}
In practice, this interpretation favors agent plans built from supported,
well-specified operations with clear intermediate states. Decomposing a
difficult long-horizon task into simpler supported steps can make each step
easier to identify and therefore improve its exponential rate. However, an
unnecessarily long plan creates more opportunities for error, and an ambiguous
instruction, an unsupported intermediate representation, or a weakly supported
subtask can make one $\alpha_k$ small or one $\beta_k$ large. A concise,
well-chosen plan can therefore outperform direct prediction, whereas a long or
poorly chosen execution plan can be less reliable.
\end{remark}

\subsection{Factors Governing Generalization Beyond Pretraining}
\label{subsec:pathwise_lipschitz}

\paragraph{The pathwise Lipschitz sensitivity.}

Recall the instruction path $E_s$ from
\eqref{eq:main_instruction_embedding_path}, which connects the supported
anchor $\vi_0$ to the unsupported instruction $\widetilde\vi$. Define
$\mathcal L_{\mathrm{path}}$ as the smallest constant such that, for every
competing task $t$ and every $s,s'\in[0,1]$,
\begin{equation}
\label{eq:path_lipschitz_main}
\left|
r_t(E_s)-r_t(E_{s'})
\right|
\le
\mathcal L_{\mathrm{path}}\|E_s-E_{s'}\|_F,
\end{equation}
where the reference log-odds along unsupported parts of the path are understood
through the continuation in Section~\ref{subsec:unsupported_shared}. The
resulting bound can be written schematically as
\begin{equation}
\label{eq:anchor_path_error_main}
\mathrm{error}
\;\lesssim\;
\beta\exp\!\left(
-\alpha N
+\mathcal L_{\mathrm{path}}
\|E(\widetilde\vi)-E(\vi_0)\|_F
\right).
\end{equation}
Thus a good anchor is not merely a point with a small
local derivative: it should make the combined anchor and whole-path cost small.
The value of
$\mathcal L_{\mathrm{path}}$ is determined by the pretrained model, the chosen
anchor and path, and the competing tasks. It tends to be smaller when
pretraining has learned nearby instructions, representations, or computations
that change smoothly toward the test instruction, and can be large when the
path crosses a poorly learned or task-ambiguous region. In experiments
(Section~\ref{sec:experiments}), we use the discrete gradient of the
correct-output log probability as the prompt changes with the demonstration
count as an observable proxy for effective prompt sensitivity. In the agent
experiment, changes in output log probability across execution steps provide a
qualitative diagnostic of step-wise stability.
\vspace{-0.4em}
\paragraph{The demonstrations sensitivity.}
The demonstrations etermine the evidence rate $\alpha$ in
\eqref{eq:anchor_path_error_main}: $\alpha$ is large when each additional
demonstration supplies information that distinguishes the target task from its
plausible competitors. Such demonstrations are mutually consistent with one
latent rule and include inputs on which competing rules predict different
outputs. By contrast, repetitive examples on which several tasks agree leave
the task identity ambiguous and produce a small $\alpha$. Good demonstrations
therefore make the favorable term $\alpha N$ grow rapidly enough to dominate
the fixed anchor and path penalties, and they reduce the sensitivity of
the final prediction.

For example, consider the prompt
\[
\texttt{Japan}\!\to\!\texttt{Tokyo},\quad
\texttt{France}\!\to\!\texttt{Paris},\quad
\texttt{USA}\!\to\ ?
\]
The first two pairs do not distinguish ``country to capital'' from ``country
to largest city,'' because the two rules agree on both demonstrations. For the
query, the two interpretations instead predict
\texttt{Washington DC} and \texttt{New York}. The demonstrations thus provide
little evidence against either competing task, corresponding to a small
$\alpha$ and a prediction that can change sharply under a small modification of
the prompt. In contrast, consider
\[
\texttt{Australia}\!\to\!\texttt{Sydney},\qquad
\texttt{Turkey}\!\to\!\texttt{Istanbul},\qquad
\texttt{USA}\!\to\ ?
\]
The first two pairs rule out the capital interpretation and strongly favor
the largest-city interpretation. The resulting larger $\alpha$ makes the prediction for
$\texttt{USA}\!\to\ ?$ more stable and improves the bound.
\vspace{-0.4em}
\paragraph{Applying the sensitivity view to agent execution.}
For an unsupported target task, the Bayesian reference for direct
ICL cannot select $t^\dagger$ from its pretraining task posterior. A decomposed
agent plan avoids this reference-level obstruction by using supported subtasks
$t_0,\ldots,t_{K-1}$. Applying Theorem~\ref{thm:decay_1_5} yields
two contributions: $\beta_k e^{-\alpha_kN}$ measures subtask
identification error, and $\varepsilon_M^{(k)}(N)$ measures model--reference
error. Additional demonstrations reduce the first term, while the second must
be controlled separately. The zero prior of the complete task is therefore no
longer decisive: the agent only needs to identify and execute the supported
steps. However, the errors add across the $K$ steps, so one weakly identified
or poorly approximated subtask can dominate the final result.


\begin{figure*}[t]
    \centering
    \begin{subfigure}
        [t]{0.32\textwidth}
        \centering
        \includegraphics[width=\textwidth]{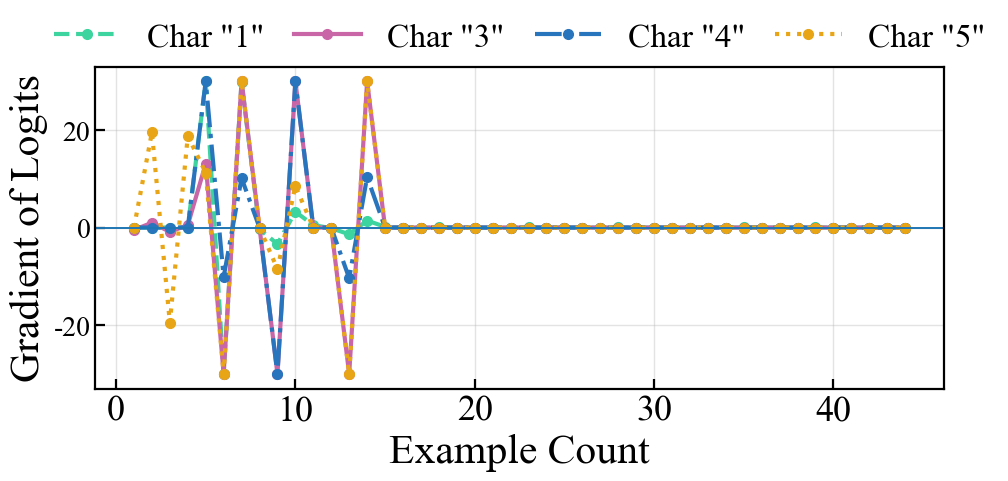}
        \caption{\textbf{Consistent mismatched instructions.} The discrete
        gradients of the correct-digit log probabilities diminish as the
        repeated demonstrations identify the intended task.}
        \label{fig:instruction_consistent_main}
    \end{subfigure}
    \hfill
    \begin{subfigure}
        [t]{0.32\textwidth}
        \centering
        \includegraphics[width=\textwidth]{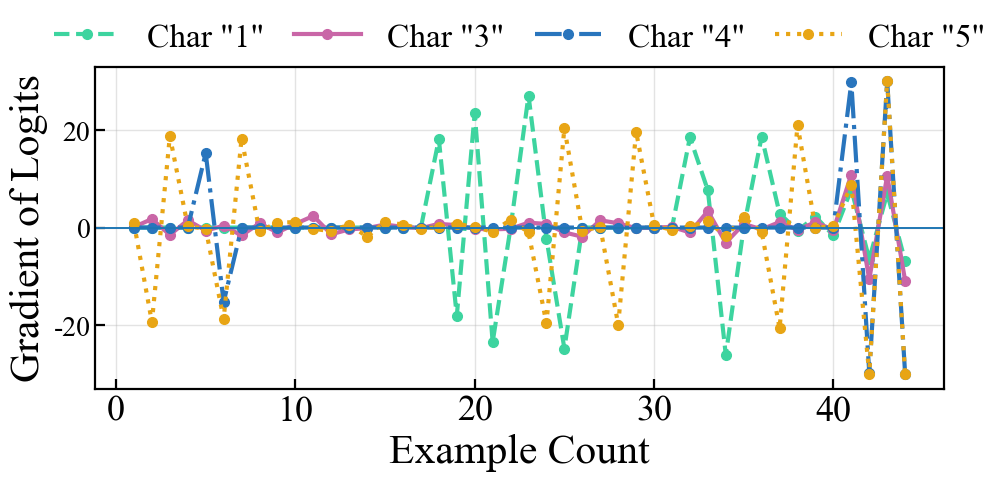}
        \caption{\textbf{Varying mismatched instructions.} The discrete
        gradients of the correct-digit log probabilities remain irregular when
        different instruction mismatches accumulate across demonstrations.}
        \label{fig:instruction_inconsistent_main}
    \end{subfigure}
    \hfill
    \begin{subfigure}
        [t]{0.32\textwidth}
        \centering
        \includegraphics[width=\textwidth]{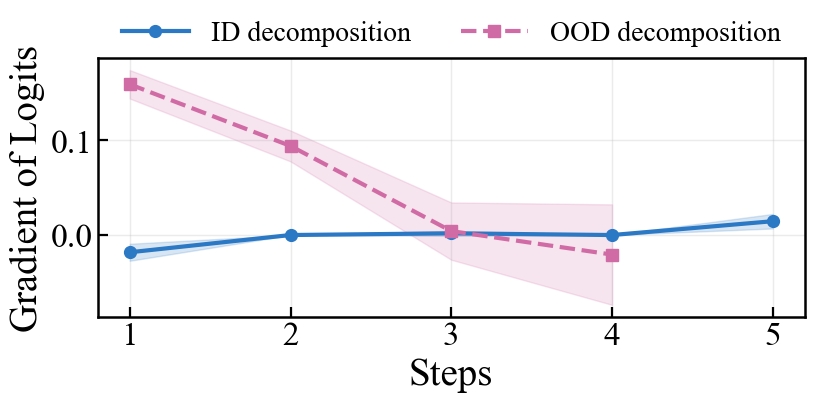}
        \caption{\textbf{Step-wise sensitivity under agent task decomposition.}
        The in-distribution execution plan has a smoother gradient profile,
        whereas the out-of-distribution plan exhibits larger fluctuations.}
        \label{fig:chain}
    \end{subfigure}
    \caption{Empirical analogues of (a) consistent and (b) varying instruction
    mismatch, together with (c) the role of supported, stable
    subtasks in Theorem~\ref{thm:cot_main} and
    Section~\ref{subsec:pathwise_lipschitz}. In (a)--(b), the fixed correct
    answer is \texttt{1345}; \emph{Char ``1''}, \emph{Char ``3''},
    \emph{Char ``4''}, and \emph{Char ``5''} denote the score-gradient traces
    for these correct digits at the first through fourth output positions,
    respectively.}
    \label{fig:experiment}
\end{figure*}
\vspace{-0.5em}
\section{Experiments}
\label{sec:experiments}
\vspace{-0.4em}
\subsection{Demonstration Evidence for Task Identification}
\label{subsec:experiment_general}
Section~\ref{subsec:pathwise_lipschitz} predicts that ICL performance under
prompt shift is governed by the demonstration
evidence rate $\alpha$. Here we give a motivational example
of the corresponding intuition: demonstrations that better distinguish the target task
from plausible competitors should produce more reliable task identification and higher ICL accuracy.

We use Sport Identification as the representative example here.
\textbf{(Setup)} Given a country, the task is to output its associated sport label. We construct
two types of demonstrations: \textit{identifying} demonstrations, where the sport
is a distinctive national or traditional sport that strongly pins down a unique
latent rule, and \textit{ambiguous} demonstrations, where the sport is related to
the country but can be explained by multiple plausible latent rules, thereby
providing weaker evidence about which rule the prompt instantiates. We sample
$k \in \{2,4,8\}$ demonstrations from the pool and evaluate on a
fixed test set. \textbf{(Results)}
As shown in Table~\ref{tab:demo_choice_acc}, the identifying pool consistently
yields higher accuracy than the ambiguous pool, which is consistent with
the prediction that informative demonstrations provide a larger evidence rate,
leading to stronger task identification and better
generalization. See Appendix~\ref{app:demo} for details, and see Appendix~\ref{app:controlled_evidence_rate}
for a controlled validation.

\begin{table}[t]
    \centering
    \caption{Test accuracy ($\uparrow$) of ICL under ambiguous and
    identifying demonstrations ($k=8$). 
    }
    \label{tab:demo_choice_acc} 
\small
\begin{tabular}{lcc}
    \toprule
    Demonstration & Qwen3-30B-A3B & Qwen3-235B-A22B \\
    \midrule
    Ambiguous / \textbf{Identifying}   & $0.164 \ensuremath{\pm}0.039 / \textbf{0.559 \ensuremath{\pm} 0.074}$           & $0.202 \ensuremath{\pm}0.037 / \textbf{0.602 \ensuremath{\pm} 0.043}$           \\
    \bottomrule
\end{tabular}
\vspace{-0.5em}
\end{table}
\vspace{-0.5em}
\vspace{-0.4em}
\subsection{Prompt Sensitivity vs. Demonstration Count}
\label{subsec:experiment_instruction} In Section~\ref{sec:instruction},
Theorems~\ref{thm:decay_1_5} and~\ref{thm:repeated_unsupported}, together with
Section~\ref{subsec:mixed_unsupported}, predict different scaling for fixed and
varying instruction shifts. We design the following experiment to test the
corresponding distinction.

\textbf{(Setup)} We use a pretrained \texttt{Qwen3-235B-A22B} model on a
three-digit addition task with a fixed query and target answer. We vary only
the operator pattern in the demonstrations: the \textit{addition-only} regime
uses $+$ throughout; the \textit{consistent-mismatch} regime uses the same
$\times$ symbol throughout while retaining sum outputs; and the
\textit{varying-mismatch} regime independently selects $\times$, $/$, or $-$
for each demonstration. These regimes provide an aligned baseline and empirical
analogues of consistent and varying instruction mismatch. Because the exact conditional
support of a pretrained LLM is not observable, the experiment tests the effect
of consistent versus varying instruction mismatch rather than claiming that
these symbols have literal zero pretraining probability. \textbf{(Results)} We
evaluate the model under the three prompt formats with 1 to 50 demonstrations. We
visualize the discrete gradient, as the demonstration count changes, of the
model-assigned log probability for each correct output digit
in Figures~\ref{fig:instruction_consistent_main}
and~\ref{fig:instruction_inconsistent_main}. In the addition-only regime, the
gradient remains near zero throughout. Under a consistent mismatch, sensitivity
is initially high but gradually decreases as the demonstrations identify
addition. Under varying mismatches, the gradients remain irregular and exhibit
large spikes even at high demonstration counts. This persistent sensitivity is
qualitatively consistent with the theoretical distinction between fixed and
varying instruction shifts. See
Appendix~\ref{app:experiment_instruction_variation} for details.

\vspace{-0.5em}
\subsection{Task Decomposition Quality in Controlled Agent Execution}
\label{subsec:experiment_cot}
\vspace{-0.5em}
We instantiate the multi-step execution setting of
Section~\ref{subsec:cot_main} with an arithmetic agent that follows a fixed
plan. The agent maintains an explicit state $x^{(k)}$ and updates it at step
$k$ by applying the prescribed operation $f_{t_k}$. We intervene on whether the operations in the plan are supported,
while holding the target task and execution horizon fixed. The experiment
therefore illustrates the subtask-identification mechanisms summarized by
Theorem~\ref{thm:cot_main}: the ICL view favors execution when the
decomposition consists of supported
subtasks with controlled step-wise sensitivity. A decomposition containing
unsupported subtasks does not receive the same guarantee.

\textbf{(Setup)} We perform two-stage LoRA training for
\texttt{Llama3.2-3B}. The base model already performs one- to five-digit
addition and simple multiplication reliably, and the first stage adds 20K
six-digit addition samples. We treat seven-digit addition as the composed
target task. For the
\textit{In-Distribution Agent Plan}, the agent executes six-digit addition,
one-digit addition, multiplication by 10, and aggregation as its component
operations, and we add 10K step-level training examples for each required
operation. The \textit{Out-of-Distribution Agent Plan} instead requires the
agent to perform intermediate seven-digit additions. These operations receive
no additional training and cannot be executed reliably in isolation. We
evaluate the same model under both plans.
\textbf{(Evaluation)} At every execution step, we record whether the state
transition is correct and the log probability assigned to the correct output.
Variation in the latter across steps is a qualitative diagnostic of step-wise
stability. Since adjacent steps generally implement different operations, this
quantity is not a direct estimate of the within-step path constant in
Section~\ref{subsec:pathwise_lipschitz}.
\textbf{(Results)} Figure~\ref{fig:chain} shows that the in-distribution
agent plan has a smoother step-wise log-probability profile. The
out-of-distribution plan exhibits larger fluctuations and performs worse than
direct prediction without agent execution, illustrating the cost of including
a poorly supported operation in the plan. Additional results are provided in
Appendix~\ref{app:cot_effect}.


\vspace{-0.5em}
\section{Conclusion}
\label{sec:conclusion}
\vspace{-0.5em}
In this work, we separate two forms of generalization in in-context learning (ICL)
beyond pretraining support. At the instruction level, supported instructions
yield exponential convergence, and the continued reference predictor preserves
this rate for a single or repeated unsupported instruction. Mixed
prompts also converge when their total
continuation burden grows more slowly than the demonstration evidence.
At the task level, the
Bayesian latent-task reference for
direct ICL assigns zero posterior mass to a target outside its supported task
set. Within the same reference model, applying the ICL perspective step by
step explains how an agent may execute such a long-horizon target without
selecting it as a single latent task, provided that a plan decomposes it into
supported, identifiable subtasks. The resulting, standard composition bound
accumulates the errors of the execution steps. Together, the instruction
results and this agent interpretation clarify when more demonstrations, a
different prompt format, or a task decomposition may extend model behavior
beyond the tasks directly represented in pretraining.

%% file: NeurIPS2026/appendix/list_notation.tex
\label{app:notation}
This appendix summarizes the notation used in the main text. Symbols that are
local to a single proof are introduced where they first appear.

\paragraph{Sequences, tasks, and data.}
\begin{itemize}[label=\raisebox{0.25ex}{\tiny$\bullet$},leftmargin=1.4em]
    \item $\gV$: the finite token vocabulary; $\mathcal S$ is the finite
    space of individual inputs, outputs, and intermediate states.

    \item $\gT$: the finite set of latent tasks supported by pretraining;
    $\pi(t)$ is the pretraining prior on $t\in\gT$.

    \item $t^*$: the target task in Section~\ref{sec:instruction}, where
    $t^*\in\gT$; $t^\dagger$ is the target task in
    Section~\ref{sec:cot}, where $t^\dagger\notin\gT$ and
    $\pi(t^\dagger)=0$.

    \item $(x,y)$: an input--output pair; $f_t$ is the task-specific mapping,
    so that $y=f_t(x)$ under a deterministic task $t$.

    \item $\xquery$ and $\yquery$: the query input and its output,
    respectively.

    \item $P(\cdot)$ and $p(\cdot)$: a distribution and its probability mass
    or density function. In particular, $p(t\mid\vz)$ and
    $p(\yquery\mid\vz)$ denote the latent-task posterior and posterior
    predictive distribution.

    \item $M(y\mid\vz)$: the output probability assigned by the actual
    pretrained model. The Bayesian quantities $p$ and $\bar p$ below are its
    theoretical reference predictors.

    \item $\dpre$: the support of the pretraining prompt distribution across
    all demonstration counts.
\end{itemize}

\paragraph{Prompts and instructions.}
\begin{itemize}[label=\raisebox{0.25ex}{\tiny$\bullet$},leftmargin=1.4em]
    \item $\vz$ or $\vz_N$: an ICL prompt containing $N$ demonstrations and
    one query. The seven prompt formats are defined in
    Section~\ref{subsec:instruction_setup}.

    \item $\vi$: an instruction sequence; $\widetilde\vi$ denotes an
    instruction unsupported for the target task, and $\vi_0$ denotes a
    supported anchor instruction. The notation $\boldsymbol\vi$ collects all
    instruction variables in a prompt when a format contains more than one.

    \item $\dicl$: the supported instruction domain used in the supported
    instruction results of Section~\ref{sec:instruction}.

    \item $L_i$: the fixed maximum instruction length, and
    $\gV^{\le L_i}:=\bigcup_{m=0}^{L_i}\gV^m$.

    \item $e:\gV\cup\{\langle\mathrm{pad}\rangle\}\to\mathbb R^d$: the
    token embedding map, with
    $e(\langle\mathrm{pad}\rangle)=0$.

    \item $E(\vi)\in\mathbb R^{L_i\times d}$: the right-padded embedding
    matrix of instruction $\vi$ defined in
    Section~\ref{subsec:instruction_setup}.

    \item $E_s=(1-s)E(\vi_0)+sE(\widetilde\vi)$, $s\in[0,1]$: the straight
    path in the padded instruction-embedding space from a supported anchor to
    an unsupported instruction.

    \item $r_t(\vi):=\log\!\bigl(p(\vi\mid t)/p(\vi\mid t^*)\bigr)$: the
    task-relative instruction score on supported instructions. The notation
    $r_t(U)$, $U\in\mathbb R^{L_i\times d}$, denotes its fixed model-induced
    extension to instruction-embedding matrices for the continuous-path
    analysis.

    \item $\bar p$: the posterior or posterior predictive induced by a fixed
    analytic continuation at a zero-probability prompt. Ordinary $p$ is used
    when the relevant conditional is determined directly on the support.
\end{itemize}

\paragraph{Instruction-level convergence.}
\begin{itemize}[label=\raisebox{0.25ex}{\tiny$\bullet$},leftmargin=1.4em]
    \item $\alpha>0$: a demonstration-evidence rate separating the target
    task from its competitors; $\beta>0$ denotes a generic
    $N$-independent prefactor in an exponential error bound.

    \item $\varepsilon_M$: a finite model-approximation or transport error
    between the actual model $M$ and the theoretical reference predictor. For
    a shared unsupported instruction, $\varepsilon_{M,N}(\widetilde\vi)$ is
    the $N$-dependent error along the tied shared-template path, while
    $\overline\varepsilon_M(\widetilde\vi)$ denotes its optional uniform
    upper bound. The pathwise definitions are given in
    Appendix~\ref{app:model_reference_transport}.

    \item $\mathcal L_{\mathrm{path}}$: the pathwise Lipschitz sensitivity of
    the continued instruction log-likelihood ratios, defined in
    \eqref{eq:path_lipschitz_main}.

    \item $\|\cdot\|_F$ denotes the Frobenius norm. In particular,
    $\|E(\widetilde\vi)-E(\vi_0)\|_F$ is the distance from the supported
    anchor to the unsupported instruction. Appendix~\ref{app:padding} also
    denotes this distance by $\Delta_{\mathrm{inst}}$.

    \item $U_N\subseteq[N]$: the indices of demonstrations containing
    unsupported instructions in the mixed setting; $\Omega_n$ is the finite
    continuation cost associated with the $n$-th such instruction.
\end{itemize}

\paragraph{Agent execution.}
\begin{itemize}[label=\raisebox{0.25ex}{\tiny$\bullet$},leftmargin=1.4em]
    \item $K$: the fixed number of execution steps; $t_0,\ldots,t_{K-1}\in\gT$
    are the supported subtasks whose composition realizes
    $t^\dagger$.

    \item $x_n^{(k)}$ and $x^{(k)}$: the states after step $k$ in
    demonstration $n$ and in the query execution, respectively. In particular,
    $x^{(0)}=\xquery$ and $x^{(k+1)}=f_{t_k}(x^{(k)})$.

    \item $\vz_N^{(k)}(x)$: the step-$k$ ICL prompt evaluated at state $x$;
    $\vz_N$ denotes the full agent-execution prompt.

    \item $\alpha_k,\beta_k>0$: the exponential convergence constants for
    supported subtask $t_k$.

\end{itemize}

\paragraph{Padding.}
\begin{itemize}[label=\raisebox{0.25ex}{\tiny$\bullet$},leftmargin=1.4em]
    \item $\operatorname{Pad}_M(\vz)$: the position-preserving extension of
    prompt $\vz$ to a fixed length $M$ using
    $\langle\mathrm{pad}\rangle$ tokens.

    \item $m$: the number of added padding positions;
    $\Delta_{\mathrm{att}}$ is the attention-logit gap between a selected
    informative position and the most competitive visible padding position.

    \item $\varepsilon_{\mathrm{pad}}$: the model-output log-probability
    perturbation caused by padding under the conditions of
    Proposition~\ref{prop:padding_transformer_output}.

    \item $\delta_{k,\mathrm{pad}}(N)$: the conditional-probability discrepancy
    between the original and padded step-$k$ execution prompts. It is zero when the
    added positions are strictly future positions removed by the causal mask.
\end{itemize}

%% file: NeurIPS2026/appendix/related_work.tex
\paragraph{Existing theories rely on strong assumptions.}
A substantial body of theoretical work seeks to understand the mechanisms
underlying ICL by analyzing how transformers implement learning procedures over
context. At the mechanistic level, studies identify specific architectural
patterns such as induction heads and function vectors—that enable example reuse
and pattern matching across the context \citep{Olsson2022, Elhage2021, todd2024}.
Complementary algorithmic analyses show that, under linear or kernelized assumptions,
ICL can be interpreted as implicitly gradient descent or kernel regression in
activation space \citep{Akyurek2023, vonOswald2023, Ahn2023, Mahankali2023}, or
as a form of task selection induced by the prior over pretraining tasks \citep{Bai2023transformers}.
While these mechanistic and algorithmic perspectives have significantly advanced
our understanding of how ICL is computed, they typically rely on restrictive architectural
assumptions, such as linear attention, single-layer or attention-only transformers,
or single-head attention—and on idealized data distributions such as random
generalized linear regression, which abstract away properties that are critical in
practice \citep{vonOswald2023, Mahankali2023, JMLR:v25:23-1042}. A complementary
line of work adopts a Bayesian or statistical inference perspective,
interpreting ICL as approximate posterior inference over latent tasks
conditioned on in-context demonstrations \citep{Xie2022, Wang2023(b), Panwar2024}.
This framing provides a principled account of how generalization improves with
more demonstrations, and has been extended to study the effects of noise,
misspecification, and pretraining bias on ICL behavior \citep{Gong2025towards}. However,
existing analyses within this framework uniformly assume a fixed and canonical
context format, treating demonstrations as drawn from a coherent, well-specified
data-generating process. As a result, they do not characterize how ICL behaves
when the format of the context itself varies—a limitation that our work directly
addresses.

\paragraph{Practical factors affect ICL generalization.}
Beyond mechanistic and statistical accounts, a large body of empirical work
investigates how practical design choices in prompt construction affect ICL
performance. A first line of work focuses on the role of demonstrations themselves:
studies show that the number of demonstrations, their selection strategy, and
the correctness of their labels each substantially influence downstream
performance \citep{Agarwal2024, Liu2022, su2023selective}. Notably, \citet{Min2022}
find that label correctness in demonstrations matters less than expected, suggesting
that ICL generalizes by retrieving a latent task structure rather than learning directly
from individual examples. A second line of work examines chain-of-thought
prompting, showing that including intermediate reasoning steps in demonstrations
can dramatically improve performance on complex tasks \citep{Wei2022, kojima2022, nye2021}.
Subsequent theoretical work explains CoT's benefit in terms of expressivity: CoT
enables deeper computation without additional layers \citep{Feng2023}, expands
the model's recognition capacity for complex problem classes \citep{Merrill2024},
or supports more serial computation than standard forward passes allow \citep{Li2024}.
However, these expressivity-based accounts do not characterize how CoT-structured
demonstrations affect ICL generalization as the number of reasoning steps grows,
which is a gap our work addresses through a multi-step error propagation
analysis. A third line of work studies sensitivity to prompt format and template
design, finding that surface-level variations in instruction wording, demonstration
ordering, and label presentation can lead to large performance swings \citep{ye2024investigating, Zhao2021, Webson2022}.
Despite the breadth of these empirical findings, they remain largely descriptive:
existing work does not provide a theoretical criterion for when format variation
preserves successful generalization and when it fundamentally disrupts
consistent task inference. Our work provides such a criterion by formally distinguishing
prompt formats under which the posterior predictive converges exponentially with
more demonstrations from those under which convergence necessarily fails.

\paragraph{Mechanistic explanations of ICL.}
A complementary line of work tries to explain ICL at the level of internal
representations and computational mechanisms, examining how specific architectural
components give rise to in-context behavior. At the representational level, studies
identify induction heads and function vectors as key circuits that enable pattern
matching and example reuse across the context window \citep{Olsson2022, Elhage2021, todd2024},
and related work shows that multi-head softmax attention can implement regression
behaviors in simplified settings \citep{He2025incontext, Li2025provable}. At the
algorithmic level, several studies argue that transformers implicitly perform
gradient descent or kernel regression over in-context examples \citep{Akyurek2023, vonOswald2023, Ahn2023},
and related perspectives characterize ICL as algorithm selection or prior-induced
regularization across tasks \citep{Bai2023transformers, Lu2025transformer}. A
separate line of work examines whether standard pretraining objectives can
produce the representational capacity assumed by these accounts, establishing conditions
under which pretrained transformers exhibit ICL-like behavior \citep{kimandsuzuki2024transformers, kim2024transformers}.
While these mechanistic and algorithmic explanations have significantly advanced
our understanding of the computational substrate of ICL, they characterize how
ICL is implemented within a fixed architecture and format but do not address when
and why ICL generalizes across varying prompt formats, nor how practical design
choices such as demonstration count, CoT structure, or instruction variation
affect the quality of task inference. Our work is therefore complementary to this
line, operating at the level of generalization theory rather than mechanistic
realization.

%% file: appendix_entropy.tex
\label{app:probability_models}

This appendix explains how the latent-task probabilities used in the analysis
can be related to a Transformer evaluated on continuous instruction embeddings. The
argument has three logically distinct parts. Population cross-entropy
identifies the Bayes predictor on the support of the pretraining distribution;
an analytic Transformer scoring map selects one particular extension away from
that support; and a separate model-transport term measures the discrepancy
between this reference and the actual pretrained model. In particular, the
off-support extension is supplied by the model architecture and is not claimed
to be statistically identified by pretraining data of zero probability.

\subsection{What population cross-entropy identifies}
\label{app:cross_entropy_bayes}

Under the latent-task framework of Section~\ref{subsec:instruction_setup},
pretraining learns a conditional predictor $q(\cdot\mid\vz)$ by minimizing
population log-loss
\begin{equation}
    \mathcal R(q)
    :=\mathbb E_{(\vz,y)\sim P}
      \bigl[-\log q(y\mid\vz)\bigr].
    \label{eq:population_ce}
\end{equation}
For every predictor whose probabilities are positive on the support of $P$,
\begin{equation}
    \mathcal R(q)
    =\mathcal R(p)
     +\mathbb E_{\vz}
       \operatorname{KL}\!\left(
          p(\cdot\mid\vz)\,\|\,q(\cdot\mid\vz)
       \right).
    \label{eq:ce_kl_decomposition}
\end{equation}
Consequently, the population minimizer equals the Bayes posterior predictive
almost surely on the pretraining-prompt support:
\begin{equation}
    q^*(y\mid\vz)
    =p(y\mid\vz)
    =\sum_{t\in\gT}p(y\mid\xquery,t)p(t\mid\vz),
    \qquad P_{\vz}\text{-a.s.}
    \label{eq:posterior_predictive}
\end{equation}
Equation~\eqref{eq:ce_kl_decomposition} also makes the limitation precise:
population risk places no restriction on $q(\cdot\mid\vz)$ when
$P_{\vz}(\vz)=0$. Thus cross-entropy optimality by itself neither defines nor
identifies a predictor at an unsupported prompt.

For finite data, finite model capacity, or imperfect optimization, exact
equality in \eqref{eq:posterior_predictive} need not hold. This discrepancy is
kept explicit through the supported-model gap $r_0$ in
\eqref{eq:app_supported_model_gap}; it is not absorbed into the analytic
continuation assumption.

\subsection{A Transformer-induced reference extension}
\label{app:transformer_reference_extension}

We distinguish the following objects throughout the paper:
\begin{enumerate}[label=(\roman*),leftmargin=*]
    \item $p$ is the latent-task data-generating distribution. Its symbolic
    instruction probabilities $p(\vi\mid t)$ have their ordinary statistical
    meaning on the pretraining support.
    \item $p_{\mathrm{ref}}$ is a hybrid Bayesian reference predictor. It
    agrees with the Bayes predictor on supported prompts and, on unsupported
    instruction embeddings, inserts fixed Transformer-induced instruction log-odds into
    the latent-task posterior while retaining the theoretical demonstration,
    prior, query, and task-conditional output factors.
    \item $M$ is the actual pretrained model. It need not equal
    $p_{\mathrm{ref}}$; their discrepancy is controlled by
    $\varepsilon_M$ in Appendix~\ref{app:model_reference_transport}.
\end{enumerate}

Only instruction log-likelihood ratios, rather than a Lebesgue density over
the entire instruction-embedding space, are required by the proofs. It is therefore both
sufficient and conceptually cleaner to extend these ratios directly. Let a
fixed Transformer-induced scoring map produce task scores
\begin{equation}
    a_t(U),\qquad t\in\gT,
    \qquad U\in\mathbb R^{L_i\times d},
    \label{eq:reference_task_scores}
\end{equation}
and define the reference instruction log-odds relative to $t^*$ by
\begin{equation}
    r_t(U)
    :=a_t(U)-a_{t^*}(U)+c_t,
    \qquad t\ne t^*,
    \label{eq:reference_instruction_log_odds}
\end{equation}
where $c_t$ is independent of $U$. Equivalently, if the scores define
$q_{\mathrm{ref}}(t\mid U)=\operatorname{softmax}(a(U))_t$, then
\begin{equation}
    \log\frac{q_{\mathrm{ref}}(t\mid U)}
                   {q_{\mathrm{ref}}(t^*\mid U)}
    =a_t(U)-a_{t^*}(U).
    \label{eq:softmax_score_odds}
\end{equation}
The softmax normalizer cancels, so the extended log-odds inherit the regularity
of the task scores directly.

The required connection to the latent-task probabilities is a supported
calibration condition: on every supported symbolic instruction used by the
analysis,
\begin{equation}
    r_t(E(\vi))
    =\log\frac{p(\vi\mid t)}{p(\vi\mid t^*)}.
    \label{eq:reference_supported_calibration}
\end{equation}
For example, if $q_{\mathrm{ref}}(t\mid E(\vi))=p(t\mid\vi)$ on the
supported domain, Bayes' rule gives
\begin{equation}
    \log\frac{p(\vi\mid t)}{p(\vi\mid t^*)}
    =\log\frac{p(t\mid\vi)}{p(t^*\mid\vi)}
     -\log\frac{\pi(t)}{\pi(t^*)},
    \label{eq:bayes_instruction_odds}
\end{equation}
so \eqref{eq:reference_supported_calibration} follows by taking
$c_t=-\log(\pi(t)/\pi(t^*))$. More generally,
\eqref{eq:reference_supported_calibration} is the factor-level calibration
condition associated with the latent-task factorization used in the proofs.
Output-level cross-entropy optimality in
\eqref{eq:posterior_predictive} motivates this idealized calibration but does
not, by itself, identify a particular latent decomposition; the latter is part
of the reference-model specification.

The Transformer-induced scoring map is not claimed to be statistically
identified by next-token cross-entropy. It is a concrete architecture-induced
reference extension used to instantiate and quantitatively control the
otherwise arbitrary finite off-support perturbation. Supported calibration is
stated explicitly in \eqref{eq:reference_supported_calibration} as a separate
condition.

At an unsupported $U$, Equation~\eqref{eq:reference_instruction_log_odds}
defines a model-selected value. We do not interpret it as the probability of
observing a continuous embedding under the original discrete pretraining
distribution. This distinction resolves the apparent conflict between
$p(\widetilde\vi\mid t^*)=0$ and the fact that a Transformer nevertheless
returns well-defined scores when fed $E(\widetilde\vi)$ or an interpolated
instruction embedding.

Thus $r_t(U)$, rather than either individual factor
$p(U\mid t)$, is the continuous object used in the analysis. It agrees with
the ordinary likelihood ratio on the supported symbolic points by
\eqref{eq:reference_supported_calibration} and is defined by the scoring map
off that support. For example, in the shared-template
setting the continuation-induced task posterior is determined by the odds
\begin{equation}
    \frac{p_{\mathrm{ref}}(t\mid\vz_N(U))}
         {p_{\mathrm{ref}}(t^*\mid\vz_N(U))}
    :=B_t\exp\!\left(r_t(U)+\lambda_t\right),
    \qquad t\ne t^*,
    \label{eq:reference_continued_task_odds}
\end{equation}
followed by normalization over the finite task set. Here $B_t$ is the
demonstration likelihood ratio and $\lambda_t$ contains the prior and query
factors, as defined in \eqref{eq:def_Bt} and \eqref{eq:chain_g_lambda}. The
associated reference output predictor is
\begin{equation}
    p_{\mathrm{ref}}(y\mid\vz_N(U))
    :=\sum_{t\in\gT}p(y\mid\xquery,t)
                         p_{\mathrm{ref}}(t\mid\vz_N(U)).
    \label{eq:reference_continued_predictive}
\end{equation}
Thus the Transformer supplies exactly the single off-support factor needed by
the Bayesian concentration proof; the supported demonstration, prior, and
query factors retain their original probabilistic meaning.

\subsection{Analyticity of standard Transformer components}
\label{app:analytic_transformer_components}

Fix the sequence length, positional convention, attention mask, and
computational graph, and regard the instruction embedding $U$ as the only
continuous input. Affine maps, matrix products, residual additions, and fixed
positional embeddings are analytic. The exponential, hyperbolic tangent,
error function, exact GELU, and SiLU are real analytic. Hence attention
softmax and an output softmax are real analytic at every real input, because
their normalizing denominators are strictly positive on the real domain.

Layer normalization with a fixed $\epsilon>0$ has the form
\begin{equation}
    \operatorname{LN}(x)
    =\gamma\odot
      \frac{x-\mu(x)}{\sqrt{\sigma^2(x)+\epsilon}}+\beta.
    \label{eq:analytic_layer_norm}
\end{equation}
Since $\sigma^2(x)+\epsilon\ge\epsilon>0$ for real $x$, it is real analytic
on the entire real domain. The same argument applies to RMS normalization
with a strictly positive stabilizer. A causal mask is harmless when it is
fixed: masked entries are omitted from the corresponding finite softmax,
rather than treated as a variable arithmetic input.

\begin{proposition}[Analytic Transformer scoring map]
\label{prop:analytic_reference_transformer}
Suppose that the fixed Transformer scoring map in
\eqref{eq:reference_task_scores} has finite depth and width and is constructed
from the analytic operations above. Suppose also that every normalization
uses a strictly positive stabilizer and that the sequence length, masks, and
computational graph remain fixed. Then each map
$U\mapsto a_t(U)$, and hence each
$U\mapsto r_t(U)$, is real analytic on
$\mathbb R^{L_i\times d}$.
\end{proposition}

\begin{proof}
Each hidden state is obtained from the preceding hidden states by finitely many
sums, products, quotients with nonzero real denominators, and compositions of
real-analytic functions. These operations preserve real analyticity. An
induction over the finitely many layers proves that every hidden state and
output score is real analytic in $U$. Equation
\eqref{eq:reference_instruction_log_odds} then gives the result for the
reference log-odds.
\end{proof}

Proposition~\ref{prop:analytic_reference_transformer} is an architectural
statement, not an assumption that follows from data.

\subsection{From architectural analyticity to the finite continuation chain}
\label{app:analytic_transformer_chain}

Let
\begin{equation}
    E_s=(1-s)E(\vi_0)+sE(\widetilde\vi),
    \qquad s\in[0,1],
    \label{eq:probability_model_instruction_path}
\end{equation}
be the path used in Theorem~\ref{thm:repeated_unsupported}. The following result
shows that the finite-chain hypothesis is automatic for the scoring-map-induced
extension above.

\begin{proposition}[Analytic Transformer implies a finite continuation chain]
\label{prop:transformer_implies_finite_chain}
Under Proposition~\ref{prop:analytic_reference_transformer}, for every
$t\ne t^*$ the function
\begin{equation}
    s\longmapsto r_t(E_s)
    \label{eq:pathwise_reference_log_odds}
\end{equation}
is real analytic on an open real neighborhood of $[0,1]$ and admits a finite
compatible chain of holomorphic continuations of the form required by
Assumption~\ref{ass:chain_local}. The chain is
independent of the number $N$ of demonstrations.
\end{proposition}

\begin{proof}
The map $s\mapsto E_s$ is affine and
$U\mapsto r_t(U)$ is real analytic, so their composition in
\eqref{eq:pathwise_reference_log_odds} is real analytic on an open
neighborhood of every point of $[0,1]$. Every real-analytic germ has a local
holomorphic complexification. These complexifications agree wherever their
domains overlap along the real interval, by uniqueness of analytic
continuation. The resulting complex neighborhoods form an open cover of the
compact interval $[0,1]$. Compactness gives a finite subcover; choosing a
sufficiently fine increasing sequence of real centers
$0=s_0<\cdots<s_L=1$ and disks inside this subcover yields radii $r_\ell$ with
$s_{\ell+1}-s_\ell<r_\ell$ and compatible holomorphic functions on the disks,
exactly as in Assumption~\ref{ass:chain_local}. The function, path, and finite
cover concern a single shared instruction template and do not depend on $N$.
\end{proof}

Complexifications of softmax, SiLU, and normalization may have zeros, poles,
or branch points away from the real domain. This does not create an additional
existence assumption in Proposition~\ref{prop:transformer_implies_finite_chain}:
at every real point of the path, the real-analytic computation supplies some
complex neighborhood free of such singularities, and compactness reduces
these neighborhoods to a finite chain. Nearby complex singularities do,
however, shrink the admissible radii and can enlarge the Cauchy constants
$M_\ell$ and the cumulative continuation cost $\Omega$ defined in
Appendix~\ref{app:repeated_unsupported_chain}. Thus architectural analyticity
guarantees existence of the chain, whereas quantitative separation from the
complex singular set governs the strength of the resulting bound.

For an abstract off-support log-odds extension not parameterized by such a
Transformer scoring map, none of the preceding architectural argument is available; in
that case the finite-continuation condition remains a substantive regularity
hypothesis. The main theorem requires only finite, $N$-independent endpoint
log-odds for the chosen extension; Proposition~\ref{prop:transformer_implies_finite_chain}
supplies a concrete model-induced construction and quantitative control of
those endpoint values.

\subsection{The actual model and the transport error}
\label{app:probability_model_transport_discussion}

The construction does not assume that the actual pretrained model $M$ equals
the idealized reference $p_{\mathrm{ref}}$. On supported prompts their initial
discrepancy is $r_0$ from \eqref{eq:app_supported_model_gap}. Along a compact
prompt path $\gamma$, Appendix~\ref{app:model_reference_transport} gives
\begin{equation}
    \left|M(y\mid\gamma(1))
          -p_{\mathrm{ref}}(y\mid\gamma(1))\right|
    \le
    \underbrace{r_0+(L_M+L_{\mathrm{ref}})
                    \operatorname{len}(\gamma)}_{
                    \varepsilon_M(\gamma)}.
    \label{eq:probability_model_transport_recall}
\end{equation}
The analytic continuation controls the reference term in the main theorems;
$\varepsilon_M$ records the logically separate approximation and transport
error for the actual model. Analyticity guarantees that the displayed
Lipschitz constants are finite on a fixed compact path, but does not guarantee
that they, $r_0$, or $\varepsilon_M$ are small.
Appendix~\ref{app:uniform_local_transport} gives an additional quantitative
result: supported-anchor calibration and length-independent network stability
imply the local certificate $\bar r+(C_M+2G)d$, uniformly over the specified
length family. Appendix~\ref{app:anchor_loss_coverage} relates calibration to
population excess loss under an explicit coverage condition, and
Appendix~\ref{app:residual_variation_transport} gives a sharper certificate
using variation of the model--reference residual. These sufficient conditions
are separate from the existence of analytic continuation.

When the reference predictor can be evaluated, this discrepancy can also be
studied empirically. For a grid $0=u_0<\cdots<u_K=1$, define
\begin{equation}
    \widehat\varepsilon_M
    :=\max_{0\le k\le K}\sup_y
       \left|M(y\mid\gamma(u_k))
             -p_{\mathrm{ref}}(y\mid\gamma(u_k))\right|.
    \label{eq:empirical_model_gap}
\end{equation}
If the largest path-segment length
$\max_k\operatorname{len}(\gamma|_{[u_k,u_{k+1}]})$ is $\delta$ and the two
predictors have pathwise Lipschitz constants $L_M$ and
$L_{\mathrm{ref}}$, then every path point is within path length at most
$\delta$ of a grid point and
\begin{equation}
    \sup_{s\in[0,1]}\sup_y
       \left|M(y\mid\gamma(s))
             -p_{\mathrm{ref}}(y\mid\gamma(s))\right|
    \le \widehat\varepsilon_M
       +(L_M+L_{\mathrm{ref}})\delta.
    \label{eq:grid_model_gap_certificate}
\end{equation}
Without computable Lipschitz bounds, the grid experiment supplies evidence
rather than a uniform certificate. Without access to an evaluable reference
predictor, $\varepsilon_M$ remains a theoretical approximation term and cannot
be identified from ordinary test accuracy alone.

\subsection{Empirical validation of the model--reference approximation}
\label{app:empirical_transport_validation}

We complement the preceding argument with a controlled experiment in which the
latent-task probabilities and the hybrid reference predictor are evaluable.
The purpose is not to test analyticity itself, which follows from the
architectural conditions of Proposition~\ref{prop:analytic_reference_transformer},
but to test the substantive quantitative premise behind the model-level bounds:
whether the actual Transformer remains close to the reference predictor along
the continuation path, and whether this discrepancy stays controlled as the
number of demonstrations increases.

\paragraph{Controlled setting.}
We use a finite latent-task benchmark with three modular-arithmetic tasks
(addition, subtraction, and multiplication). The task prior, supported
task-conditional instruction probabilities, demonstration likelihoods, and
task-conditional output distributions are known, so the factors $B_t$ and
$\lambda_t$ in \eqref{eq:reference_continued_task_odds} and the resulting
$p_{\mathrm{ref}}$ can be computed directly. A finite Transformer with analytic
components is trained by cross-entropy using supported instructions only, and
an analytic task-scoring map supplies $r_t$. We repeat the
experiment for five independently initialized models.

For a supported anchor $\vi_0$ and an instruction endpoint
$\widetilde\vi$, we evaluate 51 equally spaced points on
\begin{equation}
    E_s=(1-s)E(\vi_0)+sE(\widetilde\vi),
    \qquad s\in[0,1],
    \label{eq:empirical_transport_path}
\end{equation}
and use $N\in\{1,2,4,8,16,32,64\}$ demonstrations. To avoid confusion with
the working radii $\rho_\ell$ used in the continuation proof, denote normalized
endpoint distance by
\begin{equation}
    \delta_{\mathrm{emb}}
    :=\frac{\|E(\widetilde\vi)-E(\vi_0)\|_F}
    {\text{typical separation between supported instruction embeddings}}.
    \label{eq:normalized_embedding_distance}
\end{equation}
For each $N$, define the endpoint and pathwise model--reference gaps
\begin{align}
    d_N^{\mathrm{end}}
    &:=\sup_y\left|M(y\mid\gamma_N(1))
                  -p_{\mathrm{ref}}(y\mid\gamma_N(1))\right|,\label{eq:empirical_endpoint_gap}\\
    d_N^{\mathrm{path}}
    &:=\max_{s\in\{0,.02,\ldots,1\}}\sup_y
       \left|M(y\mid\gamma_N(s))
             -p_{\mathrm{ref}}(y\mid\gamma_N(s))\right|.
    \label{eq:empirical_path_gap}
\end{align}
Thus $d_N^{\mathrm{path}}$ is an empirical effective transport error, rather
than the analytic upper bound in
\eqref{eq:probability_model_transport_recall}.

\paragraph{Local-path discrepancies over the tested lengths.}
Table~\ref{tab:local_transport_validation} reports the complete trajectories
for two local paths. At $\delta_{\mathrm{emb}}=0.25$, the pathwise gap remains
between $0.058$ and $0.125$ over a 64-fold change in demonstration count. It is
$0.099$ at $N=64$, approximately $20\%$ below its value at $N=1$. Moreover,
the measured supported-anchor gap accounts for most of this discrepancy: across
the displayed values of $N$, the additional pathwise gap
$d_N^{\mathrm{path}}-r_0$ is only $0.007$--$0.037$. This provides evidence
of a small additional measured error on these short paths. Here the reported
$r_0$ is an empirical anchor counterpart, not an evaluation of the supremum
over the full supported domain in \eqref{eq:app_supported_model_gap}.

At $\delta_{\mathrm{emb}}=0.50$, the pathwise gap remains between $0.139$ and
$0.316$ and again ends below its $N=1$ value. Neither trajectory is monotone,
as expected for finite prompts and independently sampled demonstrations, but
neither exhibits systematic amplification with $N$. These are precisely the
local paths for which the model-level theorem is most informative: the
Bayesian term decreases exponentially while the empirical transport term
remains controlled.

\begin{table}[t]
    \centering
    \caption{\textbf{Model--reference discrepancy on local paths.} Mean
    $\pm$ standard deviation over five independently trained models. The
    empirical supported-anchor gap (denoted $r_0$ in this table) is common
    to the endpoint-distance conditions.}
    \label{tab:local_transport_validation}
    \small
    \setlength{\tabcolsep}{5pt}
    \begin{tabular}{c c c c}
        \toprule
        $N$ & $r_0$ &
        $d_N^{\mathrm{path}}$ ($\delta_{\mathrm{emb}}=.25$) &
        $d_N^{\mathrm{path}}$ ($\delta_{\mathrm{emb}}=.50$) \\
        \midrule
         1 & $0.088\pm0.046$ & $0.125\pm0.081$ & $0.253\pm0.170$ \\
         2 & $0.056\pm0.078$ & $0.069\pm0.087$ & $0.146\pm0.151$ \\
         4 & $0.064\pm0.082$ & $0.071\pm0.080$ & $0.139\pm0.113$ \\
         8 & $0.066\pm0.059$ & $0.099\pm0.084$ & $0.316\pm0.175$ \\
        16 & $0.035\pm0.024$ & $0.058\pm0.038$ & $0.182\pm0.127$ \\
        32 & $0.090\pm0.102$ & $0.115\pm0.127$ & $0.206\pm0.179$ \\
        64 & $0.071\pm0.045$ & $0.099\pm0.086$ & $0.192\pm0.157$ \\
        \bottomrule
    \end{tabular}
\end{table}

\paragraph{Distance controls the error floor.}
Table~\ref{tab:distance_transport_validation} expands the endpoint distance
from one quarter to four times the typical supported-instruction separation.
The comparison reveals a favorable and interpretable structure. For the two
local paths, the gap at $N=64$ is smaller than at $N=1$. At larger distances,
the gap increases between these two demonstration counts, but remains bounded
and the increase is modest relative to the change caused by moving the
endpoint farther from the supported anchor. For example, at
$\delta_{\mathrm{emb}}=4$, the gap changes from $0.525$ to $0.562$ as $N$
increases from $1$ to $64$. Across distances, the dominant pattern is thus a
distance-dependent error floor, as represented explicitly by
$\varepsilon_{M,N}(\widetilde\vi)$ in the theorem, rather than systematic
amplification with the number of demonstrations.

\begin{table}[t]
    \centering
    \caption{\textbf{Dependence on normalized endpoint distance.} Pathwise
    model--reference gap at the smallest and largest tested demonstration
    counts. Mean $\pm$ standard deviation over five independently trained
    models; the last column is the ratio of the two means.}
    \label{tab:distance_transport_validation}
    \small
    \begin{tabular}{c c c c}
        \toprule
        $\delta_{\mathrm{emb}}$ & $d_1^{\mathrm{path}}$ &
        $d_{64}^{\mathrm{path}}$ & $d_{64}^{\mathrm{path}}/d_1^{\mathrm{path}}$ \\
        \midrule
        $0.25$ & $0.125\pm0.081$ & $0.099\pm0.086$ & $0.799$ \\
        $0.50$ & $0.253\pm0.170$ & $0.192\pm0.157$ & $0.757$ \\
        $1.00$ & $0.323\pm0.270$ & $0.459\pm0.189$ & $1.421$ \\
        $2.00$ & $0.413\pm0.229$ & $0.543\pm0.081$ & $1.315$ \\
        $4.00$ & $0.525\pm0.057$ & $0.562\pm0.057$ & $1.072$ \\
        \bottomrule
    \end{tabular}
\end{table}

The endpoint and pathwise gaps are also nearly identical throughout the two
local conditions and for $\delta_{\mathrm{emb}}=1$ except at the smallest
demonstration count. Hence the maximum discrepancy is generally attained near
the unsupported endpoint, rather than at an unstable interior point of the
interpolation. This behavior is consistent with the finite-chain picture: the
path remains regular, while moving the endpoint farther from supported
instructions increases the distance-dependent difficulty of model transport.

\paragraph{What the experiment validates.}
The measured anchor and path-grid discrepancies support accurate local
tracking in this controlled setting over the tested demonstration counts.
Larger distances generally produce larger discrepancies. These are means and
standard deviations over five trained models, not worst-case bounds over
models, demonstration sequences, or the continuum of path points. Small
$d_N^{\mathrm{path}}$ does not establish that the analytic upper certificate
$\varepsilon_{M,N}(\widetilde\vi)$ is small: the latter uses separate
sensitivity bounds that the experiment does not certify. Nor does the tested
length grid establish an almost-sure statement for all sufficiently large
$N$ along one infinite demonstration sequence.

Proposition~\ref{prop:uniform_local_transport} identifies a sufficient regime
in which the analytic certificate is uniformly small, but its network
constants and calibration premises have not been certified for these trained
models. A continuous-path certificate additionally requires a valid
Lipschitz envelope in \eqref{eq:grid_model_gap_certificate}, or a certified
residual-variation bound from
Proposition~\ref{prop:residual_variation_transport}. Extending the diagnostic
to real models also requires an evaluable reference; task accuracy alone
does not identify the discrepancy. Thus the experiment provides finite-range
evidence compatible with the sufficient conditions, rather than verification
of their uniform hypotheses for general LLMs.

In summary, the setting makes two different commitments. Supported
calibration connects the reference scores to the latent-task probabilities at
prompts that pretraining can identify. Analytic Transformer architecture then
selects a well-defined, finite-chain extension of those scores along a
instruction-embedding path. The first commitment is statistical and support-restricted;
the second is architectural and off-support. Their separation is essential:
analyticity provides a principled extension, but does not turn zero-probability
prompts into observations from the pretraining distribution.

%% file: NeurIPS2026/appendix/minimal_modified_app_instruction.tex
%
%

\label{app:instruction}
\subsection{From the reference predictor to the pretrained model}
\label{app:model_reference_transport}

Let $M(y\mid\vz)$ denote the output probability of the actual pretrained
model, and let $p_{\mathrm{ref}}(y\mid\vz)$ denote the theoretical reference:
it equals the Bayesian predictor $p$ on supported prompts and the fixed
continued predictor $\bar p$ on unsupported prompts. Define
\begin{equation}
    \label{eq:app_supported_model_gap}
    r_0:=\sup_{\vz\in\dpre}\sup_y
    \left|M(y\mid\vz)-p(y\mid\vz)\right|.
\end{equation}
Consider a compact instruction-embedding path $\gamma$ from a supported prompt
$\gamma(0)$ to a test prompt $\gamma(1)$. Suppose $M(y\mid\gamma(s))$ and
$p_{\mathrm{ref}}(y\mid\gamma(s))$ are Lipschitz along this path, uniformly in
$y$, with constants $L_M$ and $L_{\mathrm{ref}}$, and let
$\operatorname{len}(\gamma)$ be its length. These constants are finite, for
example, if both functions are real analytic on a neighborhood of the compact
path. Define the model-transport error
\begin{equation}
    \label{eq:app_model_transport_error}
    \varepsilon_M(\gamma)
    :=r_0+(L_M+L_{\mathrm{ref}})\operatorname{len}(\gamma).
\end{equation}

\begin{lemma}[Transport from the supported domain]
\label{lmm:model_reference_transport}
Under the preceding conditions,
\begin{equation}
    \label{eq:app_model_reference_transport}
    \left|M(y\mid\gamma(1))-p_{\mathrm{ref}}(y\mid\gamma(1))\right|
    \le \varepsilon_M(\gamma).
\end{equation}
For a supported test prompt, the constant path gives
$\varepsilon_M=r_0$.
\end{lemma}

\begin{proof}
Insert and subtract the two functions at the supported endpoint:
\begin{align}
&\left|M(y\mid\gamma(1))-p_{\mathrm{ref}}(y\mid\gamma(1))\right| \\
&\quad\le
\left|M(y\mid\gamma(0))-p(y\mid\gamma(0))\right|
+\left|M(y\mid\gamma(1))-M(y\mid\gamma(0))\right| \\
&\qquad
+\left|p_{\mathrm{ref}}(y\mid\gamma(1))
-p_{\mathrm{ref}}(y\mid\gamma(0))\right| \\
&\quad\le r_0+(L_M+L_{\mathrm{ref}})\operatorname{len}(\gamma).
\end{align}
\end{proof}

For the shared-template setting, fix the instruction-coordinate path
\[
E_s=(1-s)E(\vi_0)+sE(\widetilde\vi),
\qquad
\gamma_N(s):=\vz_N(E_s).
\]
The copies of the repeated instruction in $\gamma_N(s)$ are tied to the same
coordinate $E_s$. Let $L_{M,N}$ and $L_{\mathrm{ref},N}$ satisfy, for all
$s,s'\in[0,1]$,
\begin{align}
\sup_y\left|M(y\mid\vz_N(E_s))-M(y\mid\vz_N(E_{s'}))\right|
&\le L_{M,N}\|E_s-E_{s'}\|_F, \\
\sup_y\left|p_{\mathrm{ref}}(y\mid\vz_N(E_s))
-p_{\mathrm{ref}}(y\mid\vz_N(E_{s'}))\right|
&\le L_{\mathrm{ref},N}\|E_s-E_{s'}\|_F.
\end{align}
Define
\begin{equation}
\label{eq:shared_anchor_model_gap}
r_{0,N}:=\sup_y\left|M(y\mid\vz_N(\vi_0))-p(y\mid\vz_N(\vi_0))\right|
\end{equation}
and the shared-template transport error
\begin{equation}
\label{eq:shared_model_transport_error}
\varepsilon_{M,N}(\widetilde\vi)
:=r_{0,N}
+(L_{M,N}+L_{\mathrm{ref},N})
\|E(\widetilde\vi)-E(\vi_0)\|_F.
\end{equation}
The repeated copies therefore do not introduce an artificial
$\sqrt{N+1}$ factor into the path length; any dependence on their number is
recorded in the pullback Lipschitz constants.

\begin{corollary}[Shared-template transport at fixed demonstration count]
\label{cor:shared_model_transport}
For every fixed $N$,
\begin{equation}
\label{eq:shared_model_reference_transport}
\sup_y\left|M(y\mid\vz_N(\widetilde\vi))
-\bar p(y\mid\vz_N(\widetilde\vi))\right|
\le\varepsilon_{M,N}(\widetilde\vi).
\end{equation}
\end{corollary}

\begin{proof}
Apply the triangle-inequality argument of
Lemma~\ref{lmm:model_reference_transport} to $\gamma_N$, using
$p_{\mathrm{ref}}=p$ at the supported anchor and the tied instruction
coordinate to measure the path length.
\end{proof}

For the mixed setting, $\varepsilon_M(\vz_N)$ is the quantity in
\eqref{eq:app_model_transport_error} on a path
that changes the unsupported instruction slots successively. It is finite for
each fixed prompt but need not be uniform in $N$ or converge to zero. Combining
Lemma~\ref{lmm:model_reference_transport} with the continuation bounds proved
below gives the four model-level bounds stated in the main text.

\input{NeurIPS2026/appendix/appendix_uniform_transport}

\subsection{Complete prompt formats}
\label{app:prompt_formats}

We expand all seven rows of Table~\ref{tab:instruction_formats} below. The
first axis is \emph{task-conditional support}. An instruction may be familiar
globally yet unsupported for the target task: support requires
$p(\instruction\mid t^*)>0$, meaning that pretraining paired that rendering
with $t^*$ with positive probability. Unsupported instructions satisfy
$p(\tilde{\instruction}\mid t^*)=0$. This is not a judgment of semantic
correctness. A dialectal expression or private schema label can express the
user's intended task perfectly while remaining unsupported by the pretraining
distribution.

The second axis is \emph{placement}. A single instruction appears once above
the prompt; a repeated instruction is one realization copied into every
demonstration and the query; varying instructions are separately generated for
the individual slots. Formats~1, 3, and~5 draw their instructions from
$P_{\instruction\mid t^*}$ and are supported by construction. Formats~2, 4,
and~6 use renderings outside that distribution's support. Format~7 begins with
a repeated supported template and corrupts only a subset of its occurrences.
These cases correspond respectively to ordinary user or system messages,
record-level schemas, retrieval or multi-annotator pipelines, and partially
corrupted data.

\begin{enumerate}[leftmargin=*, topsep=0pt, itemsep=3pt]
    \item \textbf{Single supported instruction.} A single supported instruction
        $\instruction$ appears once at the start of the prompt and applies to
        all demonstrations and the query.
        \[
            \vz = (\instruction, (x_{1},y_{1}), \ldots, (x_{N},y_{N}), \xquery).
        \]
        \emph{Concrete example.} A customer-support system places
        \emph{Classify each ticket as billing, technical, or account} once
        above several labeled tickets. This conventional instruction has been
        paired with ticket classification during pretraining and is modeled as
        a draw from $P_{\instruction\mid t^*}$.

    \item \textbf{Single unsupported instruction.} A single instruction
        $\tilde{\instruction}$ outside the task-conditional support appears
        once at the start of the prompt. The absent case is recovered by
        setting $\tilde{\instruction}=\langle\text{pad}\rangle$.
        \[
            \vz = (\tilde{\instruction}, (x_{1},y_{1}), \ldots, (x_{N},y_{N}), \xquery
            ),
        \]
        where $\tilde{\instruction}$ is an unsupported instruction or a virtual token
        $\langle\text{pad}\rangle$ .

        \emph{Concrete example.} An Indian-English user writes
        \emph{Please revert with a reply}, intending \emph{revert} to mean
        ``reply,'' and provides demonstrations mapping incoming emails to reply
        drafts. The wording is correct for the user, but if pretraining paired
        \emph{revert} with operations such as ``undo'' and never with reply
        drafting, then $p(\tilde{\instruction}\mid t^*)=0$ in our idealized
        model.

    \item \textbf{Repeated supported instruction.} One supported instruction
        $\instruction$ is copied into every demonstration and the query, so the
        model encounters the same realization $N+1$ times.
        \[
            \vz = ((\instruction,x_{1},y_{1}), \ldots, (\instruction,x_{N},y_{N})
            , (\instruction,\xquery)).
        \]
        \emph{Concrete example.} A JSON dataset includes
        \texttt{"task": "sentiment"} in every demonstration record and in
        the query record. The standard schema key is supported for sentiment
        classification, and all displayed occurrences are copies of one shared
        template rather than independent instruction samples.

    \item \textbf{Repeated unsupported instruction.} One unsupported (or pad)
        instruction $\tilde{\instruction}$ is copied into every demonstration
        and the query.
        \[
            \vz = ((\tilde{\instruction},x_{1},y_{1}), \ldots, (\tilde{\instruction},x_{N},y_{N})
            , (\tilde{\instruction},\xquery)).
        \]
        \emph{Concrete example.} A company repeatedly inserts
        \texttt{"operation": "Atlas"} into records for its internal
        ticket-routing service. The name was borrowed from an unrelated
        open-source project and was never paired with the company's current
        routing behavior during pretraining. Although the label is meaningful
        inside the company, it is unsupported for $t^*$ and incurs only one
        shared instruction factor because the same label is copied throughout.

    \item \textbf{Varying supported instructions.} Each demonstration is paired
        with an independently sampled supported rendering $\instruction_n$, and
        the query carries another independently sampled rendering $\iquery$.
        \[
            \vz = ((\instruction_1,x_{1},y_{1}), \ldots, (\instruction_{N},x_{N}
            ,y_{N}), (\iquery,\xquery)).
        \]
        \emph{Concrete example.} A retrieval system pools sentiment examples
        written by different annotators: \emph{classify the sentiment},
        \emph{positive or negative?}, and \emph{identify the review's tone}.
        These are distinct but familiar renderings sampled from
        $P_{\instruction\mid t^*}$, so every slot is supported even though the
        surface wording varies.

    \item \textbf{Varying unsupported instructions.} Each demonstration is
        paired with a distinct unsupported rendering
        $\tilde{\instruction}_n$, and the query carries an analogous
        $\tilde{\iquery}$. Each slot can therefore introduce a separate
        off-support cost.
        \[
            \vz = ((\tilde{\instruction_1},x_{1},y_{1}), \ldots, (\tilde{\instruction_{N}},x_{N}
            ,y_{N}), (\tilde{\iquery},\xquery)).
        \]
        \emph{Concrete example.} A multi-tenant retrieval pipeline combines
        ticket-routing demonstrations from companies that use private labels
        such as \texttt{triage\_acme}, \texttt{route\_7G}, and
        \texttt{dispatch\_private}. Each tenant intends its label to denote the
        same routing task, but none of the private labels was paired with that
        task during pretraining. Unlike format~4, the prompt introduces a new
        unsupported label in each slot.

    \item \textbf{Mixed supported and unsupported instructions.} Most slots
        retain one shared supported instruction $\instruction$, while a subset
        contains distinct unsupported corruptions $\tilde{\instruction}_n$:
        \[
            \vz=\bigl((\instruction_n,x_n,y_n)_{n=1}^{N},
            (\iquery,\xquery)\bigr),
            \qquad
            \instruction_n\in\{\instruction,\tilde{\instruction}_n\}.
        \]
        \emph{Concrete example.} A sentiment dataset normally repeats the
        standard field \texttt{sentiment}, but manual entry or OCR produces
        distinct typos such as \texttt{sentimnet} and \texttt{sentimemt} in a
        few records. The unchanged records retain a supported anchor, whereas
        each typo is treated as an unsupported corruption. This differs from
        format~6, in which every slot carries a distinct unsupported rendering.
\end{enumerate}

\subsection{Preliminaries}

We begin by expanding the posterior distribution over tasks using Bayes' rule. For
any task $t\in\gT$, conditioning on the full prompt
$\vz=(\vi,\{(x_{n},y_{n}) \}_{n=1}^{N},\xquery)$ yields
\begin{equation}
    \label{eq:posterior_bayes_full}p(t\mid \vz) = \frac{\pi(t)\,p(\vz\mid t)}{\sum_{s\in\gT}\pi(s)\,p(\vz\mid
    s)}.
\end{equation}
For instructions that lie in the support of the pretraining distribution, we
use different but compatible factorizations for the repeated and varying formats.
In formats~1 and~3, the instruction is sampled once and, in format~3, its displayed
occurrences are deterministic copies of the same shared realization. Hence
\begin{equation}
    \label{eq:likelihood_factorization_shared}
    p(\vz\mid t)=p(\vi\mid t)\,p(\xquery\mid t)
    \prod_{n=1}^{N}p(x_n,y_n\mid t).
\end{equation}
In format~5, the paraphrases $\vi_1,\ldots,\vi_N,\iquery$ are independently
sampled from the task-conditional supported instruction law. Conditional on the
latent task, they are also independent of the demonstration pairs and query
input, so
\begin{equation}
    \label{eq:likelihood_factorization_varying}
    p(\vz\mid t)=p(\iquery\mid t)\,p(\xquery\mid t)
    \prod_{n=1}^{N}p(\vi_n\mid t)p(x_n,y_n\mid t).
\end{equation}
To treat all three formats uniformly, define
\begin{equation}
    \label{eq:def_Wn}
    W_n:=
    \begin{cases}
        (x_n,y_n), & \text{formats~1 and~3},\\
        (\vi_n,x_n,y_n), & \text{format~5},
    \end{cases}
    \qquad
    \vi_\circ:=
    \begin{cases}
        \vi, & \text{formats~1 and~3},\\
        \iquery, & \text{format~5}.
    \end{cases}
\end{equation}
For format~5 we therefore have
$p(W_n\mid t)=p(\vi_n\mid t)p(x_n,y_n\mid t)$. In all three cases,
\begin{equation}
    \label{eq:likelihood_factorization}
    p(\vz\mid t)=p(\vi_\circ\mid t)\,p(\xquery\mid t)
    \prod_{n=1}^{N}p(W_n\mid t).
\end{equation}
For notational simplicity, below we write $\vi$ for the single
non-demonstration instruction variable $\vi_\circ$.

Substituting \eqref{eq:likelihood_factorization} into
\eqref{eq:posterior_bayes_full}, we obtain
\begin{equation}
    \label{eq:posterior_expanded}
    p(t\mid \vz)=
    \frac{\pi(t)p(\vi\mid t)p(\xquery\mid t)
    \prod_{n=1}^{N}p(W_n\mid t)}
    {\sum_{s\in\gT}\pi(s)p(\vi\mid s)p(\xquery\mid s)
    \prod_{n=1}^{N}p(W_n\mid s)}.
\end{equation}

To make the dependence on the number of demonstrations explicit, we compare each
task $t$ against the true task $t^{*}$. For any $t\neq t^{*}$,
\begin{equation}
    \label{eq:posterior_odds_raw}
    \frac{p(t\mid \vz)}{p(t^{*}\mid \vz)}
    =\frac{\pi(t)}{\pi(t^{*})}
    \frac{p(\vi\mid t)}{p(\vi\mid t^{*})}
    \frac{p(\xquery\mid t)}{p(\xquery\mid t^{*})}
    \frac{\prod_{n=1}^{N}p(W_n\mid t)}
         {\prod_{n=1}^{N}p(W_n\mid t^{*})}.
\end{equation}

We define the task-evidence likelihood ratio
\begin{equation}
    \label{eq:def_Bt}
    B_t:=\frac{\prod_{n=1}^{N}p(W_n\mid t)}
               {\prod_{n=1}^{N}p(W_n\mid t^{*})},
\end{equation}
and the non-demonstration ratio
\begin{equation}
    \label{eq:def_At}
    A_t(\vi):=\frac{\pi(t)}{\pi(t^{*})}
    \frac{p(\vi\mid t)}{p(\vi\mid t^{*})}
    \frac{p(\xquery\mid t)}{p(\xquery\mid t^{*})}.
\end{equation}
Thus
\begin{equation}
    \label{eq:posterior_odds_representation}
    \frac{p(t\mid \vz)}{p(t^{*}\mid \vz)}=A_t(\vi)B_t,
    \qquad t\neq t^{*},
\end{equation}
and
\begin{equation}
    \label{eq:posterior_representation}
    p(t^{*}\mid \vz)=\frac{1}{1+\sum_{s\neq t^*}A_s(\vi)B_s},
    \qquad
    p(t\mid \vz)=\frac{A_t(\vi)B_t}
    {1+\sum_{s\neq t^*}A_s(\vi)B_s},\quad t\neq t^{*}.
\end{equation}
\noindent
Then we present two lemmas.
\begin{lemma}[Exponential Decay of Task-Evidence Likelihood Ratio]
    \label{lmm:Bt_exp_decay} Defining
    \[
        \alpha := \tfrac12\min_{t\neq t^*}\kappa_{t}>0,
    \]
    there exists $N_{0}<\infty$ such that almost surely, for all $N\ge N_{0}$
    and all $t\neq t^{*}$,
    \[
        B_{t}\le e^{-\alpha N}.
    \]
\end{lemma}
\begin{proof}
    Fix any $t\neq t^{*}$. Assume the corresponding log-likelihood ratio is integrable. By the strong law of large numbers and Assumption~\ref{assp:identifiability},
    \begin{equation}
        \label{eq:SLLN_KL}\frac{1}{N}\sum_{n=1}^{N}\log \frac{p(W_n\mid t)}{p(W_n\mid t^{*})}
        ~\xrightarrow{\mathrm{a.s.}}~
        -\kl\!\Big(P_{W\mid t^*}\,\Vert\,P_{W\mid t}\Big)=:-\kappa_t,
    \end{equation}
    where $\kappa_t>0$. For formats~1 and~3 this is the demonstration KL divergence; for format~5 it is the joint KL divergence of the independently sampled instruction--demonstration triples $W_n=(\vi_n,x_n,y_n)$. Hence, almost surely, there exists $N_{t}<\infty$ such
    that for all $N\ge N_{t}$,
    \begin{equation}
        \label{eq:Bt_exp_decay_single}B_{t}\le \exp\!\big(-( \kappa_{t}/2)\, N\big
        ).
    \end{equation}
    Since $\gT$ is finite, letting
    \begin{equation}
        \label{eq:alpha_def}\alpha := \frac{1}{2}\min_{t\neq t^*}\kappa_{t}\;>\;0
        ,
    \end{equation}
    we have (after enlarging $N_{0}:=\max_{t\neq t^*}N_{t}$) that almost surely,
    for all $N\ge N_{0}$ and all $t\neq t^{*}$,
    \begin{equation}
        \label{eq:Bt_exp_decay_uniform}B_{t}\le e^{-\alpha N}.
    \end{equation}
\end{proof}

\begin{lemma}[Uniform Boundedness of Non-Demonstration Ratio]
    \label{lmm:At_uniform} There exists a finite constant $A_{\max}$ such that
    \[
        \sup_{\vi\in\dicl}\max_{t\neq t^*}A_{t}(\vi) \le A_{\max}< \infty.
    \]
\end{lemma}

\begin{proof}
    Recall that
    \[
        A_t(\vi)=\frac{\pi(t)}{\pi(t^{*})}
        \frac{p(\vi\mid t)}{p(\vi\mid t^{*})}
        \frac{p(\xquery\mid t)}{p(\xquery\mid t^{*})}.
    \]
    The vocabulary is finite and instruction length is bounded by $L_i$, so
    the supported instruction set $\dicl\subseteq\gV^{\le L_i}$ is finite. Moreover,
    $p(\vi\mid t^*)>0$ for every $\vi\in\dicl$ by the definition of support. Therefore each
    ratio $p(\vi\mid t)/p(\vi\mid t^*)$ is finite, and its maximum over
    $\dicl$ is finite. Since $\gT$ is finite and $\xquery$ is fixed,
    \[
        A_{\max}:=
        \max_{t\neq t^{*}}\frac{\pi(t)}{\pi(t^{*})}
        \frac{p(\xquery\mid t)}{p(\xquery\mid t^{*})}
        \max_{\vi\in\dicl}\frac{p(\vi\mid t)}{p(\vi\mid t^{*})}
        <\infty.
    \]
    This proves the claim without a differentiability or gradient assumption.
\end{proof}

\subsection{Proof of Theorem~\ref{thm:decay_1_5}}
\label{app:supported_instruction_proof}
We first recall the statement of the theorem for convenience.
\begin{theorem}[Formal model bound with supported instructions]
    \label{thm:decay_1_5_app} Consider prompt formats whose instructions lie in the support of the pretraining distribution, including formats~1, 3, and~5. Under the generation mechanism in Section~\ref{subsec:instruction_setup} and
    Assumption~\ref{assp:identifiability}, there exist constants $\alpha
    ,\beta>0$ such that, almost surely, for all sufficiently large numbers
    of demonstrations $N$,
    \begin{align}
        \bigl| p(\yquery \mid \vz) - p(\yquery \mid \xquery, t^{*}) \bigr| \le \beta \exp(-\alpha N).
    \end{align}
    Moreover, with $r_0$ defined in
    \eqref{eq:app_supported_model_gap} and $\varepsilon_M:=r_0$,
    \begin{equation}
        \label{eq:supported_model_bound_app}
        \bigl|M(\yquery\mid\vz)-p(\yquery\mid\xquery,t^*)\bigr|
        \le \beta e^{-\alpha N}+\varepsilon_M.
    \end{equation}
    Consequently, for any two supported instruction realizations $\boldsymbol\vi,\boldsymbol\vi'$ to which the same format-specific assumptions apply (a single shared realization in formats~1 and~3, or two independently sampled paraphrase sequences in format~5), we have
    \begin{align}
        \bigl| p(\yquery \mid \vz(\boldsymbol\vi)) - p(\yquery \mid \vz(\boldsymbol\vi')) \bigr| \le 2\beta \exp(-\alpha N).
    \end{align}
\end{theorem}

We now proceed with the proof.
\begin{proof}
    By definition of the posterior predictive,
    \begin{equation}
        \label{eq:pp_def}p(\yquery\mid \vz) ~=~ \sum_{t\in\gT}p(\yquery\mid \xquery
        ,t)\, p(t\mid \vz).
    \end{equation}
    Subtracting $p(\yquery\mid \xquery,t^{*})$ and using
    $\sum_{t}p(t\mid \vz)=1$ yields
    \begin{align}
        \label{eq:pp_diff_expand}p(\yquery\mid \vz) - p(\yquery\mid \xquery,t^{*}) & = \sum_{t\neq t^*}\big(p(\yquery\mid \xquery,t)-p(\yquery\mid \xquery,t^{*})\big)\, p(t\mid \vz).
    \end{align}
    Since $0\le p(\yquery\mid \xquery,t)\le 1$, we have $|p(\yquery\mid \xquery,t
    )-p(\yquery\mid \xquery,t^{*})|\le 1$ for all $t$, hence
    \begin{equation}
        \label{eq:pp_diff_bound_by_posterior_mass}\big|p(\yquery\mid \vz) - p(\yquery
        \mid \xquery,t^{*})\big| ~\le~ \sum_{t\neq t^*}p(t\mid \vz).
    \end{equation}
    It remains to upper bound the posterior mass on wrong tasks. From Equation~\ref{eq:posterior_representation},
    for any $t\neq t^{*}$,
    \begin{equation}
        p(t\mid \vz) ~=~ \frac{A_{t}(\vi)B_{t}}{1+\sum_{s\neq t^*}A_{s}(\vi)B_{s}}
        ~\le~ A_{t}(\vi)B_{t}.
    \end{equation}
    Summing over $t\neq t^{*}$ and applying Lemma~\ref{lmm:Bt_exp_decay} and ~\ref{lmm:At_uniform},
    we obtain that almost surely for all $N\ge N_{0}$,
    \begin{equation}
        \label{eq:wrong_mass_exp_decay}\sum_{t\neq t^*}p(t\mid \vz) ~\le~ \sum_{t\neq
        t^*}A_{t}(\vi)B_{t}~\le~ (|\gT|-1)\,A_{\max}\, e^{-\alpha N}.
    \end{equation}
    Combining Equation~\ref{eq:pp_diff_bound_by_posterior_mass},~\ref{eq:wrong_mass_exp_decay}
    proves the first statement with $\beta := (|\gT|-1)A_{\max}$:
    \[
        \big|p(\yquery\mid \vz)-p(\yquery\mid \xquery,t^{*})\big| \le \beta e^{-\alpha
        N}.
    \]

    For the equivalent statement comparing two supported instruction realizations $\boldsymbol\vi,\boldsymbol\vi'$ covered by the statement, we use the triangle inequality:
    \begin{align}
         & \big|p(\yquery\mid \vz(\boldsymbol\vi)) - p(\yquery\mid \vz(\boldsymbol\vi'))\big|                                                                             \\
         & \le \big|p(\yquery\mid \vz(\boldsymbol\vi)) - p(\yquery\mid \xquery,t^{*})\big| + \big|p(\yquery\mid \vz(\boldsymbol\vi')) - p(\yquery\mid \xquery,t^{*})\big| \\
         & \le 2\beta e^{-\alpha N},
    \end{align}
    Finally, a supported prompt is the constant-path case of
    Lemma~\ref{lmm:model_reference_transport}, so
    $|M(\yquery\mid\vz)-p(\yquery\mid\vz)|\le r_0$. Combining this inequality
    with the first bound by the triangle inequality proves
    \eqref{eq:supported_model_bound_app} and completes the proof.
\end{proof}

\subsection{Extension to an Unsupported Shared Instruction}
\label{app:repeated_unsupported_chain}

We consider prompt format~4, in which the same instruction is prepended to all
demonstrations and to the query. Let
\begin{equation}
    \label{eq:chain_repeated_prompt}
    \vz_N(\vi):=\bigl((\vi,x_1,y_1),\ldots,(\vi,x_N,y_N),(\vi,\xquery)\bigr).
\end{equation}

\subsubsection{The shared-instruction likelihood}

Although the instruction appears $N+1$ times in the surface form of
\eqref{eq:chain_repeated_prompt}, these occurrences are copies of a single
template chosen once for the whole prompt, and not $N+1$ independent draws. We
therefore model the instruction through a latent template variable.

\begin{assumption}[Shared-instruction factorization]
\label{ass:chain_factorization}
There is a latent template $\tau$, drawn once per prompt from a task-conditional
law $P_{\tau\mid t}$. Conditional on $\tau=\vi$, every displayed instruction
slot is a deterministic copy of the same realization $\vi$, while the
demonstrations and query are conditionally independent of $\tau$ given the task.
Thus, for every $t\in\gT$
\begin{equation}
    \label{eq:chain_repeated_factorization}
    p(\vz_N(\vi)\mid t)
    =
    p(\vi\mid t)\,\prod_{n=1}^{N}p(x_n,y_n\mid t)\;p(\xquery\mid t).
\end{equation}
\end{assumption}

\noindent
The instruction likelihood enters \eqref{eq:chain_repeated_factorization}
\emph{once}. This is the substantive modelling choice of this subsection: seeing
one template repeated $N+1$ times is one observation about the task, not $N+1$
of them. Consequently, for every $t\neq t^{*}$,
\begin{equation}
    \label{eq:chain_repeated_odds}
    \frac{p(t\mid\vz_N(\vi))}{p(t^{*}\mid\vz_N(\vi))}
    =B_t\exp\bigl(r_t(\vi)+\lambda_t\bigr),
\end{equation}
where
\begin{equation}
    \label{eq:chain_g_lambda}
    r_t(\vi):=\log\frac{p(\vi\mid t)}{p(\vi\mid t^{*})},
    \qquad
    \lambda_t:=\log\!\left(\frac{\pi(t)}{\pi(t^{*})}\frac{p(\xquery\mid t)}{p(\xquery\mid t^{*})}\right),
\end{equation}
and $B_t$ is the demonstration likelihood ratio of \eqref{eq:def_Bt}.

\subsubsection{Minimal endpoint condition and the role of analyticity}

The exponential mechanism does not itself require analytic continuation.  Let
a chosen off-support reference extension assign each competitor a finite
endpoint log-odds $c_t:=\bar r_t(\widetilde\vi)\in\mathbb R$, independent of
$N$, and define its endpoint odds by
\begin{equation}
    \label{eq:fixed_endpoint_odds}
    \bar R_{t,N}:=B_t\exp(c_t+\lambda_t).
\end{equation}
Then Lemma~\ref{lmm:Bt_exp_decay} immediately gives, almost surely for all
sufficiently large $N$,
\begin{equation}
    \label{eq:fixed_endpoint_robustness}
    \sum_{t\ne t^*}\bar R_{t,N}
    \le (|\gT|-1)\max_{t\ne t^*}e^{c_t+\lambda_t}e^{-\alpha N}.
\end{equation}
Since the predictive deviation is at most the total competing posterior odds,
the reference predictor converges at the same exponential rate.  Thus the
minimal substantive statement for a shared instruction is robustness to an
arbitrary fixed finite off-support odds perturbation: demonstration log-evidence
grows linearly in $N$, whereas the instruction contributes only an
$N$-independent constant.

The fixed-endpoint argument above establishes exponential convergence
conditional on arbitrary finite constants $c_t$.  The analytic-continuation
theorem below does not improve the exponent $\alpha$, which is determined by
the demonstration likelihood ratios.  Its additional content is to replace
the unstructured endpoint constants by a Transformer-induced continuation from
a supported anchor and to derive the explicit path-dependent bound
$c_t=\phi_t(1)\le\gamma_0+\Omega$.  Consequently, the unknown factor
$\max_t e^{c_t}$ in \eqref{eq:fixed_endpoint_robustness} is controlled by
$e^{\gamma_0+\Omega}$.  The Cauchy--Taylor chain further yields local
propagation bounds, a comparison bound between two unsupported endpoints, and
additive continuation costs used in the mixed and varying-instruction
analyses.  Thus, the proof burden of the analytic theorem lies in constructing
and quantitatively controlling the off-support perturbation, rather than in
establishing a different asymptotic domination mechanism.  Analyticity does
not make the endpoint semantics statistically identified by zero-probability
pretraining data.

Let $\vi_0$ be a supported anchor and $\widetilde\vi$ an unsupported instruction.
Write
\begin{equation}
    \label{eq:chain_instruction_path}
    E_s:=(1-s)E(\vi_0)+sE(\widetilde\vi),\quad s\in[0,1],
    \qquad
    \Delta:=\|E(\widetilde\vi)-E(\vi_0)\|_F ,
\end{equation}
in the fixed-length instruction-embedding space of
Section~\ref{subsec:instruction_setup}, so $s=0$ is supported and $s=1$ is the
unsupported endpoint. Since the Bayes
conditional is not determined at a zero-probability prompt, the object studied
below is the posterior predictive induced by a fixed chain of local Taylor
expansions of the supported log-odds along \eqref{eq:chain_instruction_path}.

\begin{assumption}[Supported anchor]
\label{ass:chain_anchor}
The anchor is a supported rendering of the target task with $p(\vi_0\mid t)>0$ for every $t\in\gT$. Among admissible supported renderings one may choose the most typical anchor, for example $\vi_0\in\arg\max_{\vi\in\dicl}p(\vi\mid t^{*})$ (such as ``$+$'' among supported addition renderings). The proof below uses only support and finiteness. We write
\begin{equation}
    \label{eq:chain_gamma0}
    \gamma_0:=\max_{t\neq t^{*}}r_t(\vi_0)<\infty ,
\end{equation}
which satisfies $\gamma_0\le0$ whenever the selected anchor is task-favoring, i.e.\
$p(\vi_0\mid t)\le p(\vi_0\mid t^{*})$ for every $t\neq t^{*}$.
\end{assumption}

\noindent
Under the shared reading $\gamma_0$ enters only additively in a constant, so no
sign condition on it is required.

\subsubsection{The chain hypothesis, and the constants it induces}
 
The hypothesis below specifies a finite chain of overlapping complex disks
along which the log-odds continue holomorphically. The constants used in the
proof are defined from this chain in Lemma~\ref{lmm:chain_cauchy}.

For clarity, the new symbol $\phi_{t,\ell}$ denotes the local analytic
continuation of the supported pathwise log-odds
$s\mapsto r_t(E_s)$ on the $\ell$-th disk of
the chain. The index $t$
specifies the competing task and $\ell$ specifies the local neighborhood.
Agreement on overlapping disks makes these local functions pieces of one
continued function along the whole path; we denote that resulting function by
$\phi_t$.
 
\begin{assumption}[Finite chain of analytic continuations]
\label{ass:chain_local}
There exist an integer $L<\infty$ and centers
\begin{equation}
    \label{eq:chain_centers}
    0=s_0<s_1<\cdots<s_L=1
\end{equation}
with associated open complex disks $D_\ell:=\{\zeta\in\mathbb{C}:|\zeta-s_\ell|<r_\ell\}$,
$\ell=0,\ldots,L-1$, all independent of $N$, such that
\begin{equation}
    \label{eq:chain_step_inside}
    d_\ell:=s_{\ell+1}-s_\ell<r_\ell,
    \qquad \ell=0,\ldots,L-1,
\end{equation}
and such that for every $t\neq t^{*}$ there are holomorphic functions
$\phi_{t,\ell}:D_\ell\to\mathbb{C}$ with
\begin{enumerate}
    \item[(i)] \textup{(compatibility)} $\phi_{t,\ell}=\phi_{t,m}$ on
    $D_\ell\cap D_m$ whenever this intersection is nonempty;
    \item[(ii)] \textup{(anchoring)}
    $\phi_{t,0}(s)=r_t(E_s)$ on a nontrivial
    instruction-embedding interval adjacent to $s=0$ on which the fixed extension is
    defined.
\end{enumerate}
We write $\phi_t$ for the function these germs define on $[0,1]$, and
$\phi_t(1)$ for the value at $s_L=1$, which lies in $D_{L-1}$ by
\eqref{eq:chain_step_inside}.
\end{assumption}
 
\noindent
Only $D_0$ concerns supported instructions; the remaining $L-1$ disks are the
part of the hypothesis that describes continuation outside the support.
 
\begin{lemma}[Constants induced by the chain]
\label{lmm:chain_cauchy}
Fix working radii $\rho_\ell\in(d_\ell,r_\ell)$ and set
\begin{equation}
    \label{eq:chain_M_def}
    M_\ell
    :=
    \max_{t\neq t^*}\;\max_{|\zeta-s_\ell|=\rho_\ell}
    \bigl|\phi_t(\zeta)-\phi_t(s_\ell)\bigr|,
    \qquad
    q_\ell:=\frac{d_\ell}{\rho_\ell}\in(0,1).
\end{equation}
Then, under Assumption~\ref{ass:chain_local}:
\begin{enumerate}
    \item[(a)] every $M_\ell$ is finite, and $M_\ell,\rho_\ell,q_\ell$ are
    independent of $N$;
    \item[(b)] for every $t\neq t^{*}$, every $\ell$ and every $k\ge1$,
    \begin{equation}
        \label{eq:chain_germ_bound}
        \bigl|\phi_t^{(k)}(s_\ell)\bigr|
        \;\le\;
        \frac{M_\ell\,k!}{\rho_\ell^{\,k}} ;
    \end{equation}
    \item[(c)] every $\phi_t$ has real Taylor coefficients at every $s_\ell$, and
    $\phi_t(s)\in\mathbb{R}$ for all $s\in[0,1]$.
\end{enumerate}
\end{lemma}
 
\begin{proof}
(a) Since $\rho_\ell<r_\ell$, the circle $|\zeta-s_\ell|=\rho_\ell$ is a compact
subset of $D_\ell$, on which the holomorphic---hence continuous---function
$\phi_t$ is bounded. The maximum over the finite set $\gT$ of finitely many
finite numbers is finite. None of $D_\ell$, $\rho_\ell$, $\phi_t$ depends on $N$,
by Assumption~\ref{ass:chain_local} and because $r_t$ in
\eqref{eq:chain_g_lambda} does not.
 
(b) Apply the Cauchy estimates on the circle of radius $\rho_\ell$ to the
holomorphic function $\psi_t:=\phi_t-\phi_t(s_\ell)$, which satisfies
$\max_{|\zeta-s_\ell|=\rho_\ell}|\psi_t|=M_\ell$ by \eqref{eq:chain_M_def} and
$\psi_t^{(k)}=\phi_t^{(k)}$ for $k\ge1$.
 
(c) By (ii) the germ $\phi_{t,0}$ agrees with the real-valued function
$s\mapsto r_t(E_s)$ on a nontrivial real
interval, so all its derivatives at
$s_0$ are real. If $\phi_t$ has real coefficients at $s_\ell$, then by (i) its
coefficients at the real point $s_{\ell+1}$ are obtained by term-by-term
differentiation of a real power series, hence are real. Induction on $\ell$
gives the claim, and a power series with real coefficients is real on the real
axis inside its disk of convergence.
\end{proof}
 
\noindent
Inequality \eqref{eq:chain_germ_bound} is the quantitative form of analyticity at
$s_\ell$: by Cauchy--Hadamard it says the Taylor series there has radius of
convergence at least $\rho_\ell$ with coefficients bounded by
$M_\ell\rho_\ell^{-k}$. It is a consequence of Assumption~\ref{ass:chain_local},
not an additional hypothesis.
\subsubsection{One Taylor step, and the cancellation of $k!$}

The next lemma is the sole mechanism of the proof. It converts the local
coefficient bound \eqref{eq:chain_germ_bound} at $s_\ell$ into a bound on the
\emph{increment} of $\phi_t$ across the link. The factorial $k!$ in
\eqref{eq:chain_germ_bound} is exactly cancelled by the factorial in the Taylor
coefficients, leaving a geometric series of ratio $q_\ell<1$.

\begin{lemma}[Taylor transfer across one link]
\label{lmm:chain_one_step}
Under Assumption~\ref{ass:chain_local}, for every $t\neq t^{*}$ and every
$\ell=0,\ldots,L-1$:
\begin{enumerate}
    \item[(a)] the increment across the link satisfies
    \begin{equation}
        \label{eq:chain_eta}
        \eta_\ell:=\max_{t\neq t^*}\bigl|\phi_t(s_{\ell+1})-\phi_t(s_\ell)\bigr|
        \;\le\;
        \frac{M_\ell\,q_\ell}{1-q_\ell};
    \end{equation}
    \item[(b)] more generally, for $0\le u<1$ and $|s-s_\ell|\le u\rho_\ell$,
    \begin{equation}
        \label{eq:chain_local_osc}
        \bigl|\phi_t(s)-\phi_t(s_\ell)\bigr|\le\frac{M_\ell\,u}{1-u};
    \end{equation}
    \item[(c)] the germ at $s_{\ell+1}$ inherited from $s_\ell$ obeys
    \begin{equation}
        \label{eq:chain_germ_propagation}
        \bigl|\phi_t^{(k)}(s_{\ell+1})\bigr|
        \;\le\;
        \frac{M_\ell}{1-q_\ell}\cdot\frac{k!}{\bigl(\rho_\ell-d_\ell\bigr)^{k}},
        \qquad k\ge1 .
    \end{equation}
\end{enumerate}
\end{lemma}

\begin{proof}
(b) Since the Taylor series of $\phi_t$ at $s_\ell$ converges for
$|s-s_\ell|<\rho_\ell$ by \eqref{eq:chain_germ_bound},
\begin{equation}
    \label{eq:chain_cancellation}
    \bigl|\phi_t(s)-\phi_t(s_\ell)\bigr|
    \le
    \sum_{k\ge1}\frac{\bigl|\phi_t^{(k)}(s_\ell)\bigr|}{k!}\,|s-s_\ell|^{k}
    \le
    \sum_{k\ge1}\frac{1}{k!}\cdot\frac{M_\ell k!}{\rho_\ell^{\,k}}\,(u\rho_\ell)^{k}
    =M_\ell\sum_{k\ge1}u^{k}
    =\frac{M_\ell u}{1-u}.
\end{equation}
The middle step is the cancellation of $k!$ against $1/k!$: without it the series
would be $\sum_k k!\,u^{k}$, which diverges for every $u>0$.

(a) Apply (b) with $u=q_\ell$, legitimate because $d_\ell=q_\ell\rho_\ell$ and
$q_\ell<1$, and take the maximum over the finite set $\gT$.

(c) Differentiating the Taylor series of $\phi_t$ at $s_\ell$ term by term $k$
times and evaluating at $s_{\ell+1}$,
\[
    \phi_t^{(k)}(s_{\ell+1})
    =\sum_{j\ge k}\frac{\phi_t^{(j)}(s_\ell)}{(j-k)!}\,d_\ell^{\,j-k},
\]
so by \eqref{eq:chain_germ_bound} and $\tfrac{(m+k)!}{m!}=k!\binom{m+k}{k}$,
\[
    \bigl|\phi_t^{(k)}(s_{\ell+1})\bigr|
    \le\frac{M_\ell}{\rho_\ell^{\,k}}\sum_{m\ge0}\frac{(m+k)!}{m!}q_\ell^{\,m}
    =\frac{M_\ell\,k!}{\rho_\ell^{\,k}}\,(1-q_\ell)^{-(k+1)},
\]
using $\sum_{m\ge0}\binom{m+k}{k}q^{m}=(1-q)^{-(k+1)}$. Rearranging with
$\rho_\ell(1-q_\ell)=\rho_\ell-d_\ell$ gives \eqref{eq:chain_germ_propagation}.
\end{proof}

\begin{remark}[The chain hypothesis is not vacuous, and not redundant]
\label{rem:chain_fresh_information}
Part~(c) shows that a germ at $s_\ell$ automatically induces a germ at
$s_{\ell+1}$ with degraded constants
$\bigl(M_\ell/(1-q_\ell),\,\rho_\ell-d_\ell\bigr)$: the radius shrinks by exactly
the step length, which is the classical behaviour of Weierstrass continuation.
Iterating from $s_0$ alone would therefore reach only $s<\rho_0$, recovering the
single-germ case. Assumption~\ref{ass:chain_local} is binding precisely when
some center supplies constants better than those propagated by (c), i.e.\ when
the log-odds regain analytic room further along the path. This is what allows
$L>1$ to travel a distance exceeding $\rho_0$.
\end{remark}
\begin{lemma}[Additive propagation along the chain]
\label{lmm:chain_telescope}
Define the \emph{cumulative continuation cost}
\begin{equation}
    \label{eq:chain_Omega}
    \Omega:=\sum_{\ell=0}^{L-1}\frac{M_\ell\,q_\ell}{1-q_\ell} .
\end{equation}
Under Assumption~\ref{ass:chain_local}, $\Omega<\infty$ automatically, since the
chain is finite and every $M_\ell$ is finite and every $q_\ell<1$. Moreover,
under Assumption~\ref{ass:chain_anchor}, for every $t\neq t^{*}$ and every
$\ell=0,\ldots,L$,
\begin{equation}
    \label{eq:chain_telescope}
    \phi_t(s_\ell)\;\le\;\gamma_0+\sum_{j=0}^{\ell-1}\eta_j\;\le\;\gamma_0+\Omega .
    \qquad\text{In particular}\qquad
    \phi_t(1)\le\gamma_0+\Omega=:\Omega_0 .
\end{equation}
\end{lemma}

\begin{proof}
By Lemma~\ref{lmm:chain_cauchy}(c) all quantities are real. Telescoping,
\[
    \phi_t(s_\ell)
    =\phi_t(s_0)+\sum_{j=0}^{\ell-1}\bigl(\phi_t(s_{j+1})-\phi_t(s_j)\bigr)
    \le r_t(\vi_0)+\sum_{j=0}^{\ell-1}\eta_j
    \le\gamma_0+\sum_{j=0}^{\ell-1}\eta_j,
\]
and $\eta_j\le M_jq_j/(1-q_j)$ by Lemma~\ref{lmm:chain_one_step}(a). Since every
term is nonnegative, the partial sums are bounded by $\Omega$.
\end{proof}

\noindent
The propagation is \emph{additive} because the chain acts on the log-odds. Had
the same chain been run on the odds themselves, or on the posterior predictive,
each link would have contributed a multiplicative factor and $L$ links would
produce an $L$-th power; this is the structural reason the present argument
succeeds where amplification-based transfers do not.

\subsubsection{Main result}

Define the continued odds, odds mass and predictive deviation by
\begin{equation}
    \label{eq:chain_R_sigma_h}
    R_{t,N}(s):=B_t\exp\bigl(\phi_t(s)+\lambda_t\bigr),
    \qquad
    \sigma_N(s):=\sum_{t\neq t^*}R_{t,N}(s),
\end{equation}
\begin{equation}
    \label{eq:chain_h}
    h_N(s):=\sum_{t\neq t^*}b_t\,\frac{R_{t,N}(s)}{1+\sigma_N(s)},
    \qquad
    b_t:=p(\yquery\mid\xquery,t)-p(\yquery\mid\xquery,t^{*}),\quad|b_t|\le1 .
\end{equation}
At every $s$ in the anchored initial instruction-embedding interval,
$h_N(s)=p(\yquery\mid\vz_N(E_s))-p(\yquery\mid\xquery,t^{*})$
by \eqref{eq:chain_repeated_odds}; at $s=1$ it is the value induced by the chain.

\begin{theorem}[Formal model bound for an unsupported shared instruction]
\label{thm:chain_unsupported_convergence}
Let Assumptions~\ref{assp:identifiability}, \ref{ass:chain_factorization},
\ref{ass:chain_anchor} and~\ref{ass:chain_local} hold. Then almost surely, for
all sufficiently large $N$,
\begin{equation}
    \label{eq:chain_main_bound}
    \bigl|h_N(1)\bigr|
    \;\le\;
    \widetilde\beta\,e^{-\alpha N},
    \qquad
    \widetilde\beta:=\beta_\circ e^{\Omega_0},
    \quad
    \beta_\circ:=(|\gT|-1)\max_{t\neq t^*}e^{\lambda_t},
\end{equation}
with $\alpha=\tfrac12\min_{t\neq t^{*}}\kappa_t$ as in
Lemma~\ref{lmm:Bt_exp_decay}. Equivalently, the posterior predictive induced by
the chain satisfies
\begin{equation}
    \label{eq:chain_predictive_convergence}
    \bigl|\bar p(\yquery\mid\vz_N(\widetilde\vi))-p(\yquery\mid\xquery,t^{*})\bigr|
    \le\widetilde\beta e^{-\alpha N},
\end{equation}
and, for the actual pretrained model,
\begin{equation}
    \label{eq:chain_model_convergence}
    \bigl|M(\yquery\mid\vz_N(\widetilde\vi))
    -p(\yquery\mid\xquery,t^{*})\bigr|
    \le\widetilde\beta e^{-\alpha N}
    +\varepsilon_{M,N}(\widetilde\vi),
\end{equation}
where $\varepsilon_{M,N}(\widetilde\vi)$ is defined by
\eqref{eq:shared_model_transport_error}, and for
any two unsupported instructions admitting chains with costs
$\Omega_0,\Omega_0'$,
\begin{equation}
    \label{eq:chain_sensitivity}
    \bigl|\bar p(\yquery\mid\vz_N(\widetilde\vi))-\bar p(\yquery\mid\vz_N(\widetilde\vi\,'))\bigr|
    \le 2\beta_\circ e^{\max\{\Omega_0,\Omega_0'\}}e^{-\alpha N}.
\end{equation}
\end{theorem}

\begin{proof}
By Lemma~\ref{lmm:Bt_exp_decay}, almost surely for all sufficiently large $N$ we
have $B_t\le e^{-\alpha N}$ for every $t\neq t^{*}$. By
Lemma~\ref{lmm:chain_telescope}, $\phi_t(1)\le\Omega_0$, and by
Lemma~\ref{lmm:chain_cauchy}(c) this is a real inequality. Hence
\begin{equation}
    \label{eq:chain_final_odds}
    R_{t,N}(1)
    =B_t\exp\bigl(\phi_t(1)+\lambda_t\bigr)
    \le e^{\lambda_t}e^{\Omega_0}e^{-\alpha N}.
\end{equation}
Summing over the $|\gT|-1$ competitors,
$\sigma_N(1)\le\beta_\circ e^{\Omega_0}e^{-\alpha N}$. Since every $R_{t,N}(1)$
is real and nonnegative we have $1+\sigma_N(1)\ge1$, so by \eqref{eq:chain_h}
and $|b_t|\le1$,
\[
    \bigl|h_N(1)\bigr|\le\frac{\sigma_N(1)}{1+\sigma_N(1)}\le\sigma_N(1)
    \le\beta_\circ e^{\Omega_0}e^{-\alpha N},
\]
which is \eqref{eq:chain_main_bound}. Statement \eqref{eq:chain_sensitivity}
follows by the triangle inequality as in Theorem~\ref{thm:decay_1_5_app}.
Combining \eqref{eq:chain_predictive_convergence} with
Corollary~\ref{cor:shared_model_transport} gives
\eqref{eq:chain_model_convergence}.
\end{proof}

\begin{assumption}[Uniform model--reference tracking along a shared-template path]
\label{assp:uniform_shared_tracking}
For the fixed unsupported instruction $\widetilde\vi$, there exist $N_0$ and
a finite constant $\overline\varepsilon_M(\widetilde\vi)\ge0$ such that,
almost surely,
\begin{equation}
\label{eq:uniform_shared_tracking}
\sup_{N\ge N_0}\varepsilon_{M,N}(\widetilde\vi)
\le\overline\varepsilon_M(\widetilde\vi).
\end{equation}
\end{assumption}

\begin{corollary}[Uniform shared-template model tracking]
\label{cor:chain_uniform_model_tracking}
If, in addition, Assumption~\ref{assp:uniform_shared_tracking} holds, then
almost surely for all sufficiently large $N$,
\begin{equation}
\label{eq:chain_uniform_model_bound}
\bigl|M(\yquery\mid\vz_N(\widetilde\vi))
-p(\yquery\mid\xquery,t^{*})\bigr|
\le\widetilde\beta e^{-\alpha N}
+\overline\varepsilon_M(\widetilde\vi),
\end{equation}
and consequently
\[
\limsup_{N\to\infty}
\bigl|M(\yquery\mid\vz_N(\widetilde\vi))
-p(\yquery\mid\xquery,t^{*})\bigr|
\le\overline\varepsilon_M(\widetilde\vi).
\]
If, more strongly, $\varepsilon_{M,N}(\widetilde\vi)\to0$, then
$M(\yquery\mid\vz_N(\widetilde\vi))
\to p(\yquery\mid\xquery,t^{*})$ almost surely.
\end{corollary}

\begin{proof}
Substitute the uniform bound from
Assumption~\ref{assp:uniform_shared_tracking} into
\eqref{eq:chain_model_convergence} and take the limit superior.  The final
claim follows from the same bound when both terms on its right-hand side
vanish.
\end{proof}

\noindent
Proposition~\ref{prop:uniform_local_transport} supplies sufficient conditions
for Assumption~\ref{assp:uniform_shared_tracking}, with the explicit choice
$\overline\varepsilon_M=\bar r+(C_M+2G)d$.
Proposition~\ref{prop:anchor_loss_coverage} further relates $\bar r$ to
population excess log loss on a covered anchor family. Architectural
analyticity alone supplies neither small calibration error nor a small
transport constant. For a finite context window these results apply only to
admissible $N$ and do not imply an infinite-length limit.

\noindent
Two features distinguish \eqref{eq:chain_main_bound} from what the
independent-occurrence reading would give. First, the decay rate is the
undegraded $\alpha$ of Theorem~\ref{thm:decay_1_5_app}: passing to an unsupported
instruction costs a constant factor, not a slower exponential. Second, no upper
bound on $\Omega$ is required, so there is no fixed admissible distance; the cost
of travelling further is paid once, in the constant.

\begin{corollary}[Certified distance under a uniform path-cost bound]
\label{cor:chain_reach}
Consider a collection of unsupported endpoints that share the supported anchor
$\vi_0$ and the fixed query contribution $\lambda_t$. For an endpoint
$\widetilde\vi$, let
$\Delta(\widetilde\vi):=\|E(\widetilde\vi)-E(\vi_0)\|_F$, and write
$\Omega(\widetilde\vi)$ for the cumulative cost of its continuation chain.
Suppose there is a constant $\mathcal L_{\mathrm{path}}>0$, independent of $N$ and of
the endpoint, such that
\begin{equation}
\label{eq:chain_uniform_distance_cost}
\Omega(\widetilde\vi)
\le \mathcal L_{\mathrm{path}}\Delta(\widetilde\vi)
\end{equation}
throughout this collection. Set $\beta_0:=\beta_\circ e^{\gamma_0}$. Then
\begin{equation}
\label{eq:chain_distance_error}
|h_N(1)|
\le
\beta_0\exp\!\left(-\alpha N
+\mathcal L_{\mathrm{path}}\Delta(\widetilde\vi)\right).
\end{equation}
Consequently, for any $\epsilon\in(0,1)$, $|h_N(1)|\le\epsilon$ whenever
\begin{equation}
\label{eq:chain_reach}
\Delta(\widetilde\vi)
\le
\frac{\alpha N-\log(\beta_0/\epsilon)}{\mathcal L_{\mathrm{path}}}.
\end{equation}
Thus, within an endpoint family satisfying the same uniform path-cost bound,
the certified distance grows linearly with $N$.
\end{corollary}

\begin{proof}
The proof of Theorem~\ref{thm:chain_unsupported_convergence} gives
$|h_N(1)|\le\beta_\circ e^{\gamma_0+\Omega(\widetilde\vi)}e^{-\alpha N}$.
Applying \eqref{eq:chain_uniform_distance_cost} proves
\eqref{eq:chain_distance_error}. Solving
$\beta_0e^{-\alpha N+\mathcal L_{\mathrm{path}}\Delta(\widetilde\vi)}\le\epsilon$
for $\Delta(\widetilde\vi)$ gives \eqref{eq:chain_reach}.
\end{proof}

%
%

\subsection{Mixed Supported and Unsupported Instructions}
\label{app:mixed_supported_unsupported}

We next consider a prompt in which most demonstrations display the same
supported instruction while a few display corrupted, unsupported surface forms,
for example misspellings. Write
\begin{equation}
    \label{eq:mixed_prompt}
    \vz_N:=\bigl((\vi_1,x_1,y_1),\ldots,(\vi_N,x_N,y_N),(\vi_q,\xquery)\bigr),
\end{equation}
let $S_N\subseteq[N]$ denote the supported instruction indices and
$U_N:=[N]\setminus S_N$ the unsupported ones, and set
\begin{equation}
    \label{eq:mixed_m_r}
    m_N:=|S_N|,\qquad r_N:=|U_N|=N-m_N .
\end{equation}
All supported slots display the same instruction, which we denote $\vi_S$; the
unsupported slots display arbitrary distinct values $\{\vi_n\}_{n\in U_N}$.

\begin{assumption}[Shared supported template, non-shared corruptions]
\label{ass:mixed_factorization}
The $m_N$ supported occurrences are renderings of one template chosen once for
the whole prompt and therefore contribute a single supported factor
$p(\vi_S\mid t)$, exactly as in Assumption~\ref{ass:chain_factorization};
when $m_N=0$, this factor is omitted. Each
unsupported occurrence is treated as a separate instruction slot. On supported
values the corresponding slot likelihoods factor independently of the
demonstration pairs given the task; at unsupported realized values their
log-likelihood ratios are defined by the fixed continuation chains introduced
below.
\end{assumption}

\noindent
This is the natural mixture of the shared reading for repeated supported copies
and the independent-observation reading for distinct corruptions. The latter do
not have literal Bayes likelihood ratios at zero-probability values; the object
below is therefore a continuation-induced posterior, denoted by $\bar p$.

For $t\neq t^{*}$ define
\begin{equation}
    \label{eq:mixed_lor_terms}
    \ell_{t,n}:=\log\frac{p(x_n,y_n\mid t)}{p(x_n,y_n\mid t^{*})},
    \qquad
    r_t(\vi):=\log\frac{p(\vi\mid t)}{p(\vi\mid t^{*})}
\end{equation}
for supported $\vi$. For every unsupported slot $n\in U_N$, let
$\bar r_{t,n}(\vi_n)$ be the endpoint value of its specified finite analytic
continuation chain. For the query define
\begin{equation}
    \label{eq:mixed_lambda}
    \bar\lambda_{t,q}:=
    \log\!\left(\frac{\pi(t)}{\pi(t^{*})}
    \frac{p(\xquery\mid t)}{p(\xquery\mid t^{*})}\right)
    +
    \begin{cases}
        r_t(\vi_q), & \vi_q\text{ supported},\\
        \bar r_{t,q}(\vi_q), & \vi_q\text{ unsupported},
    \end{cases}
\end{equation}
where $\bar r_{t,q}$ is obtained from a fixed finite chain in the unsupported
case. We then define the continuation-induced odds by
\begin{equation}
    \label{eq:mixed_exact_odds}
    \log \bar R_{t,N}
    :=\bar\lambda_{t,q}
     +\sum_{n=1}^{N}\ell_{t,n}
     +r_t(\vi_S)
     +\sum_{n\in U_N}\bar r_{t,n}(\vi_n),
\end{equation}
and normalize them as
\begin{equation}
    \label{eq:mixed_continued_posterior}
    \bar p(t^{*}\mid\vz_N):=\frac{1}{1+\sum_{s\neq t^{*}}\bar R_{s,N}},
    \qquad
    \bar p(t\mid\vz_N):=\frac{\bar R_{t,N}}
    {1+\sum_{s\neq t^{*}}\bar R_{s,N}}.
\end{equation}
The induced posterior predictive is
\begin{equation}
    \label{eq:mixed_continued_predictive}
    \bar p(\yquery\mid\vz_N)
    :=\sum_{t\in\gT}p(\yquery\mid\xquery,t)\,\bar p(t\mid\vz_N).
\end{equation}
The shared supported contribution $r_t(\vi_S)$ is a single
$N$-independent term; each distinct unsupported occurrence contributes one
continued log-ratio.

\begin{assumption}[Supported template and local continuation of the corruptions]
\label{ass:mixed_typos}
For every $t\neq t^{*}$ the shared supported template satisfies
\begin{equation}
    \label{eq:mixed_supported_nonpositive}
    r_t(\vi_S)\le 0 ,
\end{equation}
and for every $n\in U_N$ there exist a supported anchor $\vi_n^{0}$ with
$r_t(\vi_n^{0})\le0$ and a finite chain of analytic continuations connecting
$\vi_n^{0}$ to $\vi_n$, in the sense of Assumption~\ref{ass:chain_local}.
\end{assumption}

\noindent
Condition \eqref{eq:mixed_supported_nonpositive} is the anchor condition
\eqref{eq:chain_gamma0} applied to the shared template, and is a single
inequality rather than one per slot. The natural choice of anchor for a
misspelling is the uncorrupted rendering itself, $\vi_n^{0}=\vi_S$, in which case the
continuation distance $\|\vi_n-\vi_S\|_2$ is the size of the typo.

\begin{lemma}[Continuation cost of a corrupted occurrence]
\label{lmm:mixed_typo_bound}
Under Assumption~\ref{ass:mixed_typos}, let $\Omega_{t,n}$ denote the cumulative
continuation cost \eqref{eq:chain_Omega} of the chain reaching $\vi_n$ for the
task pair $(t,t^{*})$, and set
\begin{equation}
    \label{eq:mixed_Omega_n}
    \Omega_n:=\max_{t\neq t^{*}}\Omega_{t,n}<\infty .
\end{equation}
Then, for every $n\in U_N$ and every $t\neq t^{*}$,
\begin{equation}
    \label{eq:mixed_typo_bound}
    \bar r_{t,n}(\vi_n)\;\le\;\Omega_{t,n}\;\le\;\Omega_n .
\end{equation}
\end{lemma}

\begin{proof}
Apply Lemma~\ref{lmm:chain_telescope} to the chain connecting $\vi_n^{0}$ to
$\vi_n$: it gives $\bar r_{t,n}(\vi_n)\le\gamma_0+\Omega_{t,n}$ with
$\gamma_0=\max_{t\neq t^*}r_t(\vi_n^{0})\le0$ by
Assumption~\ref{ass:mixed_typos}. Finiteness of $\Omega_{t,n}$ is
Lemma~\ref{lmm:chain_cauchy}(a), and the maximum over the finite set $\gT$ is
finite.
\end{proof}

The following theorem separates the task information carried by the examples
from the continuation burden created by the corrupted occurrences.

\begin{theorem}[Formal mixed supported/unsupported instruction bound]
\label{thm:mixed_typos}
Let Assumptions~\ref{ass:mixed_factorization} and~\ref{ass:mixed_typos} hold.
Suppose that, almost surely for all sufficiently large $N$, there is a sequence
$a_N\ge0$ such that simultaneously for all $t\neq t^{*}$,
\begin{equation}
    \label{eq:mixed_demo_contraction_general}
    \sum_{n=1}^{N}\ell_{t,n}\le-a_N .
\end{equation}
Then, almost surely for all sufficiently large $N$,
\begin{equation}
    \label{eq:mixed_odds_bound_general}
    \frac{\bar p(t\mid\vz_N)}{\bar p(t^{*}\mid\vz_N)}
    \le e^{\bar\lambda_{t,q}}\exp\!\left(-a_N+\sum_{n\in U_N}\Omega_n\right),
    \qquad t\neq t^{*},
\end{equation}
and consequently, with $\beta_q:=(|\gT|-1)\max_{t\neq t^{*}}e^{\bar\lambda_{t,q}}$,
\begin{equation}
    \label{eq:mixed_predictive_bound_general}
    \bigl|\bar p(\yquery\mid\vz_N)-p(\yquery\mid\xquery,t^{*})\bigr|
    \le \beta_q\exp\!\left(-a_N+\sum_{n\in U_N}\Omega_n\right).
\end{equation}
For the actual pretrained model,
\begin{equation}
    \label{eq:mixed_model_bound_general}
    \bigl|M(\yquery\mid\vz_N)-p(\yquery\mid\xquery,t^{*})\bigr|
    \le \beta_q\exp\!\left(-a_N+\sum_{n\in U_N}\Omega_n\right)
    +\varepsilon_M(\vz_N),
\end{equation}
where $\varepsilon_M(\vz_N)$ is the pathwise transport error in
\eqref{eq:app_model_transport_error}.
In particular, continued posterior predictive convergence follows whenever
\begin{equation}
    \label{eq:mixed_general_condition}
    a_N-\sum_{n\in U_N}\Omega_n\longrightarrow+\infty .
\end{equation}
\end{theorem}

\begin{proof}
From \eqref{eq:mixed_exact_odds}, \eqref{eq:mixed_demo_contraction_general},
\eqref{eq:mixed_supported_nonpositive} and Lemma~\ref{lmm:mixed_typo_bound},
\[
    \log\frac{\bar p(t\mid\vz_N)}{\bar p(t^{*}\mid\vz_N)}
    \;\le\;\bar\lambda_{t,q}-a_N+\underbrace{r_t(\vi_S)}_{\le\,0}
      +\sum_{n\in U_N}\Omega_n
    \;\le\;\bar\lambda_{t,q}-a_N+\sum_{n\in U_N}\Omega_n ,
\]
which is \eqref{eq:mixed_odds_bound_general}. Since $\bar p(t^{*}\mid\vz_N)\le1$, we have $\bar p(t\mid\vz_N)\le \bar p(t\mid\vz_N)/\bar p(t^{*}\mid\vz_N)$; summing over the
$|\gT|-1$ competing tasks bounds the continued posterior mass outside $t^{*}$, and
\eqref{eq:mixed_predictive_bound_general} follows by repeating the calculation
of \eqref{eq:pp_diff_bound_by_posterior_mass} with $\bar p$ in place of $p$.
The triangle inequality and Lemma~\ref{lmm:model_reference_transport} then give
\eqref{eq:mixed_model_bound_general}.
\end{proof}

\begin{remark}[Consistency with Assumption~\ref{ass:chain_factorization}]
\label{rem:mixed_consistency}
Under Assumption~\ref{ass:mixed_factorization} the $m_N$ supported occurrences
contribute the single term $r_t(\vi_S)$ of \eqref{eq:mixed_exact_odds}, so
Lemma~\ref{lmm:At_uniform} applies to them with an $N$-independent constant.
Only the continued corrupted-slot contributions accumulate, and they do so at the rate $\sum_{n\in U_N}\Omega_n$
rather than at rate $N$. Setting $U_N=\varnothing$ recovers the shared-supported case of Theorem~\ref{thm:decay_1_5_app}, and setting $S_N=\varnothing$ with a single
repeated corruption recovers Theorem~\ref{thm:chain_unsupported_convergence}.
\end{remark}

We now specialize Theorem~\ref{thm:mixed_typos} to the setting used in the main
text: the surface instruction may be unsupported, but its associated pair
$(x_n,y_n)$ is still generated by the true task and remains informative. In
this case Lemma~\ref{lmm:Bt_exp_decay} gives $a_N=\alpha N$ directly.

\begin{corollary}[Evidence-preserving typos]
\label{cor:mixed_typos_full_evidence}
Under the hypotheses of Theorem~\ref{thm:mixed_typos}, suppose
$\sum_{n=1}^{N}\ell_{t,n}\le-\alpha N$ almost surely eventually, uniformly over
$t\neq t^{*}$. Then
\begin{equation}
    \label{eq:mixed_full_evidence_bound}
    \bigl|\bar p(\yquery\mid\vz_N)-p(\yquery\mid\xquery,t^{*})\bigr|
    \le\beta_q\exp\!\left(-\alpha N+\sum_{n\in U_N}\Omega_n\right),
\end{equation}
so convergence holds whenever
\begin{equation}
    \label{eq:mixed_average_cost_condition}
    \limsup_{N\to\infty}\frac1N\sum_{n\in U_N}\Omega_n<\alpha .
\end{equation}
If the corruption costs are uniformly bounded, $\Omega_n\le\Omega$, then
\begin{equation}
    \label{eq:mixed_uniform_cost_bound}
    \bigl|\bar p(\yquery\mid\vz_N)-p(\yquery\mid\xquery,t^{*})\bigr|
    \le\beta_q\exp\!\left(-N\Bigl[\alpha-\Bigl(1-\tfrac{m_N}{N}\Bigr)\Omega\Bigr]\right).
\end{equation}
\end{corollary}

\begin{proof}
Apply Theorem~\ref{thm:mixed_typos} with $a_N=\alpha N$, and use
$\sum_{n\in U_N}\Omega_n\le r_N\Omega=(N-m_N)\Omega$ for the last display.
\end{proof}

\begin{remark}[Asymptotic regimes]
\label{rem:mixed_regimes}
Under Corollary~\ref{cor:mixed_typos_full_evidence}:
\begin{enumerate}
    \item if $r_N=O(1)$, a fixed number of typos changes only the prefactor,
    $\bigl|\bar p(\yquery\mid\vz_N)-p(\yquery\mid\xquery,t^{*})\bigr|\le\beta_qe^{O(1)}e^{-\alpha N}$;
    \item if $r_N=o(N)$ and $\sup_n\Omega_n<\infty$, then
    $-\alpha N+\sum_{n\in U_N}\Omega_n=-\alpha N+o(N)$, so for every
    $\varepsilon>0$ the deviation is eventually $O(e^{-(\alpha-\varepsilon)N})$;
    \item if $r_N/N\to\rho\in[0,1]$ and
    $r_N^{-1}\sum_{n\in U_N}\Omega_n\to\overline\Omega$, the effective exponent is
    $\alpha_{\mathrm{eff}}=\alpha-\rho\overline\Omega$. Thus
    $\rho\overline\Omega<\alpha$ gives exponential convergence, equality gives no
    decay from the present bound, and $\rho\overline\Omega>\alpha$ is not covered.
\end{enumerate}
\end{remark}

\subsection{Varying unsupported instructions}
\label{app:format6_varying_unsupported}

Format~6 is the special case of the preceding framework in which every
per-demonstration instruction is a distinct unsupported realization. Thus
$S_N=\varnothing$, $U_N=[N]$, and the prompt is
\begin{equation}
    \label{eq:format6_prompt}
    \vz_N^{(6)}
    :=\bigl((\widetilde\vi_1,x_1,y_1),\ldots,
    (\widetilde\vi_N,x_N,y_N),(\widetilde\vi_q,\xquery)\bigr).
\end{equation}
For each $n$, let $\bar r_{t,n}(\widetilde\vi_n)$ be the endpoint value supplied
by a fixed finite analytic-continuation chain from a supported anchor, and absorb
the single query-instruction continuation into $\bar\lambda_{t,q}$. The
continuation-induced odds therefore satisfy the exact identity
\begin{equation}
    \label{eq:format6_exact_odds}
    \log\frac{\bar p(t\mid\vz_N^{(6)})}
                  {\bar p(t^{*}\mid\vz_N^{(6)})}
    =\bar\lambda_{t,q}
     +\sum_{n=1}^{N}\ell_{t,n}
     +\sum_{n=1}^{N}\bar r_{t,n}(\widetilde\vi_n),
    \qquad t\neq t^{*}.
\end{equation}
Unlike format~4, the continuation burden is paid once for every distinct
instruction realization and can therefore grow linearly with $N$.

\begin{corollary}[Conditional bound for format~6]
\label{cor:format6_conditional_bound}
Suppose that the demonstration pairs remain informative, so that almost surely,
for all sufficiently large $N$ and all $t\neq t^{*}$,
\begin{equation}
    \label{eq:format6_demo_decay}
    \sum_{n=1}^{N}\ell_{t,n}\le-\alpha N,
\end{equation}
and suppose every unsupported instruction admits a finite continuation chain
with cost $\Omega_n$ as in Lemma~\ref{lmm:mixed_typo_bound}. Then
\begin{equation}
    \label{eq:format6_conditional_predictive}
    \bigl|\bar p(\yquery\mid\vz_N^{(6)})
       -p(\yquery\mid\xquery,t^{*})\bigr|
    \le
    \beta_q\exp\!\left(-\alpha N+\sum_{n=1}^{N}\Omega_n\right).
\end{equation}
Consequently, convergence is guaranteed if
\begin{equation}
    \label{eq:format6_sufficient_condition}
    \alpha N-\sum_{n=1}^{N}\Omega_n\longrightarrow+\infty,
\end{equation}
or, more concretely, if
\begin{equation}
    \label{eq:format6_average_cost_condition}
    \limsup_{N\to\infty}\frac1N\sum_{n=1}^{N}\Omega_n<\alpha.
\end{equation}
The finite-chain assumption alone only gives $\Omega_n<\infty$ separately for
each $n$ and does not imply either condition.
\end{corollary}

\begin{proof}
Set $S_N=\varnothing$ and $U_N=[N]$ in
Theorem~\ref{thm:mixed_typos}, and take $a_N=\alpha N$ using
\eqref{eq:format6_demo_decay}. This gives
\eqref{eq:format6_conditional_predictive}. Conditions
\eqref{eq:format6_sufficient_condition} and
\eqref{eq:format6_average_cost_condition} make its right-hand side tend to zero.
\end{proof}

The preceding corollary gives only a sufficient condition. The next result shows
that this limitation is intrinsic: under the assumptions used so far, the
continued instruction evidence may cancel or dominate the demonstration
evidence.

\begin{theorem}[No unconditional convergence guarantee for format~6]
\label{thm:format6_no_unconditional_convergence}
Assume the query contribution $\bar\lambda_{t,q}$ is fixed, the demonstrations
are i.i.d. from $P_{X,Y\mid t^{*}}$, and the log-likelihood ratios are
integrable. For every $t\neq t^{*}$ define the average continued instruction
pressure
\begin{equation}
    \label{eq:format6_instruction_pressure}
    c_{t,N}:=\frac1N\sum_{n=1}^{N}
    \bar r_{t,n}(\widetilde\vi_n).
\end{equation}
Then almost surely
\begin{equation}
    \label{eq:format6_asymptotic_odds}
    \frac1N\log\frac{\bar p(t\mid\vz_N^{(6)})}
                         {\bar p(t^{*}\mid\vz_N^{(6)})}
    =c_{t,N}-\kappa_t+o(1),
    \qquad t\neq t^{*},
\end{equation}
where
$\kappa_t=\kl(P_{X,Y\mid t^{*}}\Vert P_{X,Y\mid t})>0$.
Consequently:
\begin{enumerate}
    \item[(i)] if
    \begin{equation}
        \label{eq:format6_positive_regime}
        \limsup_{N\to\infty}c_{t,N}<\kappa_t
        \qquad\text{for every }t\neq t^{*},
    \end{equation}
    then $\bar p(t^{*}\mid\vz_N^{(6)})\to1$ exponentially fast and the
    continued posterior predictive converges to
    $p(\yquery\mid\xquery,t^{*})$;

    \item[(ii)] if $\gT=\{t^{*},t^{\dagger}\}$ and
    \begin{equation}
        \label{eq:format6_negative_regime}
        \liminf_{N\to\infty}c_{t^{\dagger},N}
        >\kappa_{t^{\dagger}},
    \end{equation}
    then
    \begin{equation}
        \label{eq:format6_wrong_concentration}
        \bar p(t^{*}\mid\vz_N^{(6)})\longrightarrow0,
        \qquad
        \bar p(t^{\dagger}\mid\vz_N^{(6)})\longrightarrow1.
    \end{equation}
    Hence, whenever
    $p(\yquery\mid\xquery,t^{\dagger})
      \neq p(\yquery\mid\xquery,t^{*})$,
    \begin{equation}
        \label{eq:format6_predictive_failure}
        \bigl|\bar p(\yquery\mid\vz_N^{(6)})
        -p(\yquery\mid\xquery,t^{*})\bigr|
        \longrightarrow
        \bigl|p(\yquery\mid\xquery,t^{\dagger})
        -p(\yquery\mid\xquery,t^{*})\bigr|>0.
    \end{equation}
\end{enumerate}
Therefore the finite analytic-continuation assumptions by themselves do not
imply convergence for format~6.
\end{theorem}

\begin{proof}
Divide \eqref{eq:format6_exact_odds} by $N$. Since the query term is fixed,
$\bar\lambda_{t,q}/N\to0$. By the strong law of large numbers,
\begin{equation}
    \label{eq:format6_demo_slln}
    \frac1N\sum_{n=1}^{N}\ell_{t,n}
    \xrightarrow{\mathrm{a.s.}}-\kappa_t.
\end{equation}
Together with \eqref{eq:format6_instruction_pressure}, this proves
\eqref{eq:format6_asymptotic_odds}.

Under \eqref{eq:format6_positive_regime}, finiteness of $\gT$ gives a
$\delta>0$ such that, almost surely for all sufficiently large $N$,
\begin{equation}
    \label{eq:format6_odds_decay}
    \log\frac{\bar p(t\mid\vz_N^{(6)})}
                  {\bar p(t^{*}\mid\vz_N^{(6)})}
    \le-\delta N
    \qquad\text{for all }t\neq t^{*}.
\end{equation}
Summing the odds over the finite competing-task set yields
$1-\bar p(t^{*}\mid\vz_N^{(6)})=O(e^{-\delta N})$, and the posterior-predictive
claim follows exactly as in \eqref{eq:pp_diff_bound_by_posterior_mass}.

Under \eqref{eq:format6_negative_regime},
\eqref{eq:format6_asymptotic_odds} gives
\begin{equation}
    \label{eq:format6_odds_growth}
    \log\frac{\bar p(t^{\dagger}\mid\vz_N^{(6)})}
                  {\bar p(t^{*}\mid\vz_N^{(6)})}
    \longrightarrow+\infty
\end{equation}
linearly in $N$. In the two-task case this implies
\eqref{eq:format6_wrong_concentration}; substituting it into the definition
\eqref{eq:mixed_continued_predictive} gives
\eqref{eq:format6_predictive_failure}.
\end{proof}

%% file: NeurIPS2026/appendix/appendix_uniform_transport.tex
\subsubsection{Quantitative conditions for uniform local transport}
\label{app:uniform_local_transport}

The transport inequality is an upper certificate, not the actual discrepancy.
Write
\[
 D_N:=\sup_{s\in[0,1]}\sup_y
 |M(y\mid\vz_N(E_s))-p_{\mathrm{ref}}(y\mid\vz_N(E_s))|.
\]
Then $D_N\le\varepsilon_{M,N}$, but a small measured $D_N$ does not imply a
small value of the particular upper bound in
\eqref{eq:shared_model_transport_error}. Even when the two predictors agree
everywhere on a path, their separate Lipschitz constants can both be large.
We give sufficient conditions for the upper bound itself to be uniform and
locally small, followed by a sharper certificate based on their difference.

\paragraph{Why calibration and analyticity alone do not suffice.}
For finitely many distinct supported embeddings $a_1,\ldots,a_m$ and an
endpoint $b$ different from all of them, the polynomial
\[
 q(E):=\prod_{j=1}^m\frac{\|E-a_j\|_F^2}{\|b-a_j\|_F^2}
\]
vanishes on every $a_j$ and equals one at $b$. Adding $Aq(E)$ to one output
logit of an analytic predictor preserves its predictions at the supported
embeddings but can change its prediction at $b$ arbitrarily strongly as
$|A|$ increases. This demonstrates the insufficiency of finite symbolic
calibration and real analyticity alone; it does not assert exact representation
of this polynomial by a particular fixed Transformer. Quantitative stability
or further restrictions on the learned predictor are needed.

\paragraph{Length-independent stability in a tokenwise norm.}
For token representations $H=(h_1,\ldots,h_m)$, set
$\|H\|_{\infty,2}:=\max_j\|h_j\|_2$. Comparisons below are between prompts
of the same length and with the same mask; only the tied instruction changes.
The bounds will be independent of that length.

\begin{lemma}[Stability of an attention head]
\label{lmm:tokenwise_attention_stability}
Let $A_i(H)=\sum_{j\in I_i}a_{ij}(H)W_Vh_j$, where $I_i$ is a fixed nonempty
mask and $a_{ij}$ is the softmax of
$\langle W_Qh_i,W_Kh_j\rangle/\sqrt{d_k}$ over $I_i$.
If $\|H\|_{\infty,2},\|H'\|_{\infty,2}\le B$, and
$q=\|W_Q\|_2$, $k=\|W_K\|_2$, $v=\|W_V\|_2$, then
\begin{equation}
\label{eq:tokenwise_attention_stability}
 \|A(H)-A(H')\|_{\infty,2}
 \le v\left(1+\frac{4qkB^2}{\sqrt{d_k}}\right)
       \|H-H'\|_{\infty,2}.
\end{equation}
The bound is independent of the number of tokens and of $|I_i|$.
\end{lemma}
\begin{proof}
Put $\delta=\|H-H'\|_{\infty,2}$. Adding and subtracting a mixed inner
product bounds each logit difference by $2qkB\delta/\sqrt{d_k}$.
For a softmax vector $a$, its Jacobian maps $w$ to
$(a_j(w_j-\sum_l a_lw_l))_j$, whose $\ell_1$ norm is at most
$2\|w\|_\infty$. Integration along the segment of logits gives
$\|a_i(H)-a_i(H')\|_1\le4qkB\delta/\sqrt{d_k}$.
The change in values contributes at most $v\delta$, and the change in
weights at most $vB\|a_i(H)-a_i(H')\|_1$. Take the maximum over $i$.
A fixed additive position bias cancels in the logit difference.
\end{proof}

\begin{assumption}[Uniform architectural stability on the path family]
\label{ass:uniform_architectural_stability}
Let $\mathcal N$ be the set of admissible demonstration counts. The model
has fixed depth and head counts. On all paths $\vz_N(E_s)$, $N\in\mathcal N$,
and all teacher-forced output prefixes of length at most a fixed $H_\star$:
\begin{enumerate}[leftmargin=*]
\item Inputs to each attention head have a common token-norm bound $B_\ell$
independent of $N$, $s$, and the prefix. Projection norms are bounded
independently of $N$. Heads have the form of
Lemma~\ref{lmm:tokenwise_attention_stability}, or obey its estimate with
uniform constants; bounded affine biases may be included in those constants.
\item Tokenwise normalization and feed-forward maps, output projections,
position-processing maps, and residual and multi-head combination maps are
Lipschitz on the relevant domains with constants independent of $N$.
The input representation satisfies
$\|H_N(E)-H_N(E')\|_{\infty,2}\le C_{\mathrm{in}}\|E-E'\|_F$
on the path with $C_{\mathrm{in}}$ independent of $N$.
\item Outputs use fixed decoding probabilities and have length at most
$H_\star$, including any termination token. The generation kernel is defined
at every teacher-forced prefix used above.
\end{enumerate}
\end{assumption}

The input inequality holds with $C_{\mathrm{in}}=1$ for direct insertion of
tied embedding copies and unchanged additive positions. A positive fixed
normalization stabilizer and bounded normalization parameters give finite
tokenwise normalization constants. These observations do not establish all
the other uniform bounds for every architecture or length regime. The
resulting constants need not be numerically small, and greedy decoding is
not a Lipschitz probability map.

\begin{proposition}[Uniform model sensitivity along tied paths]
\label{prop:uniform_model_sensitivity}
Under Assumption~\ref{ass:uniform_architectural_stability}, there is a finite
$C_M$, independent of $N\in\mathcal N$, such that
\[
 \sup_y|M(y\mid\vz_N(E_s))-M(y\mid\vz_N(E_{s'}))|
 \le C_M\|E_s-E_{s'}\|_F.
\]
\end{proposition}
\begin{proof}
Apply Lemma~\ref{lmm:tokenwise_attention_stability} to each head. A residual
map $H\mapsto H+F(H)$ has constant at most $1+L_F$. Multi-head concatenation
has constant at most the square root of the sum of squared head constants,
followed by the output-projection norm. Composing these with the remaining
uniform bounds gives a next-token logit constant $C_{\mathrm{log}}$
independent of $N$ and the forced prefix. The softmax estimate above gives
a probability-coordinate constant at most $2C_{\mathrm{log}}$.
For a fixed block $y$, compare its autoregressive factors at the same forced
prefixes. The difference of two products of numbers in $[0,1]$ is at most
the sum of the factor differences. Thus $C_M=2H_\star C_{\mathrm{log}}$
suffices.
\end{proof}

\begin{lemma}[Uniform reference sensitivity]
\label{lmm:uniform_reference_sensitivity}
Let $d=\|E(\widetilde\vi)-E(\vi_0)\|_F$ and suppose
$|\frac{d}{ds}r_t(E_s)|\le Gd$ for every competing task on the fixed path.
Suppose the continued task weights are
\begin{align*}
 w_{t,N}(s)&=\frac{c_{t,N}\exp(r_t(E_s))}
 {1+\sum_{u\ne t^*}c_{u,N}\exp(r_u(E_s))},\\
 w_{t^*,N}(s)&=\frac{1}{1+\sum_{u\ne t^*}c_{u,N}\exp(r_u(E_s))},
\end{align*}
where $c_{t,N}$ are nonnegative and finite and do not depend on $s$.
Then $L_{\mathrm{ref},N}\le2G$ is an admissible choice for every $N$.
A finite $G$ exists when the finitely many continued scores are continuously
differentiable on a neighborhood of the fixed compact path.
\end{lemma}
\begin{proof}
Put $a_t=p(y\mid\xquery,t)\in[0,1]$ and $a_N(s)=\sum_tw_{t,N}(s)a_t$.
Differentiating the normalized weights yields
\[
 a_N'(s)=\sum_{t\ne t^*}w_{t,N}(s)(a_t-a_N(s))
                 \frac{d}{ds}r_t(E_s).
\]
Thus $|a_N'(s)|\le Gd\sum_{t\ne t^*}w_{t,N}(s)\le Gd$.
Integration and $\|E_s-E_{s'}\|_F=d|s-s'|$ prove even the constant $G$;
we retain $2G$ as a common upper bound. For $d=0$ the path is constant.
For $d>0$, compactness and the finite task set bound the directional
derivatives divided by $d$.
\end{proof}

\begin{proposition}[Uniform local transport certificate]
\label{prop:uniform_local_transport}
Under Proposition~\ref{prop:uniform_model_sensitivity} and
Lemma~\ref{lmm:uniform_reference_sensitivity}, suppose the supported anchors
satisfy $\sup_{N\in\mathcal N}r_{0,N}\le\bar r$. Admissible constants in
\eqref{eq:shared_model_transport_error} can then be chosen so that
\begin{equation}
\label{eq:uniform_local_transport_certificate}
 \sup_{N\in\mathcal N}D_N
 \le\sup_{N\in\mathcal N}\varepsilon_{M,N}(\widetilde\vi)
 \le\bar r+(C_M+2G)d.
\end{equation}
For $\epsilon>\bar r$ and $C_M+2G>0$, a path with
$d\le(\epsilon-\bar r)/(C_M+2G)$ has a certificate at most $\epsilon$.
If $C_M+2G=0$, its certificate is at most $\bar r$ at every admissible distance.
\end{proposition}
\begin{proof}
Apply the fixed-$N$ transport argument to every initial path segment.
Substitute $L_{M,N}\le C_M$ and $L_{\mathrm{ref},N}\le2G$, then take
the supremum over $N\in\mathcal N$.
\end{proof}

The proposition is deterministic conditional on the demonstration family.
If its hypotheses hold on a probability-one event for prefixes of one
infinite demonstration sequence, the conclusion is almost sure and uniform.
For a fixed context window, $\mathcal N$ contains only lengths that fit;
no $N\to\infty$ claim is then made. A common radius for a neighborhood of
endpoints requires common $C_M,G$ over that neighborhood. The proposition
does not certify that the experimental endpoints fall inside a numerically
useful radius.

\subsubsection{Supported calibration from population loss and coverage}
\label{app:anchor_loss_coverage}

\begin{proposition}[Calibration on a covered anchor family]
\label{prop:anchor_loss_coverage}
Let $Q$ be a distribution on supported prompts with true conditional $p$,
and use natural logarithms. Define its population excess log loss by
\[
 \Delta:=\mathbb E_{Z\sim Q}
 D_{\mathrm{KL}}(p(\cdot\mid Z)\Vert M(\cdot\mid Z)).
\]
For a finite anchor family $\mathcal A$ with $Q(z)\ge q_{\min}>0$ for
all $z\in\mathcal A$,
\begin{equation}
\label{eq:anchor_loss_coverage}
 \sup_{z\in\mathcal A,y}|M(y\mid z)-p(y\mid z)|
 \le\min\{1,\sqrt{\Delta/(2q_{\min})}\}.
\end{equation}
If all anchors in Proposition~\ref{prop:uniform_local_transport} belong to
$\mathcal A$, its bound holds with
$\bar r=\min\{1,\sqrt{\Delta/(2q_{\min})}\}$.
\end{proposition}
\begin{proof}
Nonnegativity gives
$Q(z)D_{\mathrm{KL}}(p(\cdot\mid z)\Vert M(\cdot\mid z))\le\Delta$.
For completeness, for discrete probability distributions $P,R$ let
$A=\{y:P(y)\ge R(y)\}$. The log-sum inequality gives
$D_{\mathrm{KL}}(P\Vert R)\ge
D_{\mathrm{KL}}(\operatorname{Ber}(P(A))\Vert\operatorname{Ber}(R(A)))$.
For fixed $b\in(0,1)$ the binary divergence as a function of $a$ has second
derivative $1/[a(1-a)]\ge4$ and vanishes with zero first derivative at $a=b$.
It is at least $2(a-b)^2$; boundary cases follow by limits. Since
$P(A)-R(A)=\|P-R\|_{\mathrm{TV}}$, this proves Pinsker's inequality.
Apply it at each anchor and bound a coordinate difference by total variation.
\end{proof}

Here $\Delta$ is excess \emph{population} loss, not empirical training loss.
A learning-theoretic analysis would also have to control estimation,
approximation, and optimization error. Coverage is substantive: a probability
distribution cannot assign a common positive mass to infinitely many distinct
prompts. This finite-family corollary therefore does not establish uniform
calibration over all lengths from one average pretraining-loss bound.

For a distributional alternative, suppose a target anchor law $R_N$ is
dominated by $Q$ with $dR_N/dQ\le C$ over the chosen lengths. Pinsker's
inequality and Cauchy--Schwarz give
\[
 \mathbb E_{R_N}\sup_y|M(y\mid Z)-p(y\mid Z)|\le\sqrt{C\Delta/2}.
\]
With a common path-stability bound and path length at most $d$, expected
off-support discrepancy is at most $\sqrt{C\Delta/2}+(C_M+2G)d$.
This is an expectation guarantee for each $N$, not an almost-sure bound on
the supremum over $N$. The density-ratio premise must itself be justified.

\subsubsection{A sharper certificate based on residual variation}
\label{app:residual_variation_transport}

\begin{proposition}[Residual-variation transport]
\label{prop:residual_variation_transport}
Suppose $g_N(s,y):=M(y\mid\vz_N(E_s))-p_{\mathrm{ref}}(y\mid\vz_N(E_s))$
is absolutely continuous in $s$ for each output, with an integrable common
envelope for its derivatives. Define
\begin{equation}
\label{eq:residual_variation_certificate}
 \varepsilon^{\mathrm{res}}_{M,N}
 :=r_{0,N}+\int_0^1\sup_y|\partial_s g_N(s,y)|\,ds.
\end{equation}
Then $D_N\le\varepsilon^{\mathrm{res}}_{M,N}$. With the same Lipschitz
constants as in \eqref{eq:shared_model_transport_error},
$\varepsilon^{\mathrm{res}}_{M,N}\le\varepsilon_{M,N}$.
\end{proposition}
\begin{proof}
The fundamental theorem of calculus bounds $|g_N(s,y)|$ by
$|g_N(0,y)|+\int_0^s|\partial_u g_N(u,y)|\,du$. Take suprema and extend
the integral to $[0,1]$. The individual derivative bounds are $L_{M,N}d$
and $L_{\mathrm{ref},N}d$ almost everywhere, proving the second inequality.
\end{proof}

Every endpoint model bound obtained by transport remains valid with
$\varepsilon^{\mathrm{res}}_{M,N}$ replacing the coarser certificate.
A uniform bound on these residual certificates gives a uniform model bound
without asserting a uniform bound on the original certificates. If the
predictors agree throughout the path, the residual certificate is zero.
Path-grid errors or sampled derivatives alone do not certify the integral;
an envelope or a controlled discretization error is still required.

%% file: NeurIPS2026/appendix/appendix_cot_modified.tex
\label{app:cot_main_proof}

This appendix formalizes the two claims in
Section~\ref{subsec:cot_main}: the Bayesian latent-task reference for direct ICL
cannot assign posterior mass to a zero-prior target task, whereas its
fixed-length, explicitly composed reference can converge by invoking
pretraining-supported subtasks. This is not an impossibility result for an
arbitrary Transformer. A guarantee for extracted step prompts does not by
itself imply the same guarantee for a full-context autoregressive rollout.
We quantify this distinction by a rollout discrepancy, prove the actual-model
bound without exact rollout compatibility, and recover exact compatibility
as a special case. The reference-rollout claim is stated separately.

\subsection{Reduction of an autoregressive CoT rollout to supported steps}

Fix an integer $K$ that does not depend on $N$. For every
$k=0,\ldots,K-1$, fix a dependency set
$J_k\subseteq\{0,\ldots,k\}$ with $k\in J_k$, independent of $N$, and a fixed
history map $\operatorname{enc}_k:\mathcal S^{|J_k|}\to\mathcal S$ of the
selected states into one step input.
For any trajectory, write
\begin{equation}
    \label{eq:cot_composition_app}
    h^{(k)}:=\operatorname{enc}_k\!\left((x^{(j)})_{j\in J_k}\right),
    \qquad
    x^{(k+1)}:=f_{t_k}(h^{(k)}),
    \qquad t_0,\ldots,t_{K-1}\in\gT.
\end{equation}
This finite-history recursion defines the target map by
$f_{t^\dagger}(x):=x^{(K)}$ when $x^{(0)}=x$. The special case
$J_k=\{k\}$ recovers ordinary composition of one-step maps. For a query
$\xquery$, define its correct intermediate states and histories recursively by
\begin{equation}
    \label{eq:cot_correct_states_app}
    x_q^{(0)}:=\xquery,
    \qquad
    h_q^{(k)}:=\operatorname{enc}_k\!\left((x_q^{(j)})_{j\in J_k}\right),
    \qquad
    x_q^{(k+1)}:=f_{t_k}(h_q^{(k)}),
    \quad k=0,\ldots,K-1.
\end{equation}
Thus $x_q^{(K)}=f_{t^\dagger}(\xquery)$.
Let $\mathcal H^{(k)}$ denote the set of encoded histories $h_q^{(k)}$
reachable at step $k$ as the initial query ranges over the finite state space
$\mathcal S$.

The full prompt $\vz_N$ contains the demonstration trajectories
$(x_n^{(0)},\ldots,x_n^{(K)})$. For each step $k$, let
\begin{equation}
    \label{eq:cot_step_prompt_app}
    h_n^{(k)}:=\operatorname{enc}_k\!\left((x_n^{(j)})_{j\in J_k}\right),
    \qquad
    \vz_N^{(k)}(h)
    :=\Bigl(\vi^{(k)},
       \bigl(h_n^{(k)},x_n^{(k+1)}\bigr)_{n=1}^{N},h\Bigr)
\end{equation}
be the corresponding single-step ICL prompt.

\begin{assumption}[Supported extracted-step domains]
\label{ass:cot_supported_step_domains}
Each $\vi^{(k)}$ is supported for $t_k$, and every correct step prompt in
\eqref{eq:cot_step_prompt_app} lies in the corresponding pretraining-supported
step domain. Any use of a padded model evaluation must separately justify the
local-model bound below and retain its padding discrepancy.
\end{assumption}

\begin{assumption}[Exact finite-history rollout compatibility for a specified predictor]
\label{ass:cot_rollout_compatibility}
For a specified predictor $R$ (either $p$ or $M$), every
$k=0,\ldots,K-1$, every output $u$, and every correct prefix
$x_q^{(0:k)}$, the next-token-block conditional of the full CoT rollout agrees
with the conditional induced by the extracted single-step prompt:
\begin{equation}
    \label{eq:cot_rollout_compatibility}
    R\!\left(
      \widehat x_q^{(k+1)}=u
      \mid \vz_N,
      \widehat x_q^{(0:k)}=x_q^{(0:k)}
    \right)
    =R\!\left(u\mid\vz_N^{(k)}(h_q^{(k)})\right).
\end{equation}
An assertion for $R=p$ is distinct from one for $R=M$; neither implies the
other. Conditionals are understood as generation kernels evaluated at the
specified prefix, including prefixes with zero rollout probability.
\end{assumption}

Assumption~\ref{ass:cot_rollout_compatibility} is a sufficient condition for
transferring a step-prompt bound with no compatibility correction. It is a
substantive finite-history conditional-sufficiency
assumption: on the correct prefix, the history--next-state demonstrations and
the selected history $h_q^{(k)}$ contain all information from the full CoT
prompt that is relevant to the next step. The states selected by $J_k$ are not
being treated as padding. Rather, under the deterministic subtask model the
correct process satisfies
\begin{equation}
    \label{eq:cot_markov_state_app}
    p\bigl(x_q^{(k+1)}\mid x_q^{(0:k)},t_k\bigr)
    =p\bigl(x_q^{(k+1)}\mid h_q^{(k)},t_k\bigr),
\end{equation}
because $x_q^{(k+1)}=f_{t_k}(h_q^{(k)})$. Thus the selected finite history,
rather than the current state alone, summarizes the correct query prefix for
the target transition. Equation
\eqref{eq:cot_rollout_compatibility} requires the model conditional to respect
the same sufficiency property. This does not follow from virtual padding:
padding accounts only for the future query positions that have not yet been
generated. States outside $J_k$ may still appear in the full prompt, but the
assumption requires that they provide no additional information about the next
step once $h_q^{(k)}$ is given. If this finite-history sufficiency fails, then
the equality in
\eqref{eq:cot_rollout_compatibility} can fail and an additional approximation
error must be included in the final bound. In particular, deterministic
sufficiency for the target transition does not establish sufficiency for a
learned model's conditional. Enlarging $J_k$ to include the complete query
history also does not remove dependence on other demonstrations, absolute
positions, or prompt formatting; the enlarged step domain must still satisfy
supportedness and identifiability.

\paragraph{Quantifying the failure of exact compatibility.}
Write $u_k^*:=f_{t_k}(h_q^{(k)})$ and define
\begin{align}
 Q^M_{k,N}(u\mid x_q^{(0:k)})
 &:=M(\widehat x_q^{(k+1)}=u\mid\vz_N,
                  \widehat x_q^{(0:k)}=x_q^{(0:k)}),\\
 P^M_{k,N}(u\mid h_q^{(k)})
 &:=M(u\mid\vz_N^{(k)}(h_q^{(k)})).
\end{align}
Both kernels are distributions over the same next-state token blocks,
including a failure outcome for malformed or nonterminating outputs if
needed. The full kernel includes all actually retained context. For fixed
demonstrations, take suprema over correct query prefixes and define
\begin{align}
 \rho_{k,N}
 &:=\sup_{x_q^{(0:k)}}
 \|Q^M_{k,N}(\cdot\mid x_q^{(0:k)})
       -P^M_{k,N}(\cdot\mid h_q^{(k)})\|_{\mathrm{TV}},
 \label{eq:cot_rollout_discrepancy}\\
 \rho^+_{k,N}
 &:=\sup_{x_q^{(0:k)}}
 [P^M_{k,N}(u_k^*\mid h_q^{(k)})
       -Q^M_{k,N}(u_k^*\mid x_q^{(0:k)})]_+.
 \label{eq:cot_one_sided_discrepancy}
\end{align}
Here $[a]_+=\max\{a,0\}$ and
$\|P-Q\|_{\mathrm{TV}}=\frac12\sum_u|P(u)-Q(u)|$.
In particular $0\le\rho^+_{k,N}\le\rho_{k,N}\le1$.
The one-sided discrepancy measures only a decrease in correct-transition
probability caused by the full context. Additional context that improves that
probability does not incur a penalty. Neither quantity is assumed small by
definition. The actual-model theorem below uses $\rho^+_{k,N}$ and therefore
also holds with the more conservative $\rho_{k,N}$.

\paragraph{Virtual padding for the step-$k$ length mismatch.}
At step $k$, the current input may be shorter than its pretraining reference for
two closely related reasons. First, the suffix
$x_q^{(k+1)},\ldots,x_q^{(K)}$ has not yet been generated. Second, the user's
prompt may simply be shorter than the longer contexts commonly observed during
pretraining. These are the same dimensional-alignment problem: in both cases we
choose the pretraining reference length $M_k$ and pad the shorter step prompt to
that length. The only difference is where the padding tokens are inserted. Let
\begin{equation}
    \label{eq:cot_virtual_padding_app}
    \operatorname{Pad}_{M_k}\!\left(\vz_N^{(k)}(h)\right)
\end{equation}
denote the resulting prompt after inserting $m_k$ virtual padding tokens while
leaving all informative tokens unchanged and preserving their relative order.
The padded step prompt and the pretraining prompt then belong to the same
fixed-length ambient space.

When the padding occupies only future positions, a causal mask gives those
positions zero attention weight, so the padding has no effect on the current
conditional. For padding inserted elsewhere to match a longer pretraining
context, or in an abstract fixed-length attention model without an exact mask,
the effect is only approximately zero. Lemma~\ref{lmm:pad_output_general}
controls the attention perturbation by the aggregate padding mass, of order
\begin{equation}
    \label{eq:cot_padding_mass_app}
    m_k e^{-\Delta_{\mathrm{att},k}},
\end{equation}
where $\Delta_{\mathrm{att},k}$ is the attention-logit gap between an informative token and
the most competitive padding token. Thus a large amount of weakly masked
padding can still be non-negligible.

To record this approximation explicitly, define the padding discrepancy
\begin{equation}
    \label{eq:cot_padding_discrepancy_app}
    \delta_{k,\mathrm{pad}}(N)
    :=\sup_{h,u}\left|
      M\!\left(u\mid\vz_N^{(k)}(h)\right)
      -M\!\left(u\mid
        \operatorname{Pad}_{M_k}(\vz_N^{(k)}(h))\right)
    \right|.
\end{equation}
Under exact causal masking this discrepancy is zero for padding placed only in
future positions. Under the attention-stability conditions of
Appendix~\ref{app:padding_general}, it is small when the induced network-output
perturbation is small. Combining the padded reference bound with the
model-transport inequality in Section~\ref{sec:instruction}, the corresponding
unpadded model error is at most
\begin{equation}
    \label{eq:cot_step_with_padding_app}
    \sup_{h\in\mathcal H^{(k)}}
    \left[1-M\!\left(f_{t_k}(h)\mid\vz_N^{(k)}(h)\right)\right]
    \le \beta_ke^{-\alpha_kN}+\varepsilon_M^{(k)}(N)
    +\delta_{k,\mathrm{pad}}(N).
\end{equation}
Here $\varepsilon_M^{(k)}(N)$ is an upper bound, uniform over
$h\in\mathcal H^{(k)}$, on the transport certificates in
\eqref{eq:app_model_transport_error} for the step-$k$ prompts being compared
with the reference. For supported unpadded step prompts one may take $r_0$;
otherwise the supremum over the finitely many correct-history paths suffices
at fixed $N$. The reference bound must apply to the same local or padded
prompts used in this comparison; inserting padding alone does not establish
supportedness or reference invariance. A smaller residual-variation
certificate from Proposition~\ref{prop:residual_variation_transport} may be
used when its hypotheses hold uniformly over these histories.

Including the one-sided rollout discrepancy, the sequential argument below gives
\begin{equation}
    \label{eq:cot_total_with_padding_app}
    \text{final CoT error}
    \le
    \sum_{k=0}^{K-1}
    \left(
      \beta_ke^{-\alpha_kN}
      +\varepsilon_M^{(k)}(N)
      +\delta_{k,\mathrm{pad}}(N)
      +\rho^+_{k,N}
    \right).
\end{equation}
The formal theorem below retains both corrections explicitly. Padding controls
the local prompt versus its padded version; it does not control the full
prompt versus the extracted local prompt.

For deterministic subtasks, define the step error by
\begin{equation}
    \label{eq:cot_step_error_app}
    \epsilon_k(N)
    :=
    \sup_{h\in\mathcal H^{(k)}}
    \left[
      1-p\!\left(f_{t_k}(h)\mid\vz_N^{(k)}(h)\right)
    \right],
\end{equation}
where $\mathcal H^{(k)}$ is the set of encoded histories reachable at step
$k$ under the correct recursion. The desired next state for history $h$ is
explicitly $f_{t_k}(h)$.

\begin{lemma}[Uniform supported-step convergence]
\label{lmm:cot_uniform_step_convergence}
Suppose that, for every $k$, the extracted pairs
$\bigl(h_n^{(k)},x_n^{(k+1)}\bigr)_{n=1}^{N}$ follow the supported single-step
generation mechanism and satisfy Assumption~\ref{assp:identifiability} for
task $t_k$. Then there are constants $\alpha_k,\beta_k>0$, independent of $N$,
such that, almost surely for all sufficiently large $N$,
\begin{equation}
    \label{eq:cot_uniform_step_convergence_app}
    \epsilon_k(N)\le\beta_k e^{-\alpha_kN}.
\end{equation}
\end{lemma}

\begin{proof}
Fix $k$ and $h\in\mathcal H^{(k)}$. Theorem~\ref{thm:decay_1_5}, applied to the
supported step task $t_k$ and the correct output $f_{t_k}(h)$, gives constants
$\alpha_{k,h},\beta_{k,h}>0$ and an almost-sure finite threshold $N_{k,h}$ such
that
\[
    1-p\!\left(f_{t_k}(h)\mid\vz_N^{(k)}(h)\right)
    \le \beta_{k,h}e^{-\alpha_{k,h}N}
\]
for all $N\ge N_{k,h}$. Because $J_k$ is finite and each state has bounded
length over the finite vocabulary $\gV$, the reachable encoded-history set
$\mathcal H^{(k)}$ is finite. We may therefore take
\[
    \alpha_k:=\min_{h\in\mathcal H^{(k)}}\alpha_{k,h}>0,
    \qquad
    \beta_k:=\max_{h\in\mathcal H^{(k)}}\beta_{k,h}<\infty,
    \qquad
    N_k:=\max_{h\in\mathcal H^{(k)}}N_{k,h}<\infty.
\]
The intersection of the finitely many probability-one events also has
probability one. Taking the supremum over $h\in\mathcal H^{(k)}$ proves
\eqref{eq:cot_uniform_step_convergence_app} for all $N\ge N_k$.
\end{proof}

\subsection{Formal rollout bound and its exact-compatibility specialization}

\begin{theorem}[Model bound with rollout discrepancy]
\label{thm:cot_main_app}
Let the fixed-length finite-history decomposition
\eqref{eq:cot_composition_app}, Assumption~\ref{ass:cot_supported_step_domains},
and the hypotheses of Lemma~\ref{lmm:cot_uniform_step_convergence} hold,
with the local-model bound \eqref{eq:cot_step_with_padding_app}.
If the theoretical rollout additionally satisfies
Assumption~\ref{ass:cot_rollout_compatibility} for $R=p$, then almost surely
for all sufficiently large $N$ and every output $y$,
\begin{equation}
    \label{eq:cot_main_bound_app}
    \left|
      p\!\left(y\mid\vz_N\right)
      -p(y\mid\xquery,t^\dagger)
    \right|
    \le
    \sum_{k=0}^{K-1}\beta_ke^{-\alpha_kN}
    \le
    K\beta_{\max}e^{-\alpha_{\min}N},
\end{equation}
where $\beta_{\max}:=\max_k\beta_k$ and
$\alpha_{\min}:=\min_k\alpha_k>0$.
For the actual pretrained model, define
\begin{equation}
    \label{eq:cot_formal_step_total_error}
    \eta_k(N):=\beta_ke^{-\alpha_kN}
    +\varepsilon_M^{(k)}(N)+\delta_{k,\mathrm{pad}}(N)+\rho^+_{k,N}.
\end{equation}
Without requiring exact compatibility for either $M$ or a theoretical full
rollout, the actual full autoregressive rollout satisfies, almost surely for
all sufficiently large $N$ and every $y$,
\begin{equation}
    \label{eq:cot_model_bound_app}
    \begin{aligned}
    \left|M\!\left(y\mid\vz_N\right)
    -p(y\mid\xquery,t^\dagger)\right|
    &\le 1-\prod_{k=0}^{K-1}(1-\min\{1,\eta_k(N)\})\\
    &\le\min\left\{1,\sum_{k=0}^{K-1}\eta_k(N)\right\}.
    \end{aligned}
\end{equation}
\end{theorem}

\begin{proof}
For the reference claim, assume compatibility only for $R=p$.
Let $\widehat x_q^{(0)}:=\xquery$ and let
$\widehat x_q^{(1)},\ldots,\widehat x_q^{(K)}$ be the states generated by the
full autoregressive CoT rollout. For each $k$, let
$\widehat h_q^{(k)}:=\operatorname{enc}_k((\widehat x_q^{(j)})_{j\in J_k})$.
Define
\begin{equation}
    \label{eq:cot_correct_event_app}
    E_k
    :=\left\{
      \widehat x_q^{(k+1)}
      =f_{t_k}(\widehat h_q^{(k)})
    \right\},
    \qquad
    F_k:=\bigcap_{j=0}^{k}E_j,
\end{equation}
with $F_{-1}$ equal to the whole sample space.

On $F_{k-1}$, induction using \eqref{eq:cot_correct_states_app} gives
$\widehat x_q^{(j)}=x_q^{(j)}$ for all $j\le k$, and in particular
$\widehat h_q^{(k)}=h_q^{(k)}\in\mathcal H^{(k)}$. Assumption~\ref{ass:cot_rollout_compatibility} and
\eqref{eq:cot_step_error_app} therefore imply
\begin{align}
    \sP(E_k^c\mid F_{k-1},\vz_N)
    &=1-p\!\left(
        f_{t_k}(h_q^{(k)})
        \mid\vz_N^{(k)}(h_q^{(k)})
      \right) \\
    &\le\epsilon_k(N).
    \label{eq:cot_conditional_step_failure_app}
\end{align}
No independence between the rollout steps is required. The conditional
statement is used only when $\sP(F_{k-1}\mid\vz_N)>0$; otherwise the joint
event $E_k^c\cap F_{k-1}$ has zero probability and the recursion below is
immediate.

Using $F_k^c=F_{k-1}^c\cup(E_k^c\cap F_{k-1})$ and the disjointness of this
union,
\begin{align}
    \sP(F_k^c\mid\vz_N)
    &\le
    \sP(F_{k-1}^c\mid\vz_N)
    +\sP(E_k^c\mid F_{k-1},\vz_N)
      \sP(F_{k-1}\mid\vz_N) \\
    &\le
    \sP(F_{k-1}^c\mid\vz_N)+\epsilon_k(N).
\end{align}
Iterating this inequality yields
\begin{equation}
    \label{eq:cot_rollout_failure_app}
    \sP(F_{K-1}^c\mid\vz_N)
    \le\sum_{k=0}^{K-1}\epsilon_k(N).
\end{equation}

On $F_{K-1}$, every step is correct and
\[
    \widehat x_q^{(K)}
    =f_{t^\dagger}(\xquery).
\]
Consequently,
\begin{equation}
    \label{eq:cot_correct_output_mass_app}
    p\!\left(
      f_{t^\dagger}(\xquery)\mid\vz_N
    \right)
    \ge1-\sum_{k=0}^{K-1}\epsilon_k(N).
\end{equation}
Because the target-task conditional is deterministic,
$p(y\mid\xquery,t^\dagger)=\mathbb{1}\{y=f_{t^\dagger}(\xquery)\}$. If
$y=f_{t^\dagger}(\xquery)$, equation
\eqref{eq:cot_correct_output_mass_app} bounds $1-p(y\mid\vz_N)$.
If $y\neq f_{t^\dagger}(\xquery)$, then
$p(y\mid\vz_N)$ is at most the failure probability in
\eqref{eq:cot_rollout_failure_app}. Thus, in both cases,
\begin{equation}
    \left|
      p(y\mid\vz_N)-p(y\mid\xquery,t^\dagger)
    \right|
    \le\sum_{k=0}^{K-1}\epsilon_k(N).
\end{equation}
Finally apply Lemma~\ref{lmm:cot_uniform_step_convergence} and use the finiteness
of $K$:
\[
    \sum_{k=0}^{K-1}\epsilon_k(N)
    \le\sum_{k=0}^{K-1}\beta_ke^{-\alpha_kN}
    \le K\beta_{\max}e^{-\alpha_{\min}N}.
\]
This proves \eqref{eq:cot_main_bound_app}.

For the model-level claim, start independently with states sampled from $M$.
On a correct prefix, \eqref{eq:cot_one_sided_discrepancy} gives
\[
 1-Q^M_{k,N}(u_k^*\mid x_q^{(0:k)})
 \le 1-P^M_{k,N}(u_k^*\mid h_q^{(k)})+\rho^+_{k,N}.
\]
Together with \eqref{eq:cot_step_with_padding_app}, this implies, whenever
$\sP_M(F_{k-1}\mid\vz_N)>0$,
\[
    \sP_M(E_k^c\mid F_{k-1},\vz_N)\le\min\{1,\eta_k(N)\}.
\]
Set $a_k=\min\{1,\eta_k(N)\}$. Multiplication of the conditional success
bounds yields $\sP_M(F_{K-1}\mid\vz_N)\ge\prod_k(1-a_k)$.
If a preceding success event has probability zero, the recursive inequality
$\sP_M(F_k\mid\vz_N)\ge(1-a_k)\sP_M(F_{k-1}\mid\vz_N)$ remains valid
without conditioning on a null event. The elementary inequality
$1-\prod_k(1-a_k)\le\sum_ka_k$ gives the additive bound.
On $F_{K-1}$ the final output is
$f_{t^\dagger}(\xquery)$, so the same two-case argument for
$y=f_{t^\dagger}(\xquery)$ and $y\neq f_{t^\dagger}(\xquery)$ proves
\eqref{eq:cot_model_bound_app}.
\end{proof}

\begin{corollary}[Execution through explicit step calls]
\label{cor:cot_explicit_step_execution}
Suppose an executor constructs $\vz_N^{(k)}(h)$ at each step and calls the
same model with exactly that context, position convention, and decoding
settings. It passes the generated state to the next step and retains no
additional model-visible history or cache. Its kernel equals $P^M_{k,N}$
by construction, so $\rho_{k,N}=\rho^+_{k,N}=0$. Consequently its final
error is at most
\[
 \min\left\{1,\sum_{k=0}^{K-1}
 [\beta_ke^{-\alpha_kN}+\varepsilon_M^{(k)}(N)
                       +\delta_{k,\mathrm{pad}}(N)]\right\}.
\]
\end{corollary}
\begin{proof}
The executor's transition kernel is precisely the model kernel on the
constructed step prompt, including on every correct prefix. Substitute zero
discrepancy into Theorem~\ref{thm:cot_main_app}.
\end{proof}

\paragraph{Scope of the main-text bound.}
The bound displayed in Theorem~\ref{thm:cot_main} is the specialization with
zero one-sided rollout discrepancy and zero padding discrepancy (or with
such corrections explicitly included in the step error). This is guaranteed
by the step-call protocol above with supported, unpadded prompts, or follows
under exact model compatibility and exact padding invariance. For ordinary
full-context autoregressive execution, the general bound is
\eqref{eq:cot_model_bound_app}; supported subtasks alone do not remove its
rollout correction. An exactly composed theoretical reference may be
defined using the local reference kernels, but this construction is not an
identity with an arbitrary full-context reference predictor.

For fixed $K$, vanishing reference errors and vanishing model, padding, and
one-sided rollout corrections imply convergence of the actual final output.
Uniformly small nonzero corrections give an error floor bounded by their
sum. No independence between steps is needed. Errors following an incorrect
prefix need not be controlled: the proof already charges the first failure,
and later recovery can only reduce the final error relative to this bound.

\paragraph{Assessing compatibility empirically.}
On matched demonstrations and a teacher-forced correct prefix, evaluate both
the full-context kernel and the extracted-step kernel. In finite-state
experiments their total variation and correct-state probability difference
can be computed directly; in large-output experiments the latter tests the
one-sided quantity without estimating the whole distribution. Varying history,
unrelated demonstrations, positions, and formatting while keeping the local
transition fixed probes the discarded information. A comparison with explicit
step calls separates execution-protocol effects from subtask learning.
Measurements on sampled prefixes are diagnostics, not supremum certificates.
The inequality between final error and the sum of full-rollout conditional
errors holds even without compatibility, so verifying that inequality alone
does not validate the local-to-full transfer.

\input{NeurIPS2026/appendix/appendix_growing_horizon}

%% file: NeurIPS2026/appendix/appendix_growing_horizon.tex
\subsection{Growing horizons and comparison of supplied plans}
\label{app:cot_growing_horizon}

The sequential argument is nonasymptotic in the number of steps: for each
finite $K$ it bounds final error whenever the corresponding local error
bounds hold. Fixed $K$ was used to obtain a common almost-sure onset of those
bounds by taking a maximum over finitely many steps and histories. To allow
$K=K(N)$, one must control the local bounds simultaneously along the growing
plan; pointwise convergence for each fixed step does not suffice.

For each $N$, let $\Pi_N$ be a supplied, valid finite-history plan with
$K_N:=K(N)\ge1$ steps, realizing a deterministic target $f_N^\dagger$.
The target, operations, and history maps may depend on $N$. The extracted
prompts must satisfy the same supported-domain conditions as above, and
the actual executor must be specified. Define its conditional final error by
\[
 e_N:=1-M_{\Pi_N}(f_N^\dagger(\xquery)\mid\vz_N).
\]
Here $M_{\Pi_N}$ is the final-output law of that executor, which may use
full-context rollout or explicit step calls. All claims below are uniform
over queries if their local premises are uniform over the associated correct
prefixes. Since the target is deterministic, $e_N$ also bounds every
output-probability coordinate difference from the target law.

\begin{corollary}[Growing-horizon execution]
\label{cor:cot_growing_horizon}
Suppose that for each $N$, on an event $G_N$ in the demonstration randomness,
the following local bounds hold simultaneously for every $k<K_N$ and every
correct history of $\Pi_N$:
\begin{equation}
\label{eq:cot_triangular_local_bound}
 1-P^M_{k,N}(u_k^*\mid h_q^{(k)})
 \le\beta_{k,N}e^{-\alpha_{k,N}N}
      +\varepsilon^{(k)}_{M,N}+\delta_{k,\mathrm{pad}}(N),
\end{equation}
where $\alpha_{k,N},\beta_{k,N}>0$. With the one-sided rollout discrepancy
from \eqref{eq:cot_one_sided_discrepancy}, put
\begin{align}
 c_{k,N}&:=\varepsilon^{(k)}_{M,N}
           +\delta_{k,\mathrm{pad}}(N)+\rho^+_{k,N},\\
 a_{k,N}&:=\min\{1,\beta_{k,N}e^{-\alpha_{k,N}N}+c_{k,N}\},\\
 a_N&:=\min_{k<K_N}\alpha_{k,N},\qquad
 b_N:=\max_{k<K_N}\beta_{k,N},\qquad
 C_N:=\sum_{k<K_N}c_{k,N}.
\end{align}
Then on $G_N$,
\begin{equation}
\label{eq:cot_growing_horizon_bound}
 e_N\le1-\prod_{k=0}^{K_N-1}(1-a_{k,N})
 \le\min\{1,K_Nb_Ne^{-a_NN}+C_N\}.
\end{equation}
In particular, sufficient conditions for $e_N\to0$ on sample paths on which
$G_N$ holds eventually are
\begin{equation}
\label{eq:cot_growing_horizon_conditions}
 a_NN-\log K_N-\log b_N\longrightarrow+\infty,
 \qquad C_N\longrightarrow0.
\end{equation}
If $c_{k,N}\le\bar c_N$ uniformly, the second condition is implied by
$K_N\bar c_N\to0$.
\end{corollary}
\begin{proof}
On $G_N$, the definition of $\rho^+_{k,N}$ transfers
\eqref{eq:cot_triangular_local_bound} to a conditional failure bound
$a_{k,N}$ for the full executor on a correct prefix. The success-event
recursion in Theorem~\ref{thm:cot_main_app}, applied at this particular $N$
for $K_N$ steps, gives the product bound. The union bound and
$\sum_{k<K_N}\beta_{k,N}e^{-\alpha_{k,N}N}\le K_Nb_Ne^{-a_NN}$
give the second inequality. Under \eqref{eq:cot_growing_horizon_conditions}
its right-hand side tends to zero. No independence between steps is used.
\end{proof}

\paragraph{When are the simultaneous premises justified?}
An almost-sure version requires that $G_N$ hold eventually on a common
probability-one event. This does not follow merely from almost-sure
convergence for every fixed $k$: step-dependent onset thresholds may grow
faster than $N$ along newly introduced steps. One sufficient construction is
a fixed finite library of supported local prompt problems, with a fixed
finite collection of admissible encoded histories. For each library problem,
use length-$N$ prefixes of the same underlying infinite demonstration stream,
and instantiate steps by reusing those local problems. If each library
problem satisfies Lemma~\ref{lmm:cot_uniform_step_convergence}, the minimum
rate, maximum prefactor, and maximum onset threshold over the finite library
and histories are uniform for any number of repetitions. Explicit step calls
with a fixed bank of demonstrations for each library problem provide one
such protocol; arbitrary trajectory-extracted demonstration marginals need
not have this property. This establishes
the simultaneous reference part of the local bound; model transport,
padding, and rollout discrepancy still require their own controls. Enlarging
history encodings with $K_N$, or drawing new datasets at every $N$, is not
covered by this finite-library argument without additional estimates.

Alternatively, suppose a finite-sample analysis gives
$\Pr(G_N^c)\le\delta_N$. Then \eqref{eq:cot_growing_horizon_bound} holds
with probability at least $1-\delta_N$. If its right-hand side is bounded on
$G_N$ by a deterministic $B_N\in[0,1]$, then
\begin{equation}
\label{eq:cot_growing_horizon_expected}
 \mathbb E[e_N]\le B_N+(1-B_N)\delta_N\le B_N+\delta_N.
\end{equation}
Indeed, split the expectation over $G_N$ and its complement and use $e_N\le1$.
If each of the $K_N$ step bounds, already uniform over its correct histories,
fails with probability at most $d_N$, a union bound gives
$\delta_N\le\min\{1,K_Nd_N\}$ without independence. A per-history guarantee
would require a further union bound or another uniformity argument.
When $\sum_N\delta_N<\infty$, the first Borel--Cantelli lemma gives eventual
validity almost surely. Together with \eqref{eq:cot_growing_horizon_conditions}
this recovers almost-sure convergence; $\delta_N\to0$ alone yields the
corresponding convergence-in-probability statement when $B_N\to0$.

\paragraph{Explicit scaling regimes.}
Suppose $a_N\ge\alpha>0$ and $b_N\le\beta<\infty$ uniformly.
For a composed reference, or an ideal executor with zero corrections,
$\log K_N=o(N)$ is sufficient. More generally, for
$K_N=O(e^{\gamma N})$ the reference contribution vanishes when $\gamma<\alpha$.
For actual-model execution the correction sum is equally important:
\begin{itemize}[leftmargin=*]
\item If $K_N=O(N^q)$ and $\bar c_N=O(N^{-p})$, with $0\le q<p$, then
$e_N=O(N^qe^{-\alpha N}+N^{q-p})$ on the simultaneous good events.
\item If $K_N=O(e^{\gamma N})$ and $\bar c_N=O(e^{-\lambda N})$, then
$e_N=O(e^{-(\alpha-\gamma)N}+e^{-(\lambda-\gamma)N})$ for
$0\le\gamma<\min\{\alpha,\lambda\}$.
\item A fixed positive upper bound on $\bar c_N$ does not provide a vanishing
growing-horizon certificate, even if the reference identification term vanishes.
This is a limitation of the guarantee, not a lower bound on every model's error.
\end{itemize}
These conclusions require the uniform rates and prefactors just stated. If
longer plans introduce less identifiable operations or more difficult query
states, the general condition involving $a_N,b_N$ must be used instead.

The accumulation can be unavoidable. Consider an executor which, conditional
on a correct prefix, makes an error with probability $q_N$ at every step,
and whose error state is absorbing with an incorrect final output. Its exact
final error is $1-(1-q_N)^{K_N}$. It tends to zero if and only if
$K_Nq_N\to0$: sufficiency follows from the union bound, and necessity from
$1-(1-q_N)^{K_N}\ge1-e^{-K_Nq_N}$. Thus constant conditional error cannot
support arbitrary horizon growth in general. This example concerns actual
conditional failures, not the possibly loose upper certificates $c_{k,N}$.

\paragraph{Demonstration complexity for a requested horizon.}
For a supplied horizon $K$ and tolerance $0<\epsilon<1$, uniform constants
$\alpha,\beta$ and a correction bound $\bar c_N$ give the sufficient conditions
\begin{equation}
\label{eq:cot_horizon_sample_complexity}
 N\ge\max\{0,\alpha^{-1}\log(2K\beta/\epsilon)\},
 \qquad \bar c_N\le\frac{\epsilon}{2K},
\end{equation}
provided the simultaneous local bounds are valid at that $N$. Substitution
in \eqref{eq:cot_growing_horizon_bound} gives $e_N\le\epsilon$.
Only the identification contribution has logarithmic dependence on $K$;
the required accuracy of the corrections tightens as $1/K$. Moreover, $N$
counts demonstrations per step, not total tokens. A full prompt containing
$N$ trajectories of length $K$ generally has length of order $NK$ for
bounded-size states. The asymptotic regimes presume that the required
contexts and states remain admissible. Explicit step calls can avoid storing
all trajectories in each call, but still incur $K$ model invocations.

\paragraph{Comparing plans of different lengths.}
For supplied plans $\Pi$ realizing the same target at a given $N$, define
the untruncated certificate
\begin{equation}
\label{eq:cot_plan_comparison_certificate}
 S_N(\Pi):=\sum_{k<K(\Pi)}
 [\beta_{\Pi,k,N}e^{-\alpha_{\Pi,k,N}N}+c_{\Pi,k,N}].
\end{equation}
Each plan has $e_N(\Pi)\le\min\{1,S_N(\Pi)\}$ whenever its premises hold.
Consequently a longer plan may have a smaller certificate if its steps are
more identifiable or have smaller corrections. For example, with zero
corrections and homogeneous constants within each of two plans, a longer
plan $\Pi_\ell$ has a smaller untruncated certificate than $\Pi_s$ exactly when
\[
 \alpha_\ell-\alpha_s>
 \frac1N\log\frac{K(\Pi_\ell)\beta_\ell}{K(\Pi_s)\beta_s}.
\]
This follows by taking logarithms of their positive exponential certificates.
It compares guarantees, not actual errors, and does not imply that the
shortest plan or the plan with the best bound is always empirically best.
Selecting among candidate plans using the same data requires simultaneous
validity over the candidates; validity for each fixed plan alone does not
justify adaptive selection.

\paragraph{Separation from plan discovery.}
These results analyze a supplied plan family and do not establish that a
model can discover a valid decomposition. To record the missing component,
let a planner output $\widehat\Pi_N$, and let $A_N$ be the event that it
realizes the intended target using admissible supported steps and satisfies
the required execution certificates. If $\Pr(A_N^c)\le\xi_N$ and, conditional
on the planner output and its data, execution error on $A_N$ is at most a
deterministic $B_N\in[0,1]$, then the law of total probability gives
\[
 \Pr(\text{final failure})\le\xi_N+(1-\xi_N)B_N.
\]
No independence of planning and execution is needed, but the conditional
execution guarantee is essential. Bounding $\xi_N$ or designing a planner
that makes it small is a separate problem not solved by the composition
theorem. The growing-horizon result strengthens execution scaling without
assuming that plan discovery follows from support of the component operations.

%% file: NeurIPS2026/appendix/appendix_padding.tex
\label{app:padding}

This appendix formalizes the fixed-length padding reduction used in two places
in the paper. First, the instruction path in
Section~\ref{subsec:unsupported_shared} must compare instructions containing
different numbers of tokens in one common Euclidean space. Second, the
step-wise CoT analysis in Appendix~\ref{app:cot_main_proof} sometimes compares
a partially generated prompt with a longer pretraining reference. These are
instances of the same dimensional-alignment problem, although the padding
tokens occupy different positions. We define the common padding operator,
quantify its effect on attention, and state precisely what this reduction does
and does not establish.

\subsection{Two alignment problems}

The first problem is \emph{instruction-length alignment}. The path
\begin{equation}
    \label{eq:padding_instruction_path}
    E_s=(1-s)E(\vi_0)+sE(\widetilde\vi),\qquad s\in[0,1],
\end{equation}
requires $E(\vi_0)$ and $E(\widetilde\vi)$ to belong to the same space, even
when $\vi_0$ and $\widetilde\vi$ contain different numbers of tokens.

The second problem is \emph{rollout-length alignment}. At CoT step $k$, the
prompt $\vz_N^{(k)}(h)$ may be shorter than the reference prompt used for the
corresponding pretraining-supported subtask. One reason is that the future
query states $x_q^{(k+1)},\ldots,x_q^{(K)}$ have not yet been generated;
another is that the selected pretraining reference may simply contain a longer
context. Virtual padding places both prompts in a common fixed-length ambient
space.

Padding creates a common representation by definition. Its effect on the
model, however, is small only under additional masking or attention-separation
conditions. The results below make this distinction explicit.

\subsection{Right-padded instruction representations}

Recall the token embedding map
$e:\gV\cup\{\langle\mathrm{pad}\rangle\}\to\mathbb R^d$ from
Section~\ref{subsec:instruction_setup}, with
$e(\langle\mathrm{pad}\rangle)=0$. For
$\vi=(v_1,\ldots,v_m)$ with $m\le L_i$, its right-padded representation is
\begin{equation}
    \label{eq:padding_instruction_embedding}
    E(\vi)
    :=\bigl[e(v_1),\ldots,e(v_m),0,\ldots,0\bigr]^{\top}
    \in\mathbb R^{L_i\times d}.
\end{equation}
Thus every instruction occupies the same $L_i$-position instruction slot.

\begin{remark}[Why the position convention matters]
\label{rem:padding_position_convention}
Inserting padding before an informative token and then reindexing the prompt
can change that token's absolute or rotary position. A zero token embedding
does not prevent this positional change. Right-padding is therefore the
canonical position-preserving convention for the instruction sequence itself:
it leaves the indices of all instruction tokens unchanged. When the
instruction is embedded into a larger prompt, we additionally assume either
that the prompt reserves an $L_i$-position instruction slot or that the
position identifiers of the subsequent informative tokens are preserved.
Without one of these alignment conventions, the informative representations
can differ already at the input layer, and the zero-initial-perturbation
condition used below need not hold.
\end{remark}

For a general prompt $\vz$, let $\operatorname{Pad}_M(\vz)$ denote its
position-preserving extension to length $M$: informative tokens and their
relative order are unchanged, and the added positions contain
$\langle\mathrm{pad}\rangle$. In the applications below, the added positions
are either trailing positions, reserved instruction-slot positions, or masked
future positions.

\subsection{Attention stability under padding}
\label{app:padding_general}

We first isolate one attention head at a fixed informative query position.
Let $\mathcal I$ and $\mathcal P$ index the informative and padding positions,
respectively, with $|\mathcal P|=m$. The following result conditions on the
query, keys, and values of the informative positions being aligned between the
two computations. This condition is automatic for the first layer under the
position convention above and is imposed layer by layer in the network-level
result that follows.

Write $x_i$ for the attention logit and $v_i$ for the value vector at position
$i$, and define
\begin{equation}
    \label{eq:padding_attention_outputs}
    y=\frac{1}{Z}\sum_{i\in\mathcal I}e^{x_i}v_i,
    \qquad
    y'=\frac{1}{Z+Z_{mathrm p}}
    \left(\sum_{i\in\mathcal I}e^{x_i}v_i
    +\sum_{p\in\mathcal P}e^{x_p}v_p\right),
\end{equation}
where
$Z:=\sum_{i\in\mathcal I}e^{x_i}$ and
$Z_{mathrm p}:=\sum_{p\in\mathcal P}e^{x_p}$.

\begin{lemma}[Attention stability under padding]
\label{lmm:pad_output_general}
Suppose $\|v_i\|_2\le C_V$ at every informative and padding position. Let
\begin{equation}
    \label{eq:padding_attention_gap}
    x_{\mathrm{p,max}}:=\max_{p\in\mathcal P}x_p,
    \qquad
    \Delta_{\mathrm{att}}:=x_j-x_{\mathrm{p,max}}>0
\end{equation}
for some informative position $j\in\mathcal I$. Then
\begin{equation}
    \label{eq:padding_attention_bound}
    \|y'-y\|_2
    \le 2C_V\frac{Z_{\mathrm p}}{Z+Z_{\mathrm p}}
    \le 2C_V\frac{Z_{\mathrm p}}{Z}
    \le 2C_Vm e^{-\Delta_{\mathrm{att}}}.
\end{equation}
Consequently, for any $\varepsilon>0$, the condition
\begin{equation}
    \label{eq:padding_attention_tolerance}
    m\le \frac{\varepsilon e^{\Delta_{\mathrm{att}}}}{2C_V}
\end{equation}
is sufficient for $\|y'-y\|_2\le\varepsilon$.
\end{lemma}

\begin{proof}
Let
$A:=\sum_{i\in\mathcal I}e^{x_i}v_i$ and
$B:=\sum_{p\in\mathcal P}e^{x_p}v_p$. Subtracting the two expressions in
\eqref{eq:padding_attention_outputs} gives
\begin{equation}
    y'-y=\frac{B}{Z+Z_{\mathrm p}}
    -\frac{Z_{\mathrm p}A}{Z(Z+Z_{\mathrm p})}.
\end{equation}
Because $\|A\|_2\le C_VZ$ and $\|B\|_2\le C_VZ_{\mathrm p}$, the triangle
inequality yields
\begin{equation}
    \|y'-y\|_2
    \le 2C_V\frac{Z_{\mathrm p}}{Z+Z_{\mathrm p}}.
\end{equation}
Moreover,
$Z_{\mathrm p}\le m e^{x_{\mathrm{p,max}}}$ and
$Z\ge e^{x_j}$, and hence
$Z_{\mathrm p}/Z\le m e^{-\Delta_{\mathrm{att}}}$. This proves
\eqref{eq:padding_attention_bound};
\eqref{eq:padding_attention_tolerance} follows immediately.
\end{proof}

\begin{remark}[Exactly masked future positions]
\label{rem:padding_causal_mask}
In the first rollout-alignment case, padding occupies positions representing
future query states. If these positions lie strictly after the current
prediction position, the causal mask assigns them attention logit $-\infty$.
Then $Z_{\mathrm p}=0$, and the perturbation in
Lemma~\ref{lmm:pad_output_general} is exactly zero. The quantitative bound is
needed only when padding remains visible to an informative query, as may happen
in an instruction slot or when matching a longer visible pretraining context.
\end{remark}

\subsection{From attention perturbations to output log-probabilities}

Lemma~\ref{lmm:pad_output_general} concerns a single attention call. To obtain
a model-level statement, the perturbation must be propagated through all
layers. We state this step in a form that accommodates multi-head attention,
output projections, residual connections, normalization, and MLP blocks.

Consider a Transformer with $H$ blocks. After aligning the informative
positions, let $\eta^{(\ell)}$ be the maximum norm of the hidden-state
difference over those positions after block $\ell$. Assume that, on the compact
set of hidden states reached by the two forward passes,
\begin{enumerate}[label=(\roman*),leftmargin=*]
    \item block $\ell$ is $K_\ell$-Lipschitz with respect to perturbations of
    the informative hidden states; and
    \item with the informative input states held fixed, adding the padding
    positions changes the informative outputs of block $\ell$ by at most
    $\varepsilon_\ell$.
\end{enumerate}
Lemma~\ref{lmm:pad_output_general}, combined with the head-concatenation and
attention-output projection norms, gives a bound of the form
\begin{equation}
    \label{eq:padding_layer_direct_error}
    \varepsilon_\ell
    \le C_{\mathrm{att},\ell}m e^{-\Delta_{\mathrm{att},\ell}}
\end{equation}
whenever the value norms and attention gaps are controlled uniformly at that
layer.

\begin{proposition}[Transformer-output stability under padding]
\label{prop:padding_transformer_output}
Suppose the position-alignment condition of
Remark~\ref{rem:padding_position_convention} holds, so that
$\eta^{(0)}=0$. Under conditions \textup{(i)}--\textup{(ii)},
\begin{equation}
    \label{eq:padding_hidden_propagation}
    \eta^{(H)}
    \le
    \sum_{\ell=0}^{H-1}\varepsilon_\ell
    \prod_{r=\ell+1}^{H-1}K_r,
\end{equation}
where an empty product equals one. If the output readout is
$K_{\mathrm{out}}$-Lipschitz from the hidden-state norm to the
$\ell_\infty$ logit norm, then, for every output token $u$,
\begin{equation}
    \label{eq:padding_log_probability_bound}
    \left|
      \log M(u\mid\operatorname{Pad}_M(\vz))
      -\log M(u\mid\vz)
    \right|
    \le
    \varepsilon_{\mathrm{pad}}
    :=2K_{\mathrm{out}}
    \sum_{\ell=0}^{H-1}\varepsilon_\ell
    \prod_{r=\ell+1}^{H-1}K_r.
\end{equation}
\end{proposition}

\begin{proof}
At block $\ell$, add and subtract the padded block output evaluated at the
unpadded informative hidden states. Conditions \textup{(i)}--\textup{(ii)}
then give
\begin{equation}
    \eta^{(\ell+1)}\le K_\ell\eta^{(\ell)}+\epsilon_\ell.
\end{equation}
Starting from $\eta^{(0)}=0$ and iterating this recursion proves
\eqref{eq:padding_hidden_propagation}. For logit vectors $a,b$,
\begin{equation}
    \left|log\operatorname{softmax}(a)_u
    -\log\operatorname{softmax}(b)_u\right|
    \le 2\|a-b\|_\infty,
\end{equation}
because the log-sum-exp map is one-Lipschitz in the $\ell_\infty$ norm.
Combining this fact with the readout bound proves
\eqref{eq:padding_log_probability_bound}.
\end{proof}

The products of the block Lipschitz constants in
\eqref{eq:padding_hidden_propagation} are worst-case amplification factors and
need not be tight. The substantive quantity inherited from
Lemma~\ref{lmm:pad_output_general} is
$m e^{-\Delta_{\mathrm{att},\ell}}$: many weakly masked padding positions can
produce a non-negligible perturbation, whereas a sufficiently large attention
gap suppresses their aggregate effect exponentially.

\subsection{Consequence for the instruction path}

\begin{corollary}[A common space for instructions]
\label{cor:padding_instruction_space}
If $|\vi_0|,|\widetilde\vi|\le L_i$, then
\begin{equation}
    \label{eq:padding_instruction_distance}
    \Delta_{\mathrm{inst}}
    :=\|E(\widetilde\vi)-E(\vi_0)\|_F
\end{equation}
is well defined, and the entire path in
\eqref{eq:padding_instruction_path} lies in
$\mathbb R^{L_i\times d}$. The number of padding positions in either endpoint
is at most $L_i$ and is therefore independent of the number $N$ of
demonstrations.
\end{corollary}

The corollary supplies the dimensional fact required by
the finite-continuation condition; it does not itself prove the existence of
the analytic continuation assumed there. In the paper's formulation, the
continued instruction likelihood is defined directly on the padded instruction-embedding
space, so no extra padding term is required in the continuation bound. If an
implementation instead compares padded model evaluations with their unpadded
counterparts, Proposition~\ref{prop:padding_transformer_output} provides an
$N$-independent endpoint perturbation. Such a fixed perturbation changes only
the prefactor in Theorem~\ref{thm:repeated_unsupported}, not its exponential
rate in $N$.

\subsection{Consequence for the CoT rollout}

Recall from Appendix~\ref{app:cot_main_proof} that the padding discrepancy at
step $k$ is
\begin{equation}
    \label{eq:padding_cot_discrepancy_recall}
    \delta_{k,\mathrm{pad}}(N)
    :=\sup_{h,u}\left|
      M\!\left(u\mid\vz_N^{(k)}(h)\right)
      -M\!\left(u\mid
        \operatorname{Pad}_{M_k}(\vz_N^{(k)}(h))\right)
    \right|.
\end{equation}

\begin{corollary}[Padding contribution to the CoT bound]
\label{cor:padding_cot_bound}
Suppose Proposition~\ref{prop:padding_transformer_output} holds uniformly over
the encoded histories $h$ and outputs $u$ in
\eqref{eq:padding_cot_discrepancy_recall}. Then
\begin{equation}
    \label{eq:padding_cot_discrepancy_bound}
    \delta_{k,\mathrm{pad}}(N)
    \le \varepsilon_{\mathrm{pad}}^{(k)}(N),
\end{equation}
where $\varepsilon_{\mathrm{pad}}^{(k)}(N)$ is obtained from
\eqref{eq:padding_log_probability_bound} using the step-$k$ padding count and
layerwise attention gaps. Consequently, the CoT error obeys
\begin{equation}
    \label{eq:padding_cot_total_bound}
    \left|
      M(\yquery\mid\vz_N)
      -p(\yquery\mid\xquery,t^\dagger)
    \right|
    \le
    \sum_{k=0}^{K-1}
    \left(\beta_ke^{-\alpha_kN}
    +\varepsilon_M^{(k)}(N)
    +\delta_{k,\mathrm{pad}}(N)\right).
\end{equation}
If the added positions represent only masked future states, then
Remark~\ref{rem:padding_causal_mask} gives
$\delta_{k,\mathrm{pad}}(N)=0$. If the step prompt is itself supported, then
$\varepsilon_M^{(k)}(N)=r_0$, and
\eqref{eq:padding_cot_total_bound} reduces to the bound in
Theorem~\ref{thm:cot_main}.
\end{corollary}

\begin{proof}
For probabilities $a,b\in(0,1]$, the mean-value theorem applied to the
exponential function on $(-\infty,0]$ gives
$|a-b|\le|\log a-\log b|$. Hence
\eqref{eq:padding_log_probability_bound}, uniformly over $x$ and $u$, implies
\eqref{eq:padding_cot_discrepancy_bound}. Adding this discrepancy and the
model-transport bound from Section~\ref{sec:instruction} to the supported-step
error, and repeating the union-bound argument used in the proof
of Theorem~\ref{thm:cot_main} gives
\eqref{eq:padding_cot_total_bound}. The exactly masked case follows from
Remark~\ref{rem:padding_causal_mask}.
\end{proof}

For visible padding used to match a longer pretraining context, neither the
padding count nor the attention gap is automatically favorable; the additional
term in \eqref{eq:padding_cot_total_bound} must then be retained. In particular,
the fact that a prompt becomes longer as $N$ grows does not by itself prove
that the padded and unpadded conditionals become closer. The original
The padding-free form of Theorem~\ref{thm:cot_main} is recovered only when
padding is exactly masked or when $\delta_{k,\mathrm{pad}}(N)$ is otherwise
negligible at the desired rate.

%% file: NeurIPS2026/appendix/appendix_experiment.tex
\label{app:experiment}

In this section we provide detailed introduction of our experiments setup and additional
experimental results to supplement those in the main text.

\subsection{Pretraining Dependence}
\label{app:pretrain_error} To isolate the role of \emph{intrinsic ICL capability},
we construct three variants of \texttt{Llama3.2-3B} that differ only in the
amount and \emph{format} of additional pretraining on a six-digit addition corpus.
Throughout, we use the same evaluation protocol in the main text (6--10 digit addition).

\textbf{Base (Weak).} We use the original \texttt{Llama3.2-3B} checkpoint
without any additional training.

\textbf{Fine-tuned (Strong).} We randomly generate $20{,}000$ six-digit addition
instances and format each instance as a supervised instruction-following sample in
JSONL, each training example is serialized as
\[
    \texttt{Instruction: \ldots\ \ \ Input: \ldots\ \ \ Response: \ldots }
\]
and we fine-tune the model using LoRA on this dataset.

\textbf{Fine-tuned (Medium).} We use the \emph{same} $20{,}000$ generated
instances, but concatenate all serialized training examples into a single long text
sequence and fine-tune on this single sequence. This construction preserves the
token-level training signal while substantially reducing the effective diversity
and sampling of training contexts, leading to a weaker improvement in intrinsic ICL
capability compared to the sample-level training above.

Both \textbf{Strong} and \textbf{Medium} variants are trained with the same LoRA
configuration and optimizer settings. We apply LoRA to attention and MLP
projection modules. We train
a causal LM objective with labels equal to the input tokens (standard next-token
prediction over the serialized prompt). Unless otherwise stated, we use a maximum
sequence length of $512$ tokens with padding/truncation.

\subsection{Demonstration Choice}
\label{app:demo}

In this section, our goal is to isolate how the informativeness of demonstrations
affects in-context generalization, under the same prompt template, the same query
distribution, and the same number of demonstrations.

\paragraph{Tasks.}
We consider two retrieval-style tasks constructed as ``$(x, ?, y)$'' completion:
(i) \textbf{Sports}: $x$ is a country and $y$ is its designated sport label in
our benchmark; (ii) \textbf{People}: $x$ is a company and $y$ is a canonical person
name associated with that company in our benchmark (e.g., a founder). The model
is queried with an unseen $x_{\text{test}}$ and asked to predict the
corresponding $y_{\text{test}}$.

\paragraph{Ambiguous vs.\ identifying demonstration pools.}
For each task, we build two disjoint pools of demonstration pairs: \emph{ambiguous}
and \emph{identifying}. Both pools contain $10$ labeled pairs $(x,y)$. Intuitively,
the ambiguous pool contains examples that can be explained by multiple plausible
latent rules (hence providing weaker evidence about which rule the prompt instantiates),
whereas the identifying pool contains examples that more strongly pin down a unique
rule. Importantly, the \emph{test set} (query distribution) is kept the same
across the two conditions; only the demonstration pool differs.

\paragraph{Prompt template.}
Given a demonstration set $\{(x_{j}, y_{j})\}_{j=1}^{k}$, we format the prompt as
\[
    \texttt{Question: (}x_{1}, ?, y_{1}\texttt{), \ldots, (}x_{k}, ?, y_{k}\texttt
    {). }x_{\text{test}}, ?,
\]
and instruct the model to output a single answer token span in English (sport
name for Sports; person name for People).

\paragraph{Evaluation.}
For each task and each condition (ambiguous / identifying), we generate $100$
prompts per random seed by sampling $k\in\{2,4,8\}$ uniformly and sampling $k$
demonstrations uniformly without replacement from the corresponding pool. The query
pair $(x_{\text{test}}, y_{\text{test}})$ is sampled from a fixed test list with
ground-truth labels. Model outputs are parsed into JSON and normalized by
lowercasing and removing punctuation; accuracy is computed by exact match with the
normalized ground truth. We repeat the entire procedure for $5$ random seeds and
report mean $\pm$ standard deviation across seeds.

\paragraph{Results.}
Tables~\ref{tab:demo_choice_sports}--\ref{tab:demo_choice_people} summarize the
accuracy under the two demonstration pools as a function of $k$ on \texttt{Qwen-235B-A22B}
and \texttt{Qwen-30B-A3B}, respectively. Across both tasks, the identifying pool
consistently yields higher accuracy than the ambiguous pool, and the performance
gap generally increases with more demonstrations.

\begin{table}[htbp]
    \centering
    \caption{\textbf{Improved Demonstration Experiment.} Task: sports, repeated 5
    times with different random seeds.}
    \label{tab:demo_choice_sports} \resizebox{\linewidth}{!}{
    \begin{tabular}{ccccc}
        \toprule \textbf{Demo number} & \textbf{Ambiguous Acc (235B)} & \textbf{Identifying Acc (235B)} & \textbf{Ambiguous Acc (30B)} & \textbf{Identifying Acc (30B)} \\
        \midrule 2                    & $0.265 \pm 0.073$             & $0.421 \pm 0.080$               & $0.122 \pm 0.029$            & $0.176 \pm 0.062$              \\
        4                             & $0.267 \pm 0.084$             & $0.542 \pm 0.048$               & $0.151 \pm 0.078$            & $0.410 \pm 0.059$              \\
        8                             & $0.202 \pm 0.037$             & $0.602 \pm 0.043$               & $0.164 \pm 0.039$            & $0.559 \pm 0.074$              \\
        \bottomrule
    \end{tabular}
    }
\end{table}

\begin{table}[htbp]
    \centering
    \caption{\textbf{Improved Demonstration Experiment.} Task: People, repeated 5
    times with different random seeds.}
    \label{tab:demo_choice_people} \resizebox{\linewidth}{!}{
    \begin{tabular}{ccccc}
        \toprule \textbf{Demo number} & \textbf{Ambiguous Acc (235B)} & \textbf{Identifying Acc (235B)} & \textbf{Ambiguous Acc (30B)} & \textbf{Identifying Acc (30B)} \\
        \midrule 2                    & $0.713 \pm 0.074$             & $0.940 \pm 0.049$               & $0.735 \pm 0.083$            & $0.908 \pm 0.059$              \\
        4                             & $0.708 \pm 0.096$             & $0.931 \pm 0.027$               & $0.775 \pm 0.037$            & $0.950 \pm 0.058$              \\
        8                             & $0.648 \pm 0.008$             & $0.845 \pm 0.044$               & $0.820 \pm 0.027$            & $0.990 \pm 0.001$              \\
        \bottomrule
    \end{tabular}
    }
\end{table}

\FloatBarrier
\subsection{Controlled Validation of Demonstration Evidence Rates}
\label{app:controlled_evidence_rate}

The preceding natural-language experiments vary whether demonstrations identify
the intended task.  We complement them with a controlled modular-arithmetic
experiment in which the amount of identifying evidence is known exactly.  This
isolates the evidence-accumulation mechanism in
Assumption~\ref{assp:identifiability} and tests the exponential concentration
predicted by Theorem~\ref{thm:decay_1_5}.

\paragraph{Controlled task family.}
We use three latent tasks on $\mathbb Z_{17}$,
$t\in\{\mathrm{add},\mathrm{sub},\mathrm{mul}\}$, with addition as the true task.
For an input $x=(a,b)$, the three clean outputs are $a+b$, $a-b$, and $ab$
modulo $17$.  Outputs follow a symmetric noise model with probability $0.95$
on the task-consistent answer and $0.05/16$ on each other answer, and the task
prior is uniform.  The ambiguous distribution $q_{\rm amb}$ is concentrated on
$(0,0)$, where all three tasks agree.  The identifying distribution $q_{\rm id}$
is uniform over the 242 inputs for which the three outputs are pairwise distinct.
We continuously control evidence through
\begin{equation}
    q_\lambda=(1-\lambda)q_{\rm amb}+\lambda q_{\rm id},
    \qquad \lambda\in\{0,.25,.5,.75,1\}.
\end{equation}
For either competing task $t$, the population evidence rate is available in
closed form:
\begin{equation}
    \alpha_t(\lambda)
    =\mathbb E_{x\sim q_\lambda,\,y\sim p(\cdot\mid x,t^*)}
      \log\frac{p(y\mid x,t^*)}{p(y\mid x,t)}
    =5.4133106\,\lambda .
    \label{eq:controlled_alpha}
\end{equation}
Thus $\lambda$ changes only the frequency of identifying examples, giving a
direct experimental control over the rate appearing in the theory.

\paragraph{Reference evidence accumulation.}
For each $\lambda$ and
$N\in\{1,2,4,8,16,32,64\}$, we sample 2,000 independent prompts and measure the
cumulative log-likelihood ratio
$S_{N,t}=\sum_{n=1}^N\log[p(y_n\mid x_n,t^*)/p(y_n\mid x_n,t)]$.
Table~\ref{tab:controlled_evidence_slopes} compares the exact rate in
Equation~\eqref{eq:controlled_alpha} with the fitted slope of
$\mathbb E[S_{N,t}]$ against $N$.  The two competing tasks give nearly identical
slopes, every fit has $R^2>0.99999$, and all empirical slopes are within $0.8\%$
of the exact population value.  At $\lambda=0$, the cumulative evidence is
identically zero, as predicted.

\begin{table}[htbp]
    \centering
    \caption{\textbf{Controlled evidence-rate recovery.} Fitted slopes of
    cumulative evidence versus $N$; intervals are 95\% bootstrap confidence
    intervals over prompts.}
    \label{tab:controlled_evidence_slopes}
    \resizebox{0.92\linewidth}{!}{
    \begin{tabular}{cccc}
        \toprule
        $\lambda$ & Exact $\alpha(\lambda)$ & Subtraction competitor & Multiplication competitor \\
        \midrule
        $0.25$ & $1.3533$ & $1.3428\ [1.3298,1.3554]$ & $1.3432\ [1.3298,1.3551]$ \\
        $0.50$ & $2.7067$ & $2.7144\ [2.6981,2.7292]$ & $2.7134\ [2.6988,2.7315]$ \\
        $0.75$ & $4.0600$ & $4.0681\ [4.0533,4.0830]$ & $4.0679\ [4.0531,4.0847]$ \\
        $1.00$ & $5.4133$ & $5.4141\ [5.4070,5.4212]$ & $5.4127\ [5.4051,5.4198]$ \\
        \bottomrule
    \end{tabular}}
\end{table}

\paragraph{Posterior and predictive concentration.}
Using the exact Bayesian reference predictor, increasing either $N$ or
$\lambda$ sharply reduces task-posterior error.  Table~\ref{tab:controlled_reference_error}
shows that the ordering induced by the analytical evidence rate is reproduced
throughout the sweep.  For example, at $\lambda=0.25$ the posterior error falls
from $0.503$ at $N=1$ to $1.08\times10^{-8}$ at $N=64$, while larger evidence
rates reach the numerical floor earlier.  The corresponding predictive error
has the same ordering and concentration pattern.  This provides a direct
finite-sample realization of the theory's exponential evidence mechanism.

\begin{table}[htbp]
    \centering
    \caption{\textbf{Reference task-posterior error.} Mean posterior mass on
    the two competing tasks over 2,000 prompts.}
    \label{tab:controlled_reference_error}
    \resizebox{0.88\linewidth}{!}{
    \begin{tabular}{cccccccc}
        \toprule
        $\lambda$ & $N=1$ & $2$ & $4$ & $8$ & $16$ & $32$ & $64$ \\
        \midrule
        $0.25$ & $.5033$ & $.3742$ & $.2205$ & $.07523$ & $.01243$ & $2.61\!\times\!10^{-4}$ & $1.08\!\times\!10^{-8}$ \\
        $0.50$ & $.3541$ & $.1892$ & $.04979$ & $.003138$ & $3.40\!\times\!10^{-4}$ & $3.90\!\times\!10^{-16}$ & $<10^{-16}$ \\
        $0.75$ & $.1998$ & $.05201$ & $.005197$ & $2.52\!\times\!10^{-4}$ & $1.18\!\times\!10^{-13}$ & $<10^{-16}$ & $<10^{-16}$ \\
        $1.00$ & $.04103$ & $.006837$ & $2.56\!\times\!10^{-4}$ & $2.99\!\times\!10^{-13}$ & $<10^{-16}$ & $<10^{-16}$ & $<10^{-16}$ \\
        \bottomrule
    \end{tabular}}
\end{table}

\paragraph{Recovery by a learned Transformer.}
Finally, we train five independently initialized Transformers to infer the
latent task from the same demonstrations.  Each model has two layers, width 64,
four attention heads, GELU activations, and is trained on uniformly sampled
modular-arithmetic inputs; evaluation uses 200 fresh prompts per model,
$\lambda$, and $N$ (35,000 measurements in total).  As shown in
Table~\ref{tab:controlled_transformer_error}, the learned models recover the
same central dependence on demonstration evidence: error decreases strongly
with $N$, and larger $\lambda$ produces faster task identification.  The rank
correlation between the analytical rate $\alpha(\lambda)$ and negative model
error is $1.0$ for every $N\leq32$ and remains $0.9$ at $N=64$.  At full
identification ($\lambda=1$), mean error decreases from $0.4994$ to $0.0028$.

\begin{table}[htbp]
    \centering
    \caption{\textbf{Learned-Transformer task error.} Means over five model
    seeds and 200 evaluation prompts per model and condition.}
    \label{tab:controlled_transformer_error}
    \resizebox{0.86\linewidth}{!}{
    \begin{tabular}{cccccccc}
        \toprule
        $\lambda$ & $N=1$ & $2$ & $4$ & $8$ & $16$ & $32$ & $64$ \\
        \midrule
        $0.25$ & $.6598$ & $.6285$ & $.5576$ & $.4520$ & $.2868$ & $.1321$ & $.0483$ \\
        $0.50$ & $.6265$ & $.5662$ & $.4375$ & $.2531$ & $.1067$ & $.0293$ & $.0086$ \\
        $0.75$ & $.5625$ & $.4458$ & $.2952$ & $.1444$ & $.0485$ & $.0150$ & $.0093$ \\
        $1.00$ & $.4994$ & $.3471$ & $.1828$ & $.0630$ & $.0139$ & $.0031$ & $.0028$ \\
        \bottomrule
    \end{tabular}}
\end{table}

Taken together, the exact calculation, reference-predictor simulation, and
learned-model experiment consistently support the paper's central setting:
identifying demonstrations contribute a positive per-example evidence rate,
that evidence accumulates linearly with context length, and the resulting task
and predictive errors contract rapidly with $N$.  The controlled design also
shows that the advantage of identifying demonstrations in the natural-language
experiments is explained by information about the latent task rather than by
demonstration count alone.

\FloatBarrier
\subsection{Empirical Validation of Effective Instruction Sample Size}
\label{app:effective_instruction_sample_size}

The shared-template factorization used for repeated instructions treats one
instruction realization copied across several displayed positions as one
task-conditional draw.  By contrast, independently varying realizations may
provide separate pieces of instruction evidence.  We directly test whether a
Transformer can learn this distinction from the cross-position correlation
structure of its training distribution.  The experiment independently varies
the total number $N$ of displayed instruction positions and the number $K$ of
distinct realizations occupying those positions.  Each realization is copied
exactly $N/K$ times.  Thus $K=1$ is the shared-template endpoint and $K=N$ is
the independent-realization endpoint.

\paragraph{Tasks and paired measurement.}
We use the three latent tasks
\begin{equation}
    f_0(a,b)=(a+b)\bmod 17,\qquad
    f_1(a,b)=(a-b)\bmod 17,\qquad
    f_2(a,b)=ab\bmod 17.
\end{equation}
Every demonstration is $(0,0)\mapsto0$, which is equally compatible with all
three tasks.  Consequently, demonstrations provide no task-identifying
information, and the model must infer the task from the instruction
realizations alone.  Query inputs are selected so that the three tasks produce
distinct answers.

For each configuration we measure the paired task-margin shift
\begin{equation}
    C_\theta(N,K)
    :=m_\theta(z_{N,K}^{\rm mismatch})
      -m_\theta(z_{N,K}^{\rm anchor}),
    \qquad
    m_\theta(z)
    :=\log p_\theta(y_{\rm sub}\mid z)
      -\log p_\theta(y_{\rm add}\mid z).
    \label{eq:effective_sample_margin}
\end{equation}
The paired prompts have identical $N$, $K$, cluster identities, repeat counts,
cluster ordering, demonstrations, query, positions, attention mask, and
nuisance embedding coordinates.  They differ only in the task-likelihood
coordinates of the instructions: the anchor favors addition and the mismatch
favors subtraction.  This pairing isolates instruction-induced task pressure
from sequence length and demonstration content.

\paragraph{Instruction realizations.}
Each realization has four coordinates,
\begin{equation}
U_{c,j}=\left[
\log\frac{p(c\mid t=1)}{p(c\mid t=0)},
\log\frac{p(c\mid t=2)}{p(c\mid t=0)},
r\cos\theta_j,
r\sin\theta_j
\right],
\end{equation}
where $c\in\{0,1,2\}$ is the instruction class,
$j\in\{1,\ldots,64\}$ is its alias identity, $r=0.5$, and
$\theta_j=2\pi j/64$.  The first two coordinates contain task evidence; the
last two distinguish aliases but contain no task information.  We set
$p(c=t\mid t)=0.65$ and $p(c\ne t\mid t)=0.175$, so a matching realization
has single-sample log-likelihood ratio
$\log(0.65/0.175)\simeq1.312$.  All 64 aliases within a class have the same
task likelihood.

\paragraph{Training regimes.}
We train under three distributions with identical single-position instruction
marginals, latent-task distribution, query distribution, sequence-length
distribution, architecture, optimization, and total training-token count.
They differ only in cross-position correlation:
\begin{enumerate}[leftmargin=*,itemsep=2pt,topsep=2pt]
    \item \emph{Shared-template}: one realization $U\sim P(U\mid t)$ is sampled
    and copied to all $N$ positions, so the independent instruction sample size
    is one.
    \item \emph{Hierarchical}: $K\in\{1,2,4,\ldots,N\}$ is sampled, $K$
    realizations are drawn independently from $P(U\mid t)$, and each is copied
    $N/K$ times.  The exact instruction likelihood therefore contains $K$,
    rather than $N$, independent factors.
    \item \emph{Independent-occurrence}: all $N$ displayed realizations are
    sampled independently, giving the endpoint $K=N$ during training.
\end{enumerate}
The model is not given $K$, cluster labels, or cluster boundaries; it must
recover the effective sample structure from equality and variation among its
input embeddings.

\paragraph{Model and evaluation.}
The predictor is a two-layer pre-norm Transformer encoder with hidden width 64,
four attention heads, FFN width 128, GELU activations, and no dropout.  It is
trained for 1,600 AdamW steps with batch size 128, learning rate
$8\times10^{-4}$, and weight decay $10^{-4}$.  We train five initialization
seeds and output a 17-way query-answer distribution.  In the principal
experiment the target is the exact latent-task posterior mapped to the three
distinct query answers.  A separate experiment uses sampled hard answer labels.

We evaluate every $N\in\{1,2,4,8,16,32,64\}$ and every
$K\in\{1,2,4,\ldots,N\}$ dividing $N$, using 20 held-out query seeds and 16
realization-identity pools.  This yields 44,800 measurements per training
regime and 134,400 in total.  We fit
\begin{equation}
    C_\theta(N,K)=a+b_NN+b_KK
    \label{eq:effective_sample_regression}
\end{equation}
and compute 95\% confidence intervals using 2,000 model-seed/prompt-seed
cluster-bootstrap replicates.

\paragraph{Results.}
Table~\ref{tab:effective_sample_size} shows a clean separation between displayed
copies and distinct realizations.  For hierarchical training, the coefficient
of $K$ is strictly positive, whereas no residual dependence on $N$ is
statistically distinguishable after controlling for $K$.  The shared-template
control has no significant $K$ coefficient.  Moreover, the estimated
distinct-realization response is ordered with the amount of independence in
the training distribution:
\begin{equation}
    b_K^{\rm shared}
    <b_K^{\rm hierarchical}
    <b_K^{\rm occurrence}.
    \label{eq:effective_sample_ordering}
\end{equation}
The hierarchical point estimate is $10.8$ times the shared-control estimate,
and the confidence intervals for the hierarchical and shared coefficients do
not overlap.  Thus the same Transformer architecture learns different
effective instruction sample sizes solely from cross-position correlations.

\begin{table}[htbp]
    \centering
    \caption{\textbf{Effective instruction sample size.} Coefficients in
    $C_\theta(N,K)=a+b_NN+b_KK$ for posterior-distilled models.  Intervals are
    95\% cluster-bootstrap confidence intervals.}
    \label{tab:effective_sample_size}
    \resizebox{0.94\linewidth}{!}{
    \begin{tabular}{ccccc}
        \toprule
        Training regime & $b_N$ & 95\% CI & $b_K$ & 95\% CI \\
        \midrule
        Shared-template & $-0.000197$ & $[-0.000794,\ 0.000395]$
                        & $0.000030$ & $[-0.000002,\ 0.000062]$ \\
        Hierarchical & $0.001244$ & $[-0.001079,\ 0.003472]$
                     & $0.000328$ & $[0.000207,\ 0.000447]$ \\
        Independent-occurrence & $-0.001733$ & $[-0.007019,\ 0.003765]$
                               & $0.001230$ & $[0.000960,\ 0.001518]$ \\
        \bottomrule
    \end{tabular}}
\end{table}

The hard-label replication confirms that the central result is not an artifact
of posterior distillation.  Under hierarchical training it gives
\begin{equation}
    b_K^{\rm hierarchical}=0.000187,
    \qquad 95\%\ \mathrm{CI}=[0.000097,\ 0.000283],
\end{equation}
while the shared-template coefficient remains statistically indistinguishable
from zero,
$b_K^{\rm shared}=0.000053$ with 95\% CI
$[-0.000011,0.000116]$.  The fitted hierarchical $b_N$ is likewise not
distinguishable from zero.  A distance sweep further recovers the selective
$K$ response at both medium and full mismatch strength: at interpolation
$s=0.5$, $b_K=0.000116$ with 95\% CI $[0.000024,0.000213]$, and at $s=1$,
$b_K=0.000187$ with 95\% CI $[0.000097,0.000282]$.

These results empirically support the factorization used in the theory.  When
training presents repeated slots as deterministic copies of one shared draw,
the learned predictor treats them as a single effective instruction sample.
When training contains clustered independent realizations, task pressure grows
with the number of distinct draws rather than with the number of displayed
copies.  The experiment therefore validates the substantive distinction
between the shared-template and varying-realization settings: their different
scalings can be learned by a Transformer from correlation structure alone,
even when all single-position marginals and all non-instruction aspects of the
training distribution are held fixed.

\FloatBarrier
\subsection{Prompt Sensitivity vs. Demonstration Count}
\label{app:experiment_instruction_variation} We study how token-level numerical
predictions respond to changes in the in-context operator pattern while keeping
the underlying query answer fixed. We evaluate \texttt{Qwen3-30B-A3B} model and
\texttt{Qwen3-235B-A22B} model.

\paragraph{Prompt regimes.}
Each evaluation prompt consists of a fixed-format arithmetic query paired with a
fixed target answer (e.g., \texttt{1345}). We construct three regimes that
differ only in the operator pattern appearing in the demonstrations: (i) \emph{addition-only},
where all in-context examples use \texttt{+}; (ii) \emph{multiplication-only}, where
all examples use \texttt{$\times$}; and (iii) \emph{mixed-symbol}, where different
operators are interleaved across demonstrations. Across all regimes, the query problem
and its correct final answer remain identical.

\paragraph{Relation to theoretical support.}
Because the Qwen pretraining corpus and its latent task labels are unavailable,
we cannot determine whether any of these operator--task pairings has literal
zero conditional probability. Accordingly, this experiment tests consistent
versus varying instruction mismatch, rather than directly verifying
$p(\vi\mid t^*)=0$. The same fixed-cost mechanism remains relevant when
$p(\vi\mid t^*)>0$ but is extremely small: a finite corpus may then contain too
few such pairings for the model to learn the association reliably, making the
instruction effectively unsupported for the learned model, while its fixed
instruction pressure can still be overcome by accumulating demonstration
evidence. It is natural to expect prompts that use a multiplication symbol
while requiring addition outputs to be much rarer in pretraining than ordinary
aligned uses, but we do not treat this expectation as a verified property of
Qwen's training distribution.

\paragraph{Token-level score extraction.}
For each prompt, we extract token-level scores corresponding to the correct digits
of the target answer. Concretely, for a target digit string $a_{1}a_{2}\cdots a_{L}$,
we record, at each output position $\ell$, the model-assigned log-probability of
the correct digit $a_{\ell}$. These per-digit log-probabilities serve as a fine-grained
proxy for the model’s confidence in the correct numerical prediction at each position.
Specifically, in the figures, Char~$1 ,3,4 ,5$ denotes the score trace
associated with the $1,2,3,4$-th digit position of the fixed target answer
\texttt{1345}.

\paragraph{Findings.}
Figure~\ref{fig:add_logit} and~\ref{fig:add_logit_a3b} present the token-level score
trajectories for the \texttt{Qwen3-235B-A22B} and \texttt{Qwen3-30B-A3B} models,
respectively, under the three operator regimes. Corresponding score gradients are
shown in Figure~\ref{fig:add_logit_235b_gradient} and~\ref{fig:add_logit_a3b_gradient}.
In homogeneous-symbol regimes (addition-only or multiplication-only), per-digit score
trajectories are relatively smooth and exhibit small discrete gradients,
indicating stable symbolic behavior. In contrast, the mixed-symbol regime
produces larger variance and more frequent sign changes in the gradients, suggesting
increased instability when the operator distribution shifts within the context.
These patterns are consistent across model scales, although the larger \texttt{Qwen3-235B-A22B}
model exhibits systematically smaller gradient magnitudes, indicating improved robustness
to symbolic heterogeneity.


\begin{figure}[htbp]
    \centering
    \begin{subfigure}
        [t]{0.32\textwidth}
        \centering
        \includegraphics[width=\textwidth]{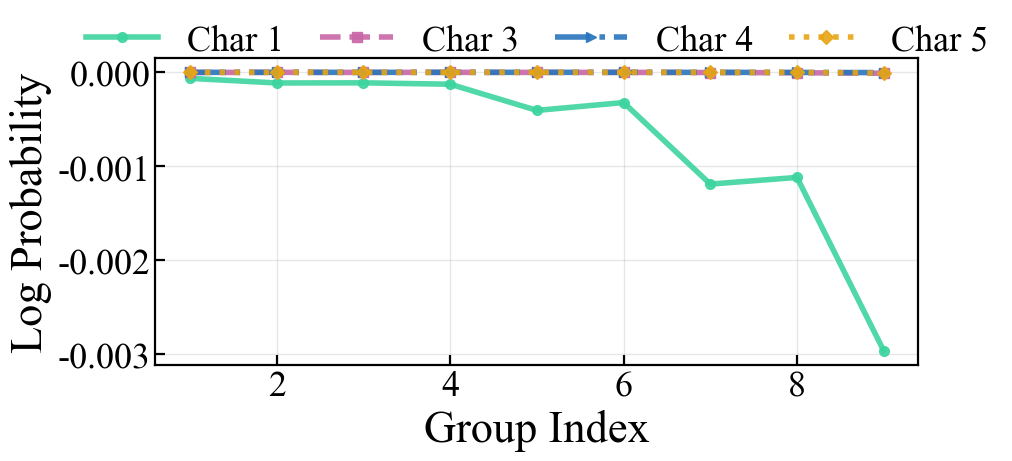}
        \caption{Equivalent Instructions}
        \label{fig:correct_logit_235b}
    \end{subfigure}
    \hfill
    \begin{subfigure}
        [t]{0.32\textwidth}
        \centering
        \includegraphics[width=\textwidth]{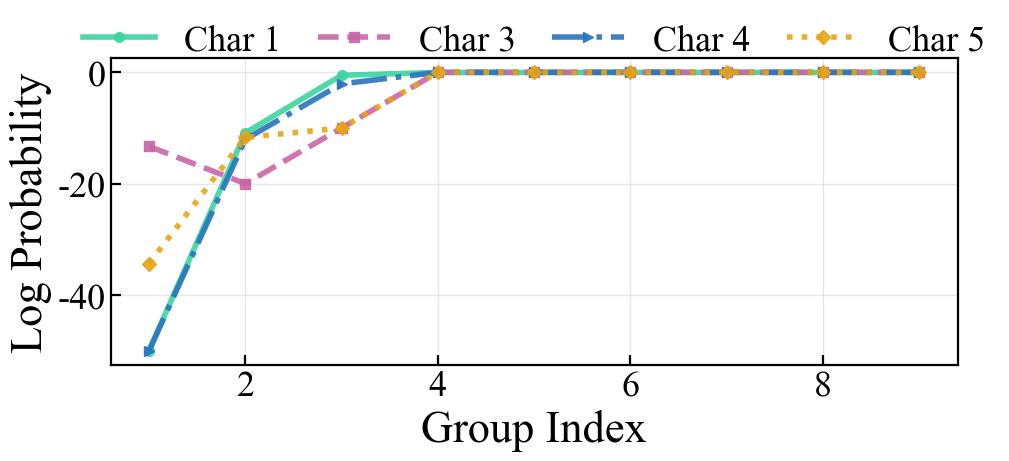}
        \caption{Consistent Instructions}
        \label{fig:consistent_logit_235b}
    \end{subfigure}
    \hfill
    \begin{subfigure}
        [t]{0.32\textwidth}
        \centering
        \includegraphics[width=\textwidth]{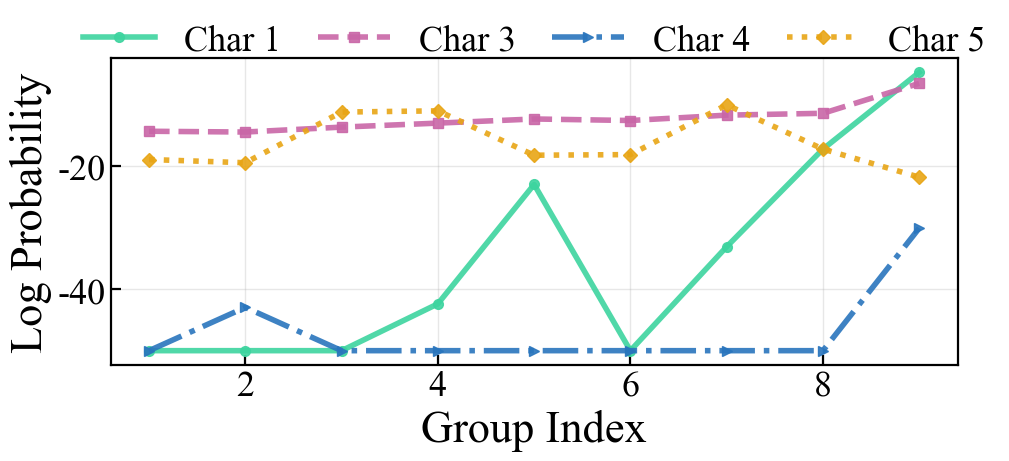}
        \caption{Inconsistent Instructions}
        \label{fig:inconsistent_logit_235b}
    \end{subfigure}
    \caption{\textbf{Posterior confidence under instruction variation.} (\texttt{Qwen-235B-A22B-Instruct})
    (a) Equivalent instructions yield stable, high-confidence predictions. (b)
    Incorrect but consistent instructions allow confidence to recover as more
    demonstrations are added. (c) Incorrect and inconsistent instructions lead to
    persistently low and unstable confidence despite increasing context length.}

    \label{fig:add_logit}
\end{figure}

\begin{figure}[htbp]
    \centering
    \includegraphics[width=0.32\textwidth]{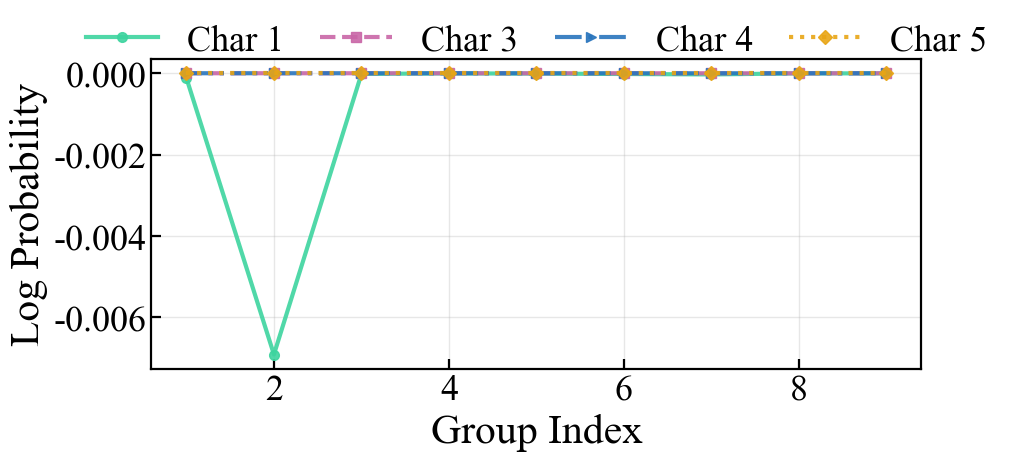}
    \includegraphics[width=0.32\textwidth]{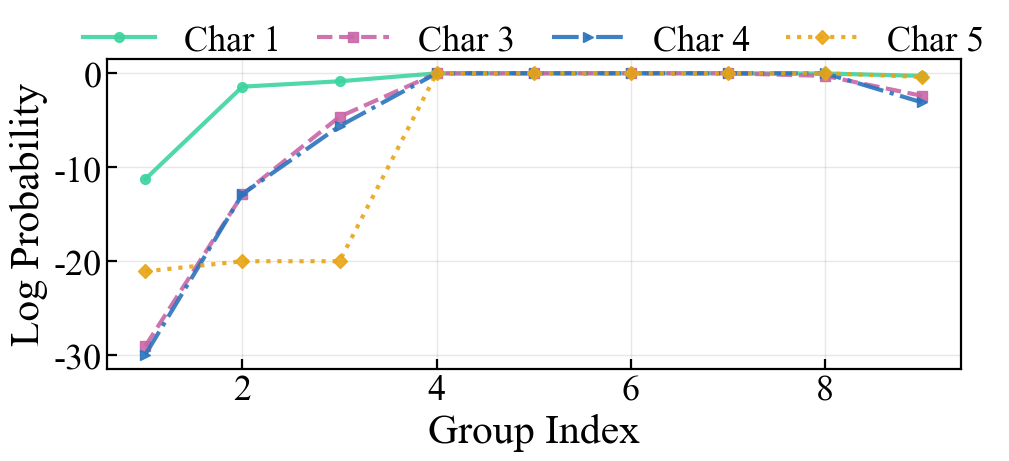}
    \includegraphics[width=0.32\textwidth]{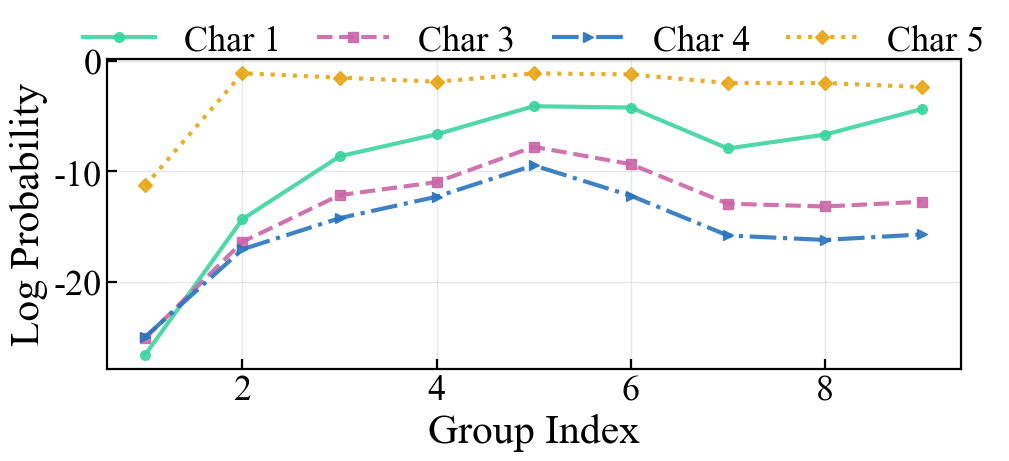}
    \caption{\textbf{Posterior confidence under instruction variation.} (\texttt{Qwen-30B-A3B-Instruct})
    (a) Equivalent instructions yield stable, high-confidence predictions. (b)
    Incorrect but consistent instructions allow confidence to recover as more
    demonstrations are added. (c) Incorrect and inconsistent instructions lead to
    persistently low and unstable confidence despite increasing context length.}
    \label{fig:add_logit_a3b}
\end{figure}

\begin{figure}[htbp]
    \centering
    \includegraphics[width=0.32\textwidth]{
        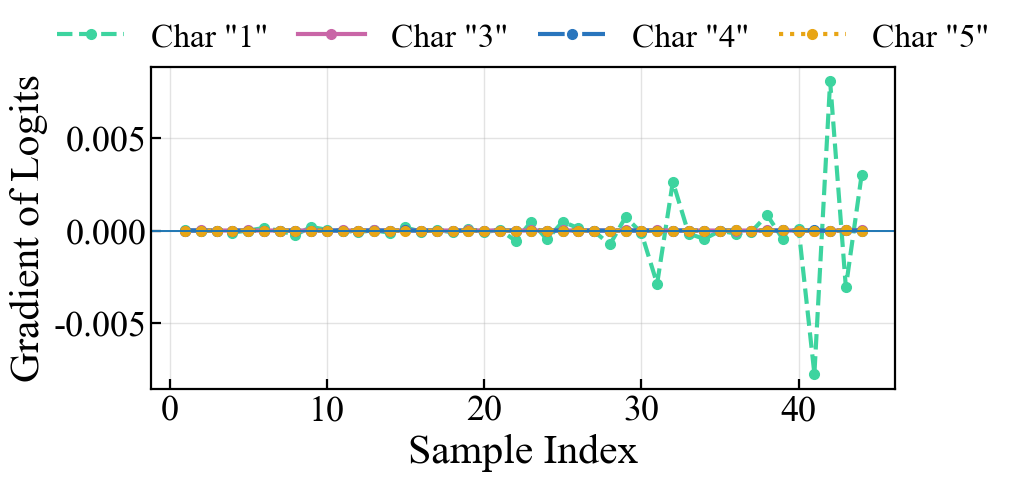
    }
    \includegraphics[width=0.32\textwidth]{
        NeurIPS2026/fig/235b_add_consistent_logits_gradient.png
    }
    \includegraphics[width=0.32\textwidth]{
        NeurIPS2026/fig/235b_add_inconsistent_logits_gradient.png
    }
    \caption{\textbf{Discrete gradients of log probability.} (\texttt{Qwen-235B-A22B-Instruct})
    (a) Equivalent instructions yield small and stable gradients. (b) Incorrect but
    consistent instructions show gradual stabilization. (c) Incorrect and inconsistent
    instructions produce large and irregular gradients.}
    \label{fig:add_logit_235b_gradient}
\end{figure}

\begin{figure}[htbp]
    \centering
    \includegraphics[width=0.32\textwidth]{
        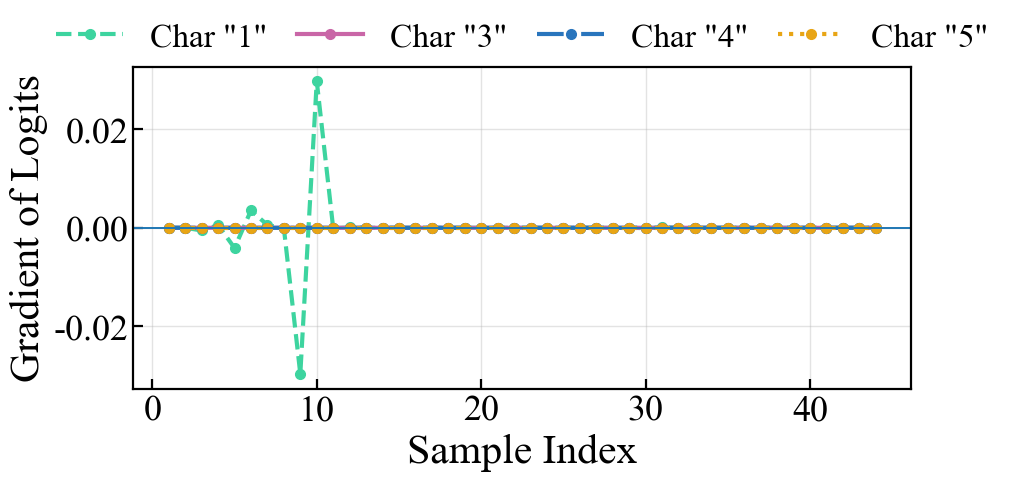
    }
    \includegraphics[width=0.32\textwidth]{
        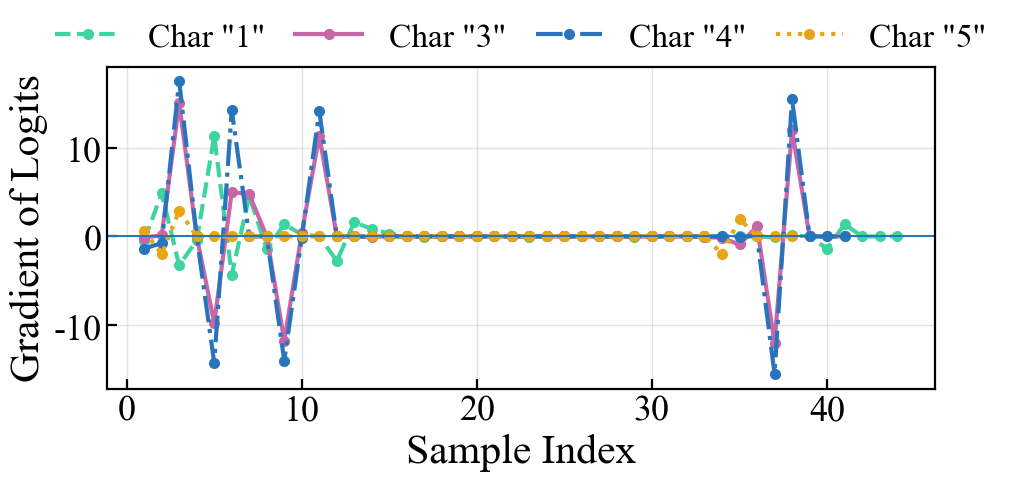
    }
    \includegraphics[width=0.32\textwidth]{
        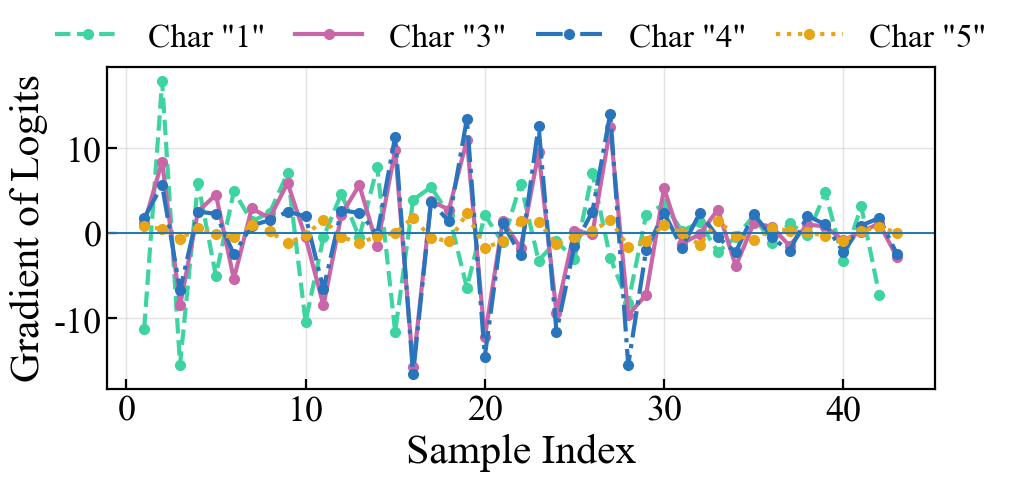
    }
    \caption{\textbf{Discrete gradients of log probability.} (\texttt{Qwen-30B-A3B-Instruct})
    (a) Equivalent instructions yield small and stable gradients. (b) Incorrect but
    consistent instructions show gradual stabilization.(c) Incorrect and inconsistent
    instructions produce large and irregular gradients.}
    \label{fig:add_logit_a3b_gradient}
\end{figure}

\subsection{Controlled Agent Execution through Explicit Intermediate States}
\label{app:cot_effect}
We study a fixed-plan arithmetic agent implemented with
\texttt{Llama3.2-3B}. At execution step $k$, the agent reads its current state
$x^{(k)}$, applies the operation prescribed by the plan, and emits the updated
state $x^{(k+1)}$. Because the language model serializes these states and
updates as an autoregressive text trace, the implementation uses
chain-of-thought (CoT)-style prompting. Here CoT is not a separate theoretical
object: it is the textual interface through which the agent exposes its state
trajectory and executes one transition at a time. This correspondence makes
intermediate states and step-level errors observable and lets us test the
agent-execution mechanism in Section~\ref{subsec:cot_main} and
Theorem~\ref{thm:cot_main}. We compare two agent plans for the same multi-digit
addition objective. They have the same correct terminal state but differ in
whether their component operations are supported by training.

\paragraph{In-distribution agent plan.}
In the \emph{in-distribution} plan, every state transition invokes a
well-defined arithmetic operation, such as digit-wise decomposition, carry
propagation, or aggregation. To ensure that the agent can execute these
operations in isolation, we generate $10{,}000$ training examples for
\emph{each} operation and fine-tune the model on the resulting step-level
datasets before evaluation. Each example is formatted as an
instruction--input--output triple, analogous to the pretraining setup in
Appendix~\ref{app:pretrain_error}.
\begin{quote}
    \small \textbf{Step 1:} $1234561 = 1234560 + 1$ \\
    \textbf{Step 2:} $1323212 = 1323210 + 2$ \\
    \textbf{Step 3:} $1234560 + 1323210 \rightarrow (123456 + 132321)\times 10$
    \\
    \textbf{Step 4:} $1 + 2$ \\
    \textbf{Step 5:} Combine results.
\end{quote}
Each step corresponds to a sub-computation the model can reliably perform in isolation.

\paragraph{Out-of-distribution agent plan.}
The \emph{out-of-distribution} plan reaches the same correct terminal state,
but it requires intermediate operations on which the model has \emph{not} been
explicitly trained. No additional fine-tuning is performed for these
operations; the agent must execute the prescribed state transitions from the
plan supplied in the prompt.
\begin{quote}
    \small \textbf{Step 1:} $1234561 = 1000000 + 234561$ \\
    \textbf{Step 2:} $1323212 = 1000000 + 323212$ \\
    \textbf{Step 3:} Add corresponding components. $1000000 + 1000000 = 2000000,
    234561 + 323212 = 557773$ \\
    \textbf{Step 4:} Combine results.
\end{quote}
Although the plan is algebraically correct, its intermediate transitions do
not correspond to well-supported arithmetic routines in the model's training
distribution and are therefore unreliable during execution.

\paragraph{Training Setup.}
All fine-tuning for the in-distribution agent plan uses the same LoRA configuration, optimizer,
and training hyperparameters as in Appendix~\ref{app:pretrain_error}. In particular, the
architecture, learning rate, number of epochs, batch size, and quantization
settings are kept fixed. This isolates the effect of step-level intrinsic capability
from other confounding factors.

\paragraph{Evaluation.}
We evaluate the step-wise log probabilities to make the empirical measurement
consistent with the pathwise perspective in
Section~\ref{subsec:pathwise_lipschitz}. Let
$f_s(\bm z)$ denote the model logits at execution step $s$, and let
$\ell_s(\bm z)=g(f_s(\bm z))$ be the corresponding cross-entropy loss on the
correct step output. If $f_s$ is $L_{f,s}$-Lipschitz and the cross-entropy map
$g$ is $L_g$-Lipschitz on the bounded logit region considered in the experiment,
then for two prompts $\bm z_1,\bm z_2$ along the same step-wise path,
\[
|\ell_s(\bm z_1)-\ell_s(\bm z_2)|
\le
L_g \|f_s(\bm z_1)-f_s(\bm z_2)\|
\le
L_g L_{f,s}\|\bm z_1-\bm z_2\|.
\]
Thus, if prompts are varied along a fixed step-wise path, gradients of the
correct-output log probability provide a proxy for the effective Lipschitz
sensitivity of that transition. Our plots instead compare different
execution steps, which generally implement different functions; their
variation is therefore a qualitative stability diagnostic rather than a
direct estimate of a within-step Lipschitz constant. Under the in-distribution
agent plan, each transition invokes an operation that is better represented
during training, so the log-probability profile is expected to vary more
smoothly across steps. In contrast, the out-of-distribution plan retains
difficult operations in the execution trajectory, leading to sharper
step-level variations and greater effective step-wise sensitivity.
\paragraph{Results.}
We report three types of measurements: (i) final accuracy, (ii) step-wise accuracy, and (iii) step-wise log-probabilities, which serve as a proxy for the effective Lipschitz behavior at each step.

\textbf{Final accuracy.}
Consistent with Figure~\ref{fig:chain}, the in-distribution (ID) agent plan
substantially improves performance over direct prediction, whereas the
out-of-distribution (OOD) plan degrades performance and is often worse than
executing no plan. This matches Theorem~\ref{thm:cot_main}: an agent can realize
an unsupported target when its planned operations are supported and
identifiable, whereas one poorly identified or poorly approximated transition
can dominate the execution trajectory.

\textbf{Step-wise accuracy.}
To further understand this behavior, we measure the correctness of every
state transition in the agent trajectory. Under the ID plan, early transitions
are executed almost perfectly and most errors occur only at final aggregation.
Under the OOD plan, accuracy degrades progressively, with each subsequent state
transition less reliable than the previous one.

This progressive degradation is consistent with the error-propagation structure
in Theorem~\ref{thm:cot_main}, whose final-error bound sums the identification
and model-transport errors of the individual steps. When intermediate steps are
unreliable, these contributions accumulate and lead to poor final performance.

\textbf{Step-wise Lipschitz behavior (via log-probabilities).}
To connect these observations with the pathwise theory, we examine the average
log-probability assigned to the correct token at each step. As explained above,
the cross-step profile diagnoses stability along the realized agent trajectory;
it does not by itself estimate a within-step pathwise Lipschitz constant.

In the ID setting, the log-probabilities remain stable across steps, providing
evidence that no individual supported subtask introduces a pronounced loss
spike. This is consistent with the favorable regime of
Theorem~\ref{thm:cot_main}, in which every step has a controlled identification
and model-transport error.

In contrast, in the OOD setting, step-wise accuracy decreases substantially and
the log-probability profile exhibits larger fluctuations. This indicates that
some intermediate computations are unreliable, exactly the failure mode
captured by the step-wise sum in Theorem~\ref{thm:cot_main}. When a substep is
poorly identified or has a large model-transport error, its contribution to
the sum can dominate the final result. The experiments are therefore
consistent with the theory's qualitative prediction that an unreliable step
can degrade the final answer; isolating the two error sources individually
would require controlled within-step interventions.
\begin{figure}
\centering
\includegraphics[width=0.45\textwidth]{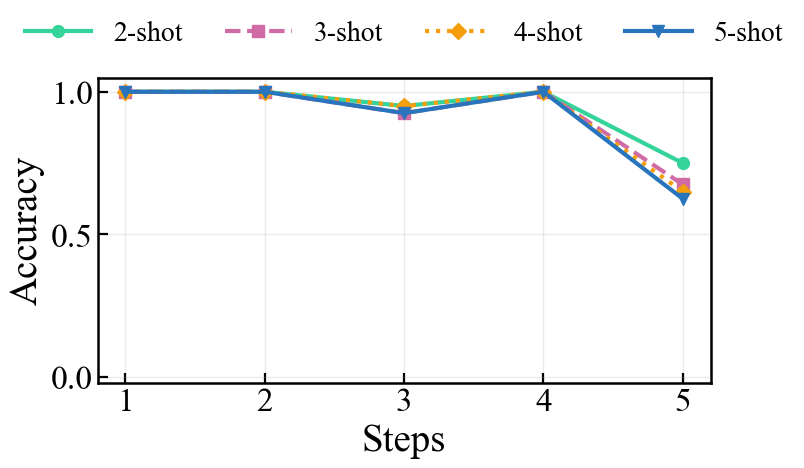}
\includegraphics[width=0.45\textwidth]{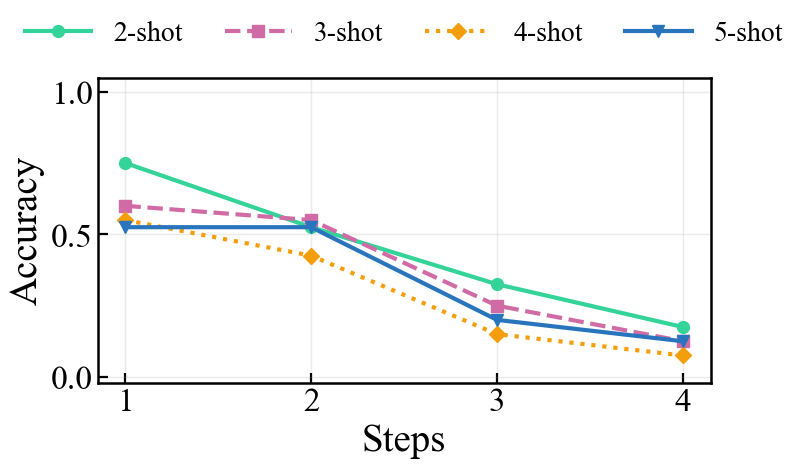}
\caption{Agent transition accuracy. Left: ID plan; right: OOD plan.}
\label{fig:cot_stepwise_acc}
\end{figure}

\begin{figure}
\centering
\includegraphics[width=0.45\textwidth]{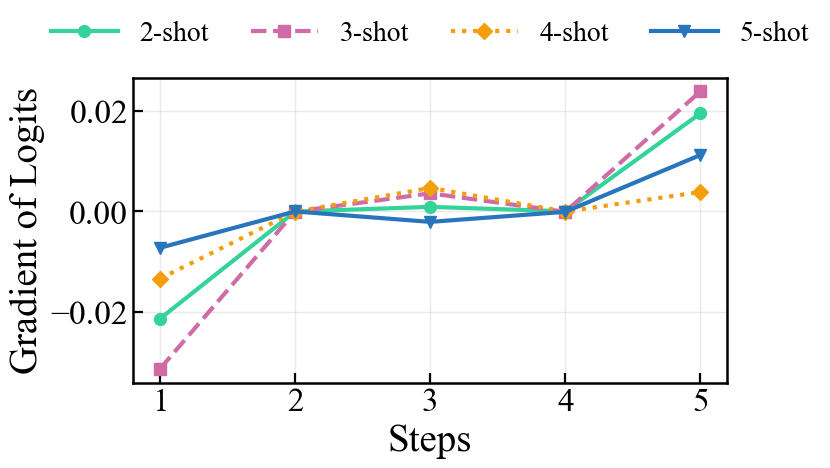}
\includegraphics[width=0.45\textwidth]{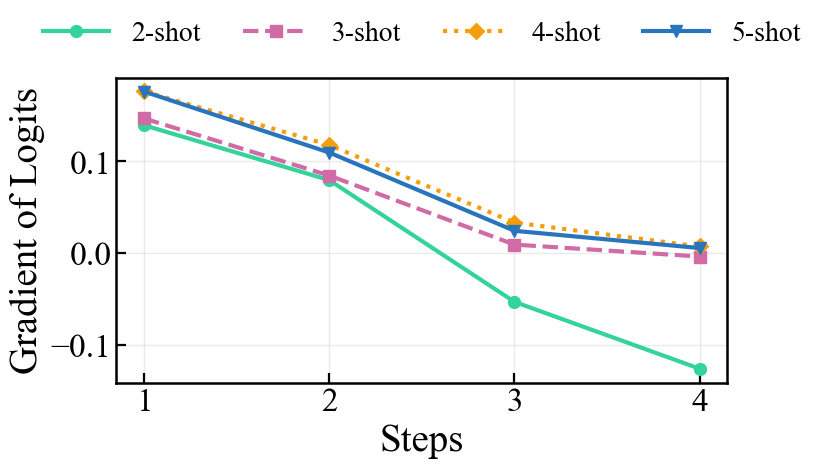}
\caption{Step-wise agent log-probability gradients. Left: ID plan; right: OOD plan.}
\label{fig:cot_logits_gradient}
\end{figure}

\FloatBarrier
\paragraph{Controlled validation of supported agent composition.}
We additionally conduct a controlled finite-task experiment that tests the
central mechanism of Theorem~\ref{thm:cot_main} directly. The agent follows an
explicit three-operation plan, and its intermediate states are fully
observable, so its conditional transition errors can be compared directly
with the theorem's agent-execution bound. All computations are performed modulo $17$. The training
distribution contains only the three supported tasks
\begin{equation}
    f_0(x)=x+1,\qquad f_1(x)=2x,\qquad f_2(x)=x-3,
    \label{eq:cot_controlled_supported_tasks}
\end{equation}
whereas the composed target
\begin{equation}
    f^\dagger(x)=2x-1=f_2\circ f_1\circ f_0(x)\pmod {17}
    \label{eq:cot_controlled_target}
\end{equation}
is never included as a direct training task. Demonstrations contain $8\%$
observation noise, so competing-task likelihood ratios have common support;
the query and task outputs remain deterministic, avoiding an artificial final
error floor.

We compare three execution procedures. \emph{Direct ICL} asks for
$f^\dagger(x)$ as one indivisible mapping. The \emph{supported agent} follows
the three-operation plan
\begin{equation}
    x\xrightarrow{x+1}x^{(1)}
     \xrightarrow{2x}x^{(2)}
     \xrightarrow{x-3}y.
    \label{eq:cot_controlled_good_chain}
\end{equation}
The equal-length \emph{OOD agent} instead follows
\begin{equation}
    x\xrightarrow{x+2}x^{(1)}
     \xrightarrow{2x-2}x^{(2)}
     \xrightarrow{x-3}y,
    \label{eq:cot_controlled_bad_chain}
\end{equation}
whose first two subtasks never occur during training, although their
composition still gives the same correct target $2x-1$. This equal-length
control separates support of the planned operations from the mere presence
or length of an explicit execution trajectory.

We train five independently initialized models and use $12$ prompt seeds. For
each model--prompt pair, we evaluate all $17$ query states at
$N\in\{1,2,4,8,16,32,64\}$. Each reported configuration therefore averages
over $60$ model--prompt measurements, with all query states included in each
measurement. The complete evaluation contains $840$ aggregate records.

\begin{table}[t]
    \centering
    \caption{\textbf{Controlled validation of supported agent composition.}
    Probabilistic error under direct prediction and stochastic agent rollout,
    together with the sum of the supported-agent conditional transition
    errors. Values are mean $\pm$ standard deviation over five models and
    twelve prompt seeds; every measurement includes all $17$ query states.}
    \label{tab:cot_controlled_validation}
    \small
    \setlength{\tabcolsep}{4pt}
    \begin{tabular}{c c c c c}
        \toprule
        $N$ & Direct error & Supported-agent rollout & OOD-agent rollout
            & $\sum_{k=0}^{2}q_k(N)$ \\
        \midrule
         1 & $0.964\pm0.023$ & $0.456\pm0.302$ & $0.877\pm0.103$ & $0.531\pm0.385$ \\
         2 & $0.973\pm0.020$ & $0.185\pm0.170$ & $0.865\pm0.098$ & $0.194\pm0.178$ \\
         4 & $0.974\pm0.021$ & $0.072\pm0.091$ & $0.877\pm0.113$ & $0.074\pm0.093$ \\
         8 & $0.966\pm0.022$ & $0.038\pm0.048$ & $0.900\pm0.078$ & $0.038\pm0.048$ \\
        16 & $0.963\pm0.023$ & $0.016\pm0.018$ & $0.904\pm0.054$ & $0.016\pm0.018$ \\
        32 & $0.966\pm0.022$ & $0.0088\pm0.0035$ & $0.892\pm0.061$ & $0.0088\pm0.0035$ \\
        64 & $0.965\pm0.021$ & $0.0078\pm0.0036$ & $0.898\pm0.042$ & $0.0079\pm0.0036$ \\
        \bottomrule
    \end{tabular}
\end{table}

The contrast is sharp. The supported-agent rollout error decreases from
$0.456$ to $0.0078$, a factor of approximately $58$, and its greedy rollout
error is zero for $N\ge32$. Direct ICL remains near $0.96$ error, consistent
with the fact that the held-out composed task cannot be selected directly from
the training-supported task set. The equal-length OOD agent remains between
$0.865$ and $0.904$ error. Thus the improvement cannot be explained merely by
emitting intermediate states or by plan length: it appears precisely when the
plan consists of supported, identifiable operations, as required by
Theorem~\ref{thm:cot_main}.

\paragraph{Step-wise error propagation.}
For the supported agent plan, let $F_{k-1}$ be the event that all preceding transitions
are correct and define the conditional step-error probability
\begin{equation}
    q_k(N):=P_M(E_k^c\mid F_{k-1}).
    \label{eq:cot_empirical_conditional_step_error}
\end{equation}
We compute these probabilities under the correct prefix and evaluate the exact
stochastic autoregressive rollout distribution, rather than relying only on
greedy decoding. The recursion used in the proof of
Theorem~\ref{thm:cot_main} gives
\begin{equation}
    e_{\mathrm{rollout}}(N)
    \le \sum_{k=0}^{2}q_k(N).
    \label{eq:cot_empirical_union_bound}
\end{equation}
All $840$ evaluated records satisfy this inequality. This checks error
accumulation for the rollout kernel actually evaluated; the inequality itself
does not require exact compatibility with an extracted local prompt.
The bound is non-vacuous and becomes nearly tight once the individual step errors
are small. Using unrounded means,
\begin{equation}
\begin{aligned}
N=4 &: \quad 0.0725 \le 0.0736,\\
N=16&: \quad 0.0161 \le 0.0162,\\
N=64&: \quad 0.00783 \le 0.00785.
\end{aligned}
\label{eq:cot_empirical_tight_examples}
\end{equation}
This tightness is expected in the low-error regime because simultaneous-error
corrections are higher order in the $q_k$. The experiment therefore shows that
the theorem's step-wise error decomposition is not merely a loose worst-case
accounting device in this controlled setting: the measured conditional errors
closely explain the final rollout error. It does not separately establish
that extracted-prompt error bounds transfer to a full-context kernel without
a correction. That requires matched evaluations of the two kernels in
\eqref{eq:cot_rollout_discrepancy} or
\eqref{eq:cot_one_sided_discrepancy}. For an executor that explicitly invokes
the extracted step prompts with no additional retained context, the
zero-discrepancy guarantee follows instead from
Corollary~\ref{cor:cot_explicit_step_execution}.

\paragraph{Actual-model tracking of the reference predictor.}
Finally, we evaluate the discrepancy between the Transformer and the Bayesian
reference predictor on the supported steps. The maximum step-wise
model--reference gap decreases from $0.582$ at $N=1$ to $0.027$ at $N=64$.
Hence the observed concentration is not confined to the idealized Bayesian
reference: as demonstration evidence accumulates, the actual Transformer also
tracks that reference increasingly closely. The theorem keeps this discrepancy
explicit through the terms $\varepsilon_M^{(k)}(N)$ and does not require them
to vanish; their empirical decrease is an additional favorable property of
the learned model in this experiment.

\subsection{Continuous Soft-Prompt Paths in a Pretrained Qwen3 Model}
\label{app:qwen_soft_prompt_interpolation}

We finally test the pathwise perspective of
Section~\ref{subsec:pathwise_lipschitz} directly in a pretrained open-weight
language model. Unlike the preceding discrete prompt comparisons, this
experiment bypasses tokenization after constructing the path: interpolated
instruction embeddings are injected into
\texttt{Qwen3-4B-Instruct-2507} through \texttt{inputs\_embeds}. The model is
evaluated in BF16 without quantization or parameter updates.

\paragraph{Task and paths.}
The task is addition modulo $17$. Each prompt contains one instruction, $N$
addition demonstrations, and one query, with
$N\in\{1,2,4,8,16,32\}$. We use eight independently sampled prompts for each
$N$ and choose every query so that addition, subtraction, and multiplication
give distinct answers. We pre-specify two paths. The first interpolates from
the supported instruction ``Perform addition modulo 17.'' to the mismatched
instruction ``Perform subtraction modulo 17.'' The second interpolates from
the same supported anchor to a synthetic soft-OOD endpoint whose token-wise
embedding norms are matched to those of the anchor. For both paths,
\begin{equation}
    E_s=(1-s)E(\vi_0)+sE(\widetilde\vi),
    \qquad s\in\{0,.05,\ldots,1\}.
    \label{eq:qwen_soft_prompt_path}
\end{equation}
The resulting evaluation contains $2{,}016$ path points.

Both symbolic endpoint instructions occupy the same seven token slots and
differ only in the task word. We first obtain the ordinary input embeddings
from the model embedding layer, replace only these seven instruction slots by
$E_s$, and pass the resulting tensor directly to the model. Sequence length,
attention mask, position IDs, demonstrations, query, and every embedding
outside the instruction span remain fixed along a path. At $s=0$, the maximum
logit difference between the \texttt{inputs\_embeds} computation and the
ordinary \texttt{input\_ids} computation is exactly zero; the same check at
$s=1$ for the symbolic path also gives zero. No evaluated prompt produces a
NaN or infinite value. These checks verify that the experiment follows a
genuine continuous path through the embedding space rather than a sequence of
retokenized discrete prompts.

\paragraph{Output probabilities and empirical path sensitivity.}
We compute the full teacher-forced continuation log-probability
$\ell_t(s,N)$ for the addition, subtraction, and multiplication answers,
including all answer tokens, and define the candidate-normalized probability
\begin{equation}
    q_t(s,N)
    :=\frac{\exp(\ell_t(s,N))}
            {\sum_{t'\in\{\mathrm{add},\mathrm{sub},\mathrm{mul}\}}
             \exp(\ell_{t'}(s,N))}.
    \label{eq:qwen_candidate_normalized_probability}
\end{equation}
For adjacent path points, we measure the path-restricted finite-difference
sensitivity of the correct-output log-probability by
\begin{equation}
    \widehat L_j^{\log p}(N)
    :=\frac{\left|
       \log M(y^*\mid\vz_N(E_{s_{j+1}}))
       -\log M(y^*\mid\vz_N(E_{s_j}))\right|}
       {\|E_{s_{j+1}}-E_{s_j}\|_F}.
    \label{eq:qwen_empirical_path_sensitivity}
\end{equation}
We analogously compute $\widehat L_j^{\mathrm{err}}$ for
$1-q_{\mathrm{add}}$ and $\widehat L_j^{\mathrm{sub}}$ for the subtraction
log-odds $\ell_{\mathrm{sub}}-\ell_{\mathrm{add}}$. These are empirical
sensitivities restricted to the specified paths, precisely matching the local
quantity suggested by the pathwise analysis rather than a global Lipschitz
constant.

\begin{table}[t]
    \centering
    \caption{\textbf{Empirical sensitivity along genuine Qwen3 soft-prompt
    paths.} Values average the finite-difference sensitivities within each
    path region and then across the six demonstration counts. Both paths show
    a substantially smoother neighborhood around the supported anchor than
    the intermediate task-transition region.}
    \label{tab:qwen_soft_prompt_sensitivity}
    \small
    \setlength{\tabcolsep}{5pt}
    \begin{tabular}{l l c c c}
        \toprule
        Path & Region & $\widehat L^{\log p}$
             & $\widehat L^{\mathrm{err}}$
             & $\widehat L^{\mathrm{sub}}$ \\
        \midrule
        Symbolic mismatch
          & Anchor $[0,.25]$   & $3.668$  & $0.166$ & $3.708$ \\
          & Middle $[.25,.75]$ & $20.674$ & $0.911$ & $26.458$ \\
          & Endpoint $[.75,1]$ & $7.518$  & $0.076$ & $7.431$ \\
        \addlinespace
        Synthetic soft-OOD
          & Anchor $[0,.25]$   & $1.221$ & $0.053$ & $1.198$ \\
          & Middle $[.25,.75]$ & $6.805$ & $0.249$ & $6.143$ \\
          & Endpoint $[.75,1]$ & $3.447$ & $0.112$ & $2.962$ \\
        \bottomrule
    \end{tabular}
\end{table}

\paragraph{A locally smooth supported-instruction neighborhood.}
The two independently specified paths exhibit the same geometric structure.
On the symbolic mismatch path, the middle-to-anchor sensitivity ratios are
$5.64$, $5.50$, and $7.14$ for the correct-output log-probability, candidate
error, and subtraction log-odds, respectively. On the synthetic soft-OOD path,
the corresponding ratios are $5.58$, $4.66$, and $5.13$. Thus, across both a
semantically meaningful competing endpoint and a synthetic endpoint, Qwen3
has a comparatively smooth neighborhood around the supported addition
instruction, separated from the alternative endpoint regime by a much more
sensitive transition region.

This experiment supplies a direct model-level realization of the soft-prompt
geometry used in the theory. The path is formed inside the actual embedding
space of a pretrained language model, and the measured output variations show
that the pathwise sensitivity distinguishes a stable supported neighborhood
from the transition to alternative instruction interpretations. It therefore
complements the discrete Qwen3 prompt experiments by exposing the continuous
local structure that those comparisons cannot observe.

%% file: NeurIPS2026/appendix/limitations.tex
\label{app:limitations}

Our results are conditional characterizations rather than unconditional claims
about every pretrained language model. The latent-task framework assumes a
finite task family, task-conditional generation, and the format-specific
factorizations used in the proofs. In particular, repeated copies of one
instruction are modeled as a single shared template, whereas distinct
corruptions are treated as separate instruction slots. These choices isolate
the effect of prompt structure, but a real pretraining distribution need not
admit this exact decomposition. Task identifiability and the integrability
conditions used for posterior concentration are likewise sufficient conditions
and need not hold for ambiguous or overlapping tasks.

The treatment of unsupported instructions requires a distinction between what
is identified by data and what is selected by the model. Population
cross-entropy identifies the Bayes predictor only on the pretraining support
and does not determine a latent-task decomposition or an off-support value. We
therefore impose supported calibration and use a Transformer-induced scoring
map to define the reference instruction log-odds away from symbolic support, as
described in Appendix~\ref{app:probability_models}. Analytic Transformer
components guarantee the existence of the finite continuation chain along a
fixed compact soft-prompt path, but they do not imply that the selected
extension is semantically correct or close to an unknown data-generating law
off support. Moreover, nearby complex singularities can make the continuation
radii small and the constants $M_\ell$ and $\Omega$ large, even when the chain
exists. Appendix~\ref{app:uniform_local_transport} provides a quantitative
local model-transport certificate under supported-anchor calibration and
length-independent architectural stability. Its constants may be too large
to certify an experimentally useful radius. Small average pretraining loss
alone does not imply worst-case anchor accuracy: the loss-to-calibration
result requires coverage, and its finite-family version does not extend to
all prompt lengths. Fixed-context models are covered only over admissible
lengths, not by an infinite-context asymptotic claim.

The condition $p(\instruction\mid t^*)=0$ is a statement about the latent
pretraining data-generating distribution, rather than a quantity that can
generally be verified for a real LLM. Complete pretraining corpora and their
sampling mechanisms are usually unavailable. Even access to a finite corpus
would show only whether a particular instruction--task pairing occurs in that
realization, not whether its population probability is exactly zero.
Consequently, experiments with real pretrained models can only use reasonable
proxies---such as the provenance, novelty, or apparent absence of a pairing---
to motivate that an instruction is unsupported or extremely rare for the
target task. Such evidence does not certify the literal zero-support condition,
so empirical ``unsupported'' instructions should be interpreted as plausible
analogues of the theoretical regime.